\documentclass[twoside,11pt]{article}

\usepackage[preprint]{jmlr2e}

\usepackage{amsmath}
\usepackage{multirow}
\usepackage{arydshln}
\usepackage[]{algorithm2e}
\usepackage{subcaption}
\usepackage{booktabs}
\usepackage{xcolor}
\usepackage{float}
\usepackage{adjustbox}
\usepackage{enumitem}
\usepackage[normalem]{ulem}

\jmlrheading{}{2026}{}{}{}{}{Ghosh, Changala, Kanakamalla, and Harikumar}
\makeatletter
\def\ps@jmlrtps{\let\@mkboth\@gobbletwo
  \def\@oddhead{}\def\@evenhead{}\def\@evenfoot{}
  \def\@oddfoot{\parbox[t]{\textwidth}{\raggedright \scriptsize \copyright 2026 Ghosh, Changala, Kanakamalla, and Harikumar\hfill}}}
\makeatother

\ShortHeadings{Stable Dimensionality Reduction via Isometric Centroid-Encoder}{Ghosh, Changala, Kanakamalla and Harikumar}
\firstpageno{1}

\begin{document}

\title{Stable Dimensionality Reduction via Isometric Centroid-Encoder}

\author{\name Tomojit Ghosh \email tomojit.ghosh@wright.edu \\
3640 Colonel Glenn Hwy, Dayton, OH 45435-0001, USA\\
       \addr Department of Computer Science and Engineering\\
       Wright State University\\
       \name Siva Rama Krishna Prasad Changala \email changala.2@wright.edu \\
       3640 Colonel Glenn Hwy, Dayton, OH 45435-0001, USA\\
       \addr Department of Computer Science and Engineering\\
       Wright State University\\
       \name Prathyusha Kanakamalla \email kanakamalla.2@wright.edu \\
       3640 Colonel Glenn Hwy, Dayton, OH 45435-0001, USA\\
        \addr Department of Computer Science and Engineering\\
       Wright State University\\
       \name Harigovind Harikumar \email harikumar.2@wright.edu \\
       3640 Colonel Glenn Hwy, Dayton, OH 45435-0001, USA\\
       \addr Department of Computer Science and Engineering\\
       Wright State University
}

\maketitle

\begin{abstract}%
Many supervised dimensionality reduction methods promote intra-class compactness and inter-class separation but do not explicitly constrain class geometry, leading to variability in the relative arrangement of classes across independent runs. Such variability limits the reproducibility and interpretability of low-dimensional embeddings. We propose the Isometric Centroid-Encoder (ICE), a supervised dimensionality reduction framework in which class centroids are learned directly in the latent space while preserving a prescribed pairwise class-centroid geometry. Unlike existing centroid- or prototype-based methods that primarily encourage intra-class compactness, ICE explicitly regularizes the relative arrangement of latent class centroids. Unlike multidimensional scaling and many metric-learning approaches that preserve sample-level relationships, ICE operates on class centroids, preserving inter-class geometry without requiring large numbers of sample-pair or triplet constraints. We further show that the intrinsic dimensionality required to represent the centroid-distance configuration of $M$ classes is at most $M-1$, explaining the observed saturation of distance-preservation measures as the latent dimensionality increases. To quantify embedding reproducibility, we introduce the Centroid Neighborhood Stability (CNS) metric, which measures the consistency of class-level neighborhood topology across independent runs and complements existing embedding quality measures. Experiments on image and biological datasets demonstrate that ICE produces stable embeddings with improved preservation of class-centroid geometry while remaining computationally efficient. Additional analyses of optimization dynamics, sensitivity, and preprocessing further provide practical insights into latent dimensionality and explicit geometry modeling for supervised dimensionality reduction.
\end{abstract}


\section{Introduction}
\label{intro}
Dimensionality reduction is a fundamental component of pattern recognition, particularly in the analysis of high-dimensional data where visualization, classification, and interpretation are of primary interest \citep{zahir2026novel,artemiou2021real}. In the supervised setting, numerous methods have been developed that exploit class labels to improve low-dimensional representations \citep{ghosh2024linear}. Existing approaches typically emphasize one or more desirable properties, including class separation \citep{he2021representative}, intra-class compactness \citep{ghosh2024linear,zhou2023latent}, local neighborhood preservation \citep{wei2022neighborhood}, or discriminative metric learning through pairwise and triplet relationships \citep{duan2025deep}. More recently, prototype- and proxy-based methods \citep{ghosh2024linear,chen2025fourier} have been proposed to improve scalability by replacing sample-level comparisons with class representatives.

Despite their success, most supervised dimensionality reduction methods are evaluated primarily through classification accuracy, neighborhood preservation, or visual separation of classes. Comparatively little attention has been paid to the \emph{stability} of the learned embedding. In practice, repeated executions of the same model often produce substantially different low-dimensional layouts due to random initialization, stochastic optimization, and model architecture choices \citep{kobak2021initialization}. While these embeddings may exhibit similar classification performance, the relative arrangement of classes can vary significantly across runs \citep{gursoy2023alignment}. Such variability limits interpretability and reproducibility, particularly in scientific applications where the geometric relationships among classes are often used to draw conclusions about the underlying data.

Several approaches have attempted to preserve geometric structure during embedding. Classical multidimensional scaling (MDS) \citep{torgerson1952} and its deep extensions \citep{kamyab2024deep} preserve pairwise distances among samples, while neighborhood-based methods maintain local relationships through graph or probability-based objectives \citep{wei2022neighborhood,meng2023class}. Similarly, metric-learning methods based on contrastive \citep{duan2025deep}, triplet \citep{boutros2022self}, or pairwise losses \citep{szubert2019structure} encourage discriminative representations by operating on sample-level relationships. However, these methods generally require large numbers of pairwise or triplet comparisons and are not explicitly designed to produce reproducible class-level geometry across independent runs. Furthermore, existing evaluation protocols rarely assess whether the learned embedding remains stable under repeated training.

In this paper, we approach supervised dimensionality reduction from a different perspective. Rather than focusing solely on class separation or neighborhood preservation, we explicitly target the reproducibility of class-level geometry. \textit{Our central hypothesis is that explicitly constraining the relative arrangement of class centroids can substantially improve the stability and reproducibility of supervised embeddings.} To address this issue, we propose the \emph{Isometric Centroid-Encoder (ICE)}, a stability-aware supervised embedding framework that combines centroid-based representation learning with preservation of pairwise class-centroid geometry. By operating on class centroids rather than individual samples, ICE provides a compact geometric representation of inter-class relationships while avoiding the large computational burden associated with sample-level distance preservation. A conceptual comparison highlighting the distinctions between ICE and existing families of supervised dimensionality reduction methods is provided in Table~\ref{supp:tab:related_comparison} in Section~\ref{supp:sec:supp_related}.

To quantify embedding reproducibility, we introduce \emph{Centroid Neighborhood Stability (CNS)}, a metric that measures the consistency of class-level neighborhood structure across multiple independent runs. Unlike commonly used measures such as classification accuracy, stress, distance correlation, trustworthiness, and continuity, which assess embedding quality within a single training run, CNS explicitly evaluates the reproducibility of class-level neighborhood structure across independent runs. In this sense, CNS complements existing geometry-preservation and neighborhood-preservation measures by quantifying run-to-run stability. In addition, we show that the centroid geometry of an $M$-class problem can be represented within a latent space of dimension at most $M-1$, providing a practical guideline for selecting the latent dimensionality and explaining the observed saturation of distance-preservation measures beyond a certain embedding dimension.
The core contributions of this work are summarized as follows:
\begin{enumerate}[topsep=0pt,itemsep=0pt,parsep=0pt]
    \item We identify embedding reproducibility as a fundamental but largely overlooked objective of supervised dimensionality reduction and formulate stability as an explicit optimization objective.
    \item We propose ICE, a stability-aware supervised embedding framework that preserves class-level geometry through centroid-distance regularization while avoiding sample-level pairwise or triplet optimization.
    \item We introduce Centroid Neighborhood Stability (CNS), a metric for quantifying the run-to-run consistency of class neighborhood structure, complementing existing accuracy-based and distance-based evaluation measures.
    \item We derive practical guidance for selecting the latent dimension based on the intrinsic geometry of class centroids, providing a principled interpretation of why geometry-preservation measures saturate as the embedding dimension increases.
    \item We demonstrate experimentally on image and high-dimensional biological datasets that ICE produces more stable and geometrically consistent embeddings while maintaining competitive classification performance.
\end{enumerate}

The remainder of the paper is organized as follows. Section~\ref{lit_review} reviews related work. Section~\ref{ICE} presents the proposed ICE framework. Section~\ref{sec:cns} introduces the newly proposed robustness metric CNS. Section~\ref{exp_setup} describes the experimental setup and evaluation metrics followed by analysis of proposed ICE model in Section~\ref{sec:ICE_analysis}. Section~\ref{sec:benchmark} presents the benchmarking results and discussion. Finally, Section~\ref{conc_future_work} concludes the paper.

\section{Related Work}
\label{lit_review}
Dimensionality reduction (DR) has a long history in pattern recognition and remains an active area of research \citep{zahir2026novel,cai2025multi,ghosh2022supervised}. In supervised DR, the objective is not only to reduce dimensionality but also to preserve class-relevant structure in the embedding space. Existing methods achieve this through class separation, neighborhood preservation, metric learning, or prototype-based supervision. Throughout this work, \emph{class-level geometry} refers to the pairwise distance configuration among class centroids. Since ICE is motivated by the robustness and reproducibility of this geometry, we review prior work from this perspective.

\noindent{\textbf{Supervised linear dimensionality reduction:}}
Fisher's linear discriminant analysis (LDA) \citep{fisher36lda} maximizes between-class separation while minimizing within-class scatter. Numerous variants address singularity, robustness, and high-dimensional settings \citep{he2021representative,witten2011penalized}. Related approaches incorporate supervision into PCA through dependence analysis, feature selection, or multi-objective optimization \citep{bair2006,Barshan:2011:SPC:1950989.1951174,ritchie2019supervised}. The recently proposed Supervised Linear Centroid-Encoder (SLCE) \citep{ghosh2024linear} reconstructs class centroids via orthogonal projection and can be viewed as a supervised analogue of PCA. Although these methods improve class discrimination, they do not explicitly preserve a prescribed pairwise geometry among class centroids.

\noindent{\textbf{Supervised manifold and neighborhood-preserving methods:}}
Neighborhood-based approaches preserve local sample relationships using supervised variants of Isomap \citep{cheng2012}, Laplacian Eigenmaps \citep{raducanu2012}, locality preserving projection \citep{wong2012supervised}, class-constrained t-SNE (csTSNE) \citep{meng2023class}, and supervised UMAP \citep{mcinnes2018umap}. While these methods often produce visually meaningful embeddings and strong local class separation, their objectives are defined at the sample-neighborhood level rather than on class-centroid relationships. Consequently, the relative arrangement of class centroids may vary across independent runs due to stochastic optimization and random initialization \citep{bohm2022attraction,kobak2021initialization}.

\noindent{\textbf{Deep supervised embedding and metric learning:}}
Deep embedding methods extend classical objectives such as Neighborhood Component Analysis (NCA) \citep{Goldberger2004NCA}, MCML \citep{globerson2006metric}, and t-SNE \citep{tSNE} to nonlinear settings \citep{DBLP:journals/jmlr/Maaten09,salakhutdinov2007learning,min2010deep}. Modern metric learning further learns discriminative representations using contrastive, pairwise, triplet, or proxy-based objectives. Related approaches include supervised autoencoders \citep{le2018supervised}, adversarial autoencoders \citep{makhzani2015}, and the Super-Encoder (SE) \citep{kamal2022super}. Although these methods optimize sample-level similarities, reconstruction objectives, or proxy-based classification, they do not explicitly preserve a prescribed class-centroid geometry.

\noindent{\textbf{Centroid-based supervised embeddings:}}
Centroid-based methods represent each class by a prototype. In the Centroid-Encoder (CE) \citep{ghosh2022supervised}, samples are mapped to their corresponding ambient-space class centroids, encouraging intra-class compactness through centroid reconstruction instead of input reconstruction. However, centroid reconstruction alone does not constrain pairwise relationships among class centroids, allowing the global centroid configuration to vary across runs.

\noindent{\textbf{Geometry preservation and motivation for ICE:}}
Classical multidimensional scaling (MDS) \citep{torgerson1952,kruskal1964} preserves pairwise distances among samples to recover global geometry. ICE adopts the same principle at the class-centroid level by regularizing the centroid-distance matrix rather than sample distances. Consequently, ICE combines centroid-based supervision with explicit preservation of class-level geometry while avoiding the large number of sample pairs or triplets required by many metric-learning methods. Unlike proxy-based approaches, where class representatives primarily serve as discriminative anchors, ICE treats latent centroids as geometric entities and explicitly preserves their prescribed pairwise distance configuration. To the best of our knowledge, no existing supervised dimensionality reduction method jointly learns latent class centroids while explicitly preserving a prescribed class-centroid geometry.

\subsection{Relationship of ICE to Existing Representation Learning Methods}
\label{supp:sec:supp_related}
Although ICE shares certain high-level similarities with centroid-based learning, metric learning, and multidimensional scaling, its objective differs fundamentally from these approaches. Below table~\ref{supp:tab:related_comparison} provides a conceptual comparison that clarifies the design philosophy of ICE and its relationship to existing supervised representation learning methods.
\begin{table}[ht!]
\small
\adjustbox{max width=\linewidth}{\begin{tabular}{p{3cm} p{2.5cm} p{3.2cm} p{3cm} p{4cm}}
\hline
\textbf{Method Family} &
\textbf{Unit of Comparison} &
\textbf{Class Relationship Modeled} &
\textbf{Primary Objective} &
\textbf{Difference from ICE} \\
\hline

Contrastive / Triplet Learning &
Sample pairs or triplets &
Binary attraction-repulsion between samples &
Discriminative representation learning and retrieval &
Does not explicitly preserve a reference geometry among classes; all negative classes are generally treated as dissimilar. \\

Proxy-Based Metric Learning &
Samples and learned class proxies &
Sample-to-proxy similarity &
Efficient metric learning and classification &
Proxies serve as discriminative anchors, but their mutual geometry is not constrained to match a prescribed class-level structure. \\

Prototypical Networks &
Class prototypes &
Distance to class prototypes &
Few-shot classification &
Uses class representatives for classification but does not preserve pairwise prototype geometry across classes. \\

MDS / Distance-Preserving Embedding &
Sample distances &
Global sample geometry &
Distance preservation and visualization &
Preserves pairwise relationships among samples rather than supervised class-level geometry. \\

Centroid-Encoder (CE) &
Class centroids &
Within-class compactness &
Centroid reconstruction &
Does not explicitly preserve relationships among class centroids. \\

\textbf{ICE (Proposed)} &
\textbf{Class centroids} &
\textbf{Pairwise class-centroid geometry} &
\textbf{Stable and interpretable supervised embedding} &
\textbf{Simultaneously enforces intra-class compactness and preservation of class-level geometry through centroid-distance regularization.} \\

\hline
\end{tabular}}
\caption{Comparison of ICE with representative supervised representation learning paradigms, highlighting the optimization objective and the geometric relationships explicitly modeled by each method.}
\label{supp:tab:related_comparison}
\end{table}

The comparison highlights that existing methods primarily optimize sample-level discrimination, neighborhood preservation, or sample-level geometry. In contrast, ICE explicitly preserves a prescribed pairwise distance configuration among dynamically learned class centroids while simultaneously enforcing intra-class compactness. Consequently, the objective of ICE is not only discriminative representation learning but also the reproducible preservation of class-level geometry across independent training runs.

\section{Isometric Centroid-Encoder (ICE)}
\label{ICE}
\begin{figure}[!ht]
    \centering
    \includegraphics[width=0.95\linewidth]{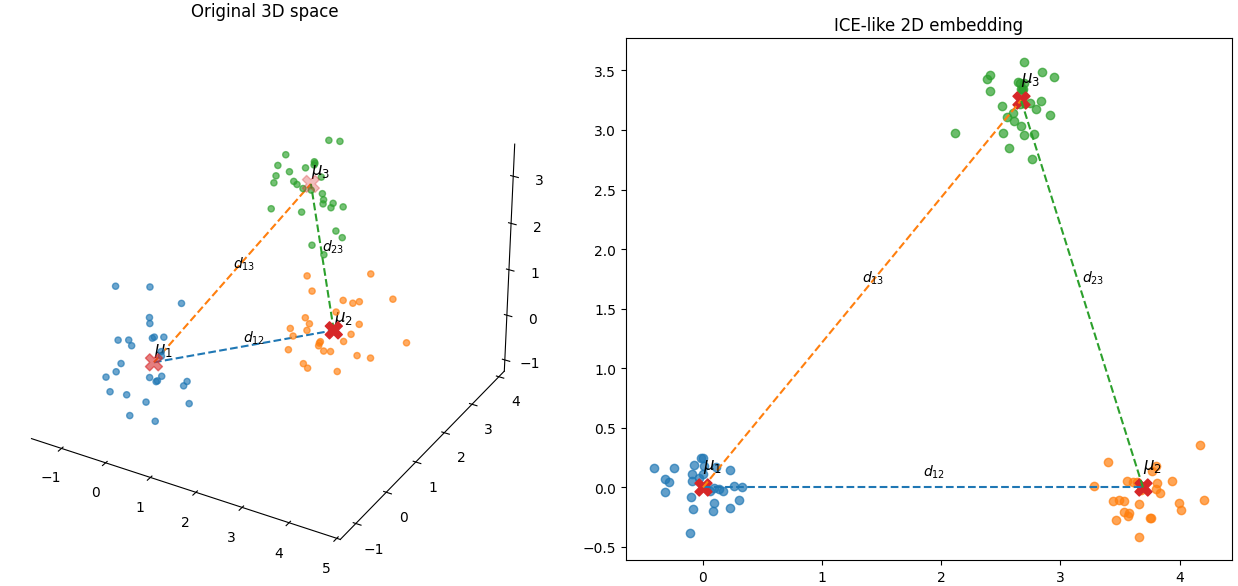}
    \caption{Illustration of Isometric Centroid-Encoder (ICE). Samples are dispersed around class centroids in 3D, with pairwise distances encoding class geometry. ICE maps 
 data to a 2D space by collapsing samples to centroids while preserving these distances, producing compact classes with retained geometric structure.}
 \vspace{0.25 cm}
    \label{fig:ICE_illustration}
\end{figure}
\subsection{Problem Setup:} 
Let $X \in \mathbb{R}^{d \times n}$ denote a data matrix consisting of $n$ samples, where each sample $x_i \in \mathbb{R}^d$. Assume the dataset contains $M$ classes $\{C_j\}_{j=1}^M$, where $I_j$ denotes the index set of samples belonging to class $C_j$. We seek a nonlinear mapping $g : \mathbb{R}^d \rightarrow \mathbb{R}^p$, with $p \ll d$, parameterized by $\theta$, which maps samples into a low-dimensional latent space. 
Let, $z_i = g(x_i;\theta) \in \mathbb{R}^p $
denote the latent representation of sample $x_i$.

\noindent{\textbf{Latent class centroids:}}
For each class $C_j$, we define its \emph{latent-space centroid} as
\begin{equation}
\mu_j = \frac{1}{|C_j|} \sum_{i \in I_j} z_i
      = \frac{1}{|C_j|} \sum_{i \in I_j} g(x_i;\theta),
\label{eq:latent_centroid}
\end{equation}
where $|C_j|$ denotes the cardinality of class $C_j$. 
In practice, the latent centroids $\{\mu_j\}_{j=1}^M$ are recomputed at each training epoch.

\noindent{\textbf{Centroid distance matrix:}}
Let $D \in \mathbb{R}^{M \times M}$ denote the pairwise centroid distance matrix computed in the ambient space, with entries $D_{jk} = \mathrm{dist}(c_j, c_k)$, where $c_j, c_k$ are the ambient-space centroids of class $C_j$ and $C_k$ correspondingly, and $\mathrm{dist}(\cdot,\cdot)$ denotes the Euclidean distance. 
The matrix $D$ serves as a reference geometry that we seek to preserve in the latent space.

\subsection{Objective Function:}
The ICE is designed to satisfy the following two objectives:\\
\noindent \textbf{1. Centroid consistency:} samples belonging to the same class should be close to their latent class centroid, i.e., 
\[g(x_i;\theta) \approx \mu_j \quad \text{for all } i \in I_j;\]

\vspace{0.15 cm}
\noindent \textbf{2. Centroid isometry:} pairwise distances between latent class centroids should preserve the ambient-space centroid geometry, i.e.,
 \[D_{jk} \approx \mathrm{dist}(\mu_j, \mu_k).\]

The first objective enforces intra-class compactness in the latent space, while the second objective encourages an approximately isometric embedding of class centroids. It is important to note that, unlike many supervised dimensionality reduction methods that draw inspiration from LDA-type objectives and explicitly enforce class separation in the embedding space, our approach takes a different route. Rather than pushing class centroids apart, we seek to \emph{transport the ambient-space class geometry} into the latent space by preserving pairwise centroid relationships. As a result, ICE does not impose an artificial separation between classes, but instead maintains the intrinsic similarity structure among them.

\noindent{\textbf{ICE cost function:}}
Combining these objectives yields the proposed \emph{Isometric Centroid-Encoder (ICE)} loss:
\vspace{0.1 cm}
\begin{equation}
\begin{aligned}
\mathcal{L}_{\mathrm{ICE}}(X,\theta)
=
\underbrace{
\frac{1}{2} \sum_{j=1}^M \sum_{i \in I_j}
\| g(x_i;\theta) - \mu_j \|_2^2
}_{\text{Centroid-consistency }(\mathcal{L}_{\mathrm{cons}})}
+
\underbrace{
\frac{\lambda}{2} \sum_{\substack{j,k \\ j \neq k}}
\big\| D_{jk} - \mathrm{dist}(\mu_j, \mu_k) \big\|_2^2
}_{\text{Centroid-isometry }(\mathcal{L}_{\mathrm{iso}})}.
\end{aligned}
\label{equation:MCELoss}
\end{equation}

In practice, the mapping $g(\cdot;\theta)$ is implemented as the encoder of a neural network, and the loss in Eq.~\eqref{equation:MCELoss} is minimized using stochastic gradient-based optimization. During each optimization step, the latent class centroids are recomputed from the current encoder outputs and used to evaluate both the centroid-consistency and centroid-isometry objectives before updating the encoder parameters. This formulation transfers the geometric class structure from the ambient space into a robust and stable low-dimensional latent space while enforcing intra-class compactness.
Although the present work uses Euclidean distances between ambient-space class centroids, the ICE formulation is not restricted to this choice. Any user-defined symmetric class-distance matrix may be used to define the reference geometry, including distances computed after feature preprocessing or from domain-specific similarity measures.

The proposed ICE model is trained using mini-batch stochastic optimization. The key steps are summarized in Algorithm~\ref{supp:alg:ice} which highlights that centroids are computed dynamically in the latent space and influence the optimization through both consistency and geometry-preserving terms.

\begin{algorithm}[!ht]
\begin{center}
\fbox{
\begin{minipage}{0.85\linewidth}

\KwIn{Training data $\{(x_i, y_i)\}_{i=1}^n$, encoder $g(\cdot;\theta)$, regularization parameter $\lambda$}
\KwOut{Learned encoder parameters $\theta$}

Compute class centroids in the ambient space:
$c_j = \frac{1}{|I_j|} \sum_{i \in I_j} x_i$\;

Compute centroid distance matrix:
$D_{jk} = \mathrm{dist}(c_j, c_k)$\;

Initialize encoder parameters $\theta$\;

\For{each training epoch}{
    \For{each mini-batch $\mathcal{B}$}{
        Compute latent representations:
        $z_i = g(x_i; \theta)$ for all $x_i \in \mathcal{B}$\;

        Compute class centroids in latent space:
        $\mu_j = \frac{1}{|\mathcal{B}_j|} \sum_{i \in \mathcal{B}_j} z_i$\;

        Compute centroid consistency loss, 
        $\mathcal{L}_{\text{cons}}$ from Equation~\eqref{equation:MCELoss}

        Compute centroid distance loss, 
        $\mathcal{L}_{\text{iso}}$ from Equation~\eqref{equation:MCELoss}

        Compute total loss:
        $\mathcal{L} = \mathcal{L}_{\text{cons}} + \lambda \mathcal{L}_{\text{iso}}$\;

        Update $\theta$ using Adam optimizer\;
    }
}
\end{minipage}
}
\end{center}
\caption{Training Procedure for ICE}
\label{supp:alg:ice}
\end{algorithm}

\subsection{Comparison with vanilla Centroid-Encoder:}

The proposed ICE framework differs from the classical Centroid-Encoder (CE)~\citep{ghosh2022supervised} in three fundamental aspects (Table~\ref{tab:ce_vs_ice}). First, CE maps each sample to a \emph{fixed} class centroid computed in the ambient space prior to training, whereas ICE learns \emph{latent-space centroids} that are dynamically recomputed from the current encoder outputs during optimization. Second, CE encourages only intra-class compactness and does not explicitly constrain the relative arrangement of class centroids in the embedding space. In contrast, ICE introduces a centroid-isometry regularization term that preserves the pairwise relationships among class centroids, thereby transferring the ambient-space class geometry into the latent space while simultaneously enforcing intra-class compactness. Finally, CE relies on an encoder-decoder architecture to reconstruct ambient-space centroids, whereas ICE operates entirely in the latent space using an encoder-only architecture. Consequently, ICE requires approximately half as many trainable parameters, making it attractive for high-dimensional, low-sample-size settings.

\vspace{0.1 cm}
\begin{table}[!ht]
\centering
\small
\setlength{\tabcolsep}{4pt}
\begin{tabular}{lcc}
\toprule
\textbf{Property} & \textbf{Centroid-Encoder (CE)} & \textbf{ICE} \\
\midrule
Centroid location 
& Fixed in ambient space 
& Learned in latent space \\

Target mapping 
& $x_i \rightarrow c_j \in \mathbb{R}^d$ 
& $x_i \rightarrow \mu_j \in \mathbb{R}^p$ \\

Centroid update 
& Pre-computed 
& Updated during training \\

Embedding constraint 
& Not enforced 
& Intra-class compactness + centroid isometry \\

Global class geometry 
& Not enforced 
& Explicitly preserved \\

Objective perspective 
& Centroid reconstruction 
& Centroid encoding + geometry preservation \\

Architecture 
& Encoder + decoder 
& Encoder only \\

Model complexity 
& Higher (due to decoder) 
& Lower \\

\bottomrule
\end{tabular}
\vspace{0.1 cm}
\caption{Comparison between Centroid-Encoder (CE) and the proposed ICE framework.}
\vspace{0.75 cm}
\label{tab:ce_vs_ice}
\end{table}

\vspace{0.25 cm}
\subsection{Theoretical Analysis}
\vspace{0.1 cm}
A key empirical observation across multiple datasets is that increasing the latent dimension $p$ beyond a certain threshold does not improve centroid stress (Kruskal stress) or Pearson correlation. This behavior can be explained by a fundamental property of Euclidean embeddings.
\vspace{0.1 cm}
\begin{lemma}[Intrinsic dimensionality of centroid geometry]
\label{lemma:centroid_dim}
Let $\{\mu_j\}_{j=1}^M \subset \mathbb{R}^p$ denote the latent class centroids produced by the encoder, where 
$
\mu_j = \frac{1}{|C_j|} \sum_{i \in I_j} g(x_i;\theta).
$
Then the affine span of $\{\mu_j\}_{j=1}^M$ has dimension at most 
$M-1$. Consequently, no embedding of dimension greater than $M-1$ can further improve the preservation of pairwise centroid distances.
\end{lemma}
\vspace{0.25 cm}
\begin{proof}
Let $\bar{\mu} = \frac{1}{M}\sum_{j=1}^M \mu_j$
denote the mean of the centroids, and define the centered vectors $\tilde{\mu}_j = \mu_j - \bar{\mu}, \qquad j=1,\dots,M.$
By construction, $\sum_{j=1}^M \tilde{\mu}_j = 0,$
so the centered vectors are linearly dependent. Hence the linear span of $\{\tilde{\mu}_j\}_{j=1}^M$ has dimension at most $M-1$. Since the affine span of $\{\mu_j\}_{j=1}^M$ is obtained by translating the linear span of the centered vectors, it also has dimension at most $M-1$. Therefore, the centroid configuration lies in an affine subspace of dimension at most $M-1$. Identifying this affine subspace with $\mathbb{R}^{r}$, where $r \le M-1$, yields an isometric representation of the centroids that preserves all pairwise Euclidean distances.
\end{proof}
\vspace{0.1 cm}
\noindent
Lemma~\ref{lemma:centroid_dim} implies that the centroid-distance regularizer in Eq.~\eqref{equation:MCELoss} cannot benefit from latent dimensions beyond $M-1$. 
In particular, once $p \ge M-1$, further increasing the latent dimension does not reduce the optimal value of the stress term. 
This provides a theoretical explanation for the empirical saturation of Kruskal stress and Pearson correlation observed in our experiments.
\vspace{0.1 cm}
\subsection{Evolution and Stabilization of Latent Class Centroids}
\label{subsec:centroid_convergence}
\vspace{0.1 cm}
Unlike conventional prototype-based methods, the latent class centroids in ICE are not fixed during optimization. Instead, they are dynamically recomputed from the encoder outputs after each training epoch and therefore evolve jointly with the learned representation. An important question is whether these dynamically updated centroids progressively stabilize or continue to drift throughout training.

\begin{figure}[!ht]
    \centering
    \includegraphics[width=1.00\linewidth]{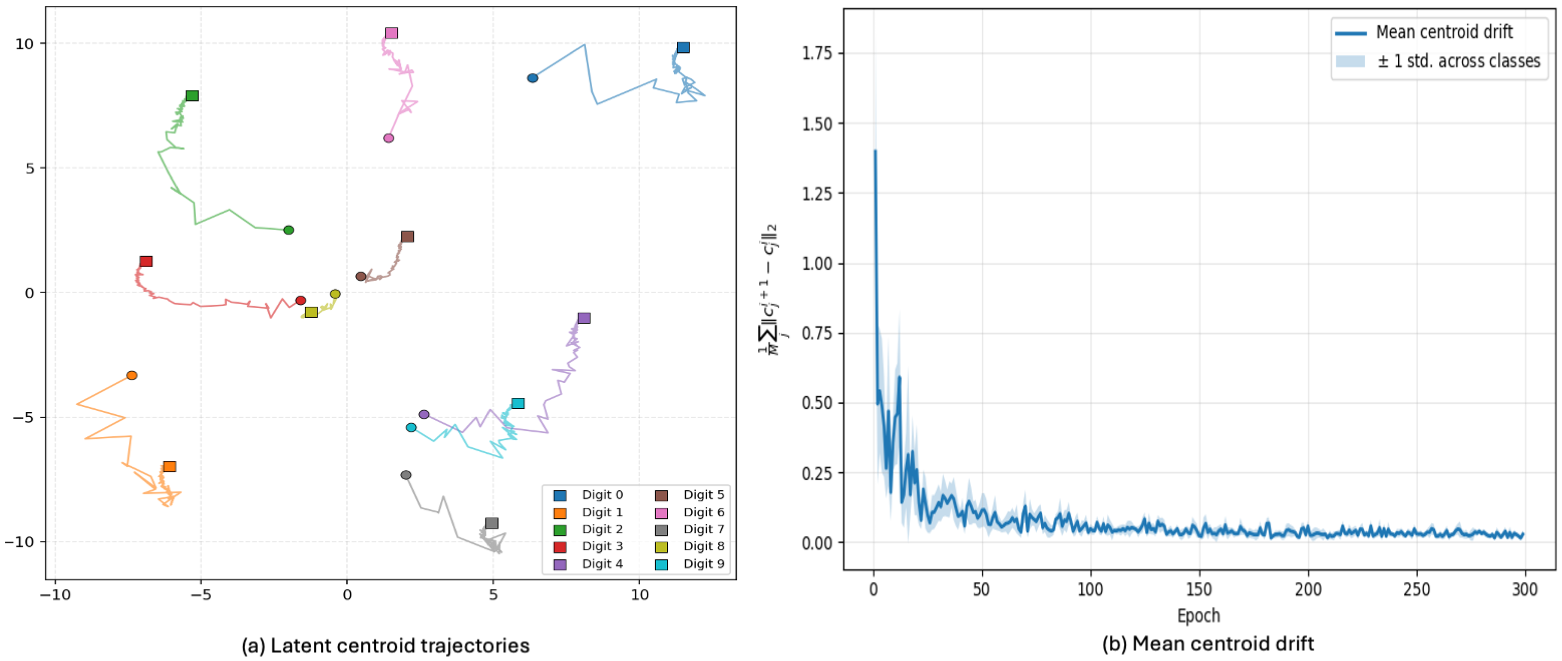}
    \vspace{1.0 cm}
    \caption{Optimization dynamics of ICE on the MNIST dataset. (a) Latent centroid trajectories in the two-dimensional embedding space. Circles and squares denote the initial and final centroid locations, respectively. Intermediate points illustrate the centroid evolution during training. (b) Mean epoch-to-epoch centroid drift averaged across all classes. The shaded region indicates one standard deviation across classes.}
    \vspace{0.25 cm}
    \label{fig:trajectory_and_drift}
\end{figure}

Figure~\ref{fig:trajectory_and_drift} illustrates the evolution of the latent class centroids during ICE optimization on MNIST. Panel (a) shows the centroid trajectories in the two-dimensional embedding space, while panel (b) quantifies the corresponding average epoch-to-epoch centroid drift,
$
\frac{1}{M}\sum_{j=1}^{M}\left\|c_j^{(i+1)}-c_j^{(i)}\right\|_2,
$
where $c_j^{(i)}$ denotes the latent centroid of class $j$ after epoch $i$. The shaded region in panel (b) represents one standard deviation across classes and therefore summarizes the variability of centroid movement among the different classes.

At the beginning of training, several class centroids are located relatively close because of the random initialization of the encoder. As optimization proceeds, they progressively reorganize into distinct locations while approaching the prescribed class-level geometry. For example, digit classes whose centroids are initially close, such as 4 and 9 or 3, 5, and 8, become increasingly separated during training. The centroid-drift curve provides a quantitative view of this behavior. The average drift is largest during the early stages of optimization, reflecting substantial changes in the latent representation as the class geometry is established. The drift then decreases substantially and reaches a small residual level in later epochs, indicating progressive stabilization of the dynamically recomputed latent centroids. Small nonzero fluctuations remain because ICE is trained using stochastic mini-batch optimization.

Together, the trajectory and drift analyses show that the latent centroids evolve toward a stable class configuration rather than exhibiting persistent oscillation or collapse. This observation is also consistent with the ablation analysis in Section~\ref{subsec:ablation}, which identifies the centroid-isometry term as the primary component responsible for preserving and stabilizing class-centroid geometry. 

A similar analysis on FMNIST, yielded the same behavior of latent centroid trajectories and drift as observed on MNIST. As optimization proceeds, the centroids progressively reorganize into distinct regions while preserving the prescribed class-level geometry. We can observe the classes Dress, T-shirt, Shirt, Coat, Pullover or Sandal, Boot, Ankle boot are becoming more distant as the training progresses. Figure~\ref{supp:fig:fmnist_drift} reports the corresponding results for FMNIST.
\begin{figure}[!ht]
    \centering
    \includegraphics[width=1.0\linewidth]{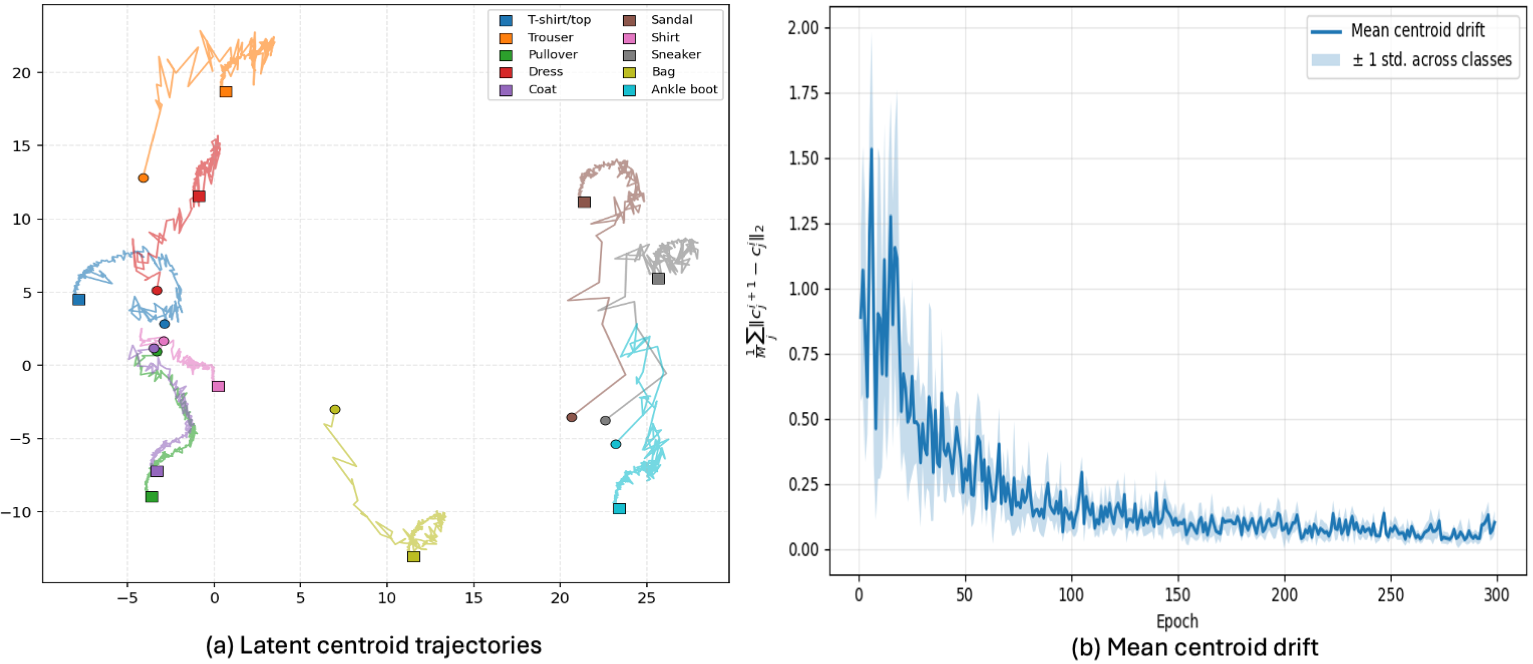}
    \caption{
    Optimization dynamics of ICE on the FMNIST dataset. (a) Evolution of the latent class centroids in the two-dimensional embedding space throughout training. (b) Mean epoch-to-epoch centroid drift averaged across all classes. The shaded region denotes one standard deviation across classes.
    }
    \label{supp:fig:fmnist_drift}
\end{figure}

Basing the results on MNIST and FMNIST, these experiments demonstrate that the dynamically recomputed latent centroids progressively stabilize during optimization while maintaining the prescribed class-centroid geometry. The consistent behavior observed across both image datasets further supports that the centroid-isometry objective promotes stable class-level geometry rather than transient centroid configurations.

\subsection{Training loss analysis:}
The latent class centroids in ICE are dynamically recomputed from the encoder outputs throughout optimization. Since the centroid-consistency and centroid-isometry objectives are optimized jointly, an important question is whether the optimization converges to a stable solution and how the available latent dimensionality influences the optimization process.

\begin{figure}[!ht]
    \centering
    \includegraphics[width=1.00\linewidth]{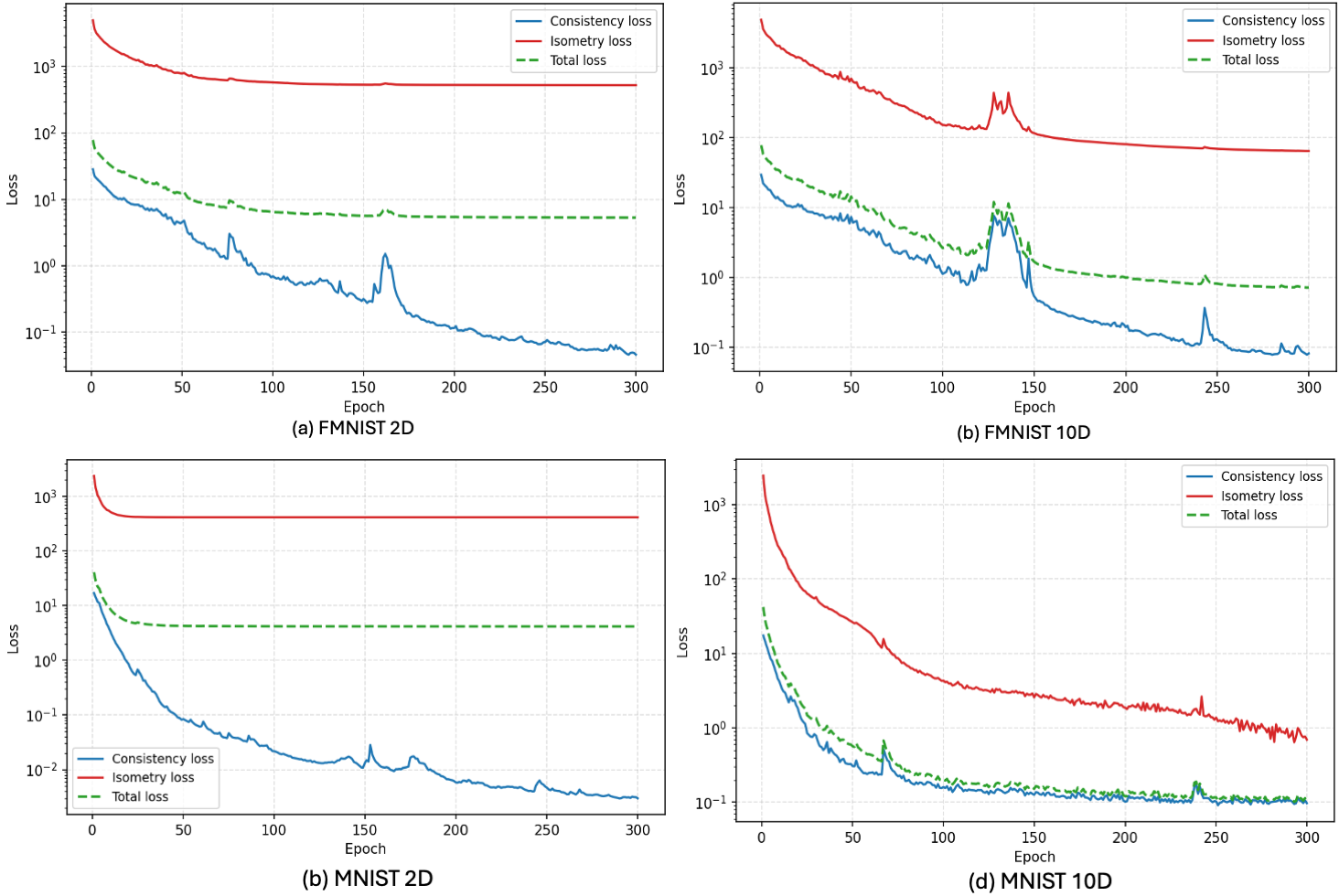}
    \caption{
Evolution of the centroid-consistency loss, centroid-isometry loss, and total objective during ICE optimization on FMNIST and MNIST using direct two-dimensional embeddings and class-dimensional embeddings ($p=M$). The total objective is computed as $\mathcal{L}_{\mathrm{cons}}+\lambda\mathcal{L}_{\mathrm{iso}}$ with $\lambda=0.01$, while the isometry loss is shown in its unweighted form. The plots show that the isometry loss quickly plateaus in the two-dimensional embedding but continues to decrease in the higher-dimensional latent space, consistent with the additional representational capacity available for preserving centroid geometry.
}
    \label{supp:fig:loss_convergence}
\end{figure}

Figure~\ref{supp:fig:loss_convergence} presents the evolution of the centroid-consistency loss, centroid-isometry loss, and the total objective for FMNIST and MNIST using both direct two-dimensional embeddings and class-dimensional embeddings ($p=M$). The total loss is computed as
\[
\mathcal{L}_{\mathrm{total}}
=
\mathcal{L}_{\mathrm{cons}}
+
\lambda
\mathcal{L}_{\mathrm{iso}},
\]
with $\lambda=0.01$. The centroid-isometry loss is plotted in its unweighted form to illustrate its optimization behavior.

Several observations can be made. First, the centroid-consistency loss decreases steadily in all experiments, indicating that samples progressively move toward their corresponding latent class centroids. Although small fluctuations are observed during training due to stochastic optimization, the overall trend is monotonic and converges to a stable value.

More importantly, the centroid-isometry loss exhibits markedly different behavior depending on the latent dimensionality. For the direct two-dimensional embedding (Figure~\ref{supp:fig:loss_convergence}(a) and (c)), the isometry loss decreases rapidly during the initial training epochs but soon reaches a plateau. This indicates that once the available degrees of freedom in the latent space are exhausted, the optimization can no longer substantially reduce the discrepancy between the latent and reference centroid distances. Consequently, subsequent optimization primarily improves within-class compactness while only marginally reducing the centroid-isometry objective.

In contrast, when the latent dimensionality equals the number of classes ($p=M$), the centroid-isometry loss continues to decrease throughout training (Figure~\ref{supp:fig:loss_convergence}(b) and (d)). The additional representational capacity allows the optimizer to progressively refine the global centroid configuration while simultaneously maintaining within-class compactness. This behavior is consistent across both FMNIST and MNIST and is reflected by the substantially lower final isometry loss compared to the corresponding two-dimensional embeddings.

These observations are consistent with the geometric analysis presented in Lemma~\ref{lemma:centroid_dim}. While a two-dimensional embedding provides insufficient degrees of freedom to accurately preserve all pairwise centroid relationships, the higher-dimensional latent space provides sufficient flexibility for the centroid-isometry objective to be optimized more effectively. Thus, the convergence behavior itself provides empirical evidence supporting the practical implications of the proposed latent-dimension selection strategy.

\section{Centroid Neighborhood Stability (CNS)}
\label{sec:cns}
This section introduces a new stability measure for class-level geometry, called \emph{Centroid Neighborhood Stability (CNS)}, a metric based on the agreement of nearest-neighbor relationships among class centroids across multiple independent training runs. Unlike stress, correlation, or classification accuracy, which evaluate the quality of a single embedding, CNS quantifies the \emph{reproducibility} of class-level neighborhood structure across repeated executions of the same model. Consequently, CNS complements geometry-preservation measures such as Kruskal Stress, and Procrustes Disparity by evaluating a different property, namely the reproducibility of class-level neighborhood structure across repeated runs. 

Let a DR model be trained independently for $R$ runs with different random initializations. For run $r \in \{1,\dots,R\}$, let $z_j^{(r)}$
denote the embedded centroid of class $j$, and let $\mathcal{N}_k^{(r)}(j)$ denote the set of $k$ nearest neighboring centroids of $z_j^{(r)}$ in the latent space, excluding $j$ itself. For a fixed reference class $j$, define the multiset, $\mathcal{M}_j = \bigcup_{r=1}^R \mathcal{N}_k^{(r)}(j)$, which contains $kR$ centroid indices counted with multiplicity. Let $\mathcal{U}_j$ denote the set of unique centroid indices appearing in $\mathcal{M}_j$. For each $u \in \mathcal{U}_j$, let $n(u)$ denote the number of times centroid $u$ appears in $\mathcal{M}_j$. By construction, $\sum_{u \in \mathcal{U}_j} n(u) = kR$. %
We therefore define the class-wise CNS score as:
\begin{equation}
\vspace{0.25 cm}
S_j = \frac{1}{kR^2} \sum_{u \in \mathcal{U}_j} n(u)^2.
\label{eq:class_stability}
\end{equation}
The overall stability score is obtained by averaging over all classes:
\begin{equation}
\vspace{0.25 cm}
\mathrm{CNS} = \frac{1}{M} \sum_{j=1}^M S_j.
\label{eq:overall_stability}
\end{equation}

Unlike pairwise distance measures, CNS depends only on neighborhood membership and is therefore insensitive to small perturbations that preserve the local class-neighborhood structure.

\noindent{\textbf{Interpretation:}}
The CNS score ($S_j$) measures the agreement of nearest-centroid selections across runs. If the same $k$ neighbors are consistently recovered for class $j$ in every run, then $n(u)=R$ for each of the $k$ neighbors and $S_j=1$. Conversely, if neighborhood assignments vary maximally across runs, the counts $n(u)$ are evenly dispersed and $S_j$ approaches its minimum value of $1/R$. A value of $S_j$ close to one indicates that the local neighborhood of class $j$ is highly reproducible across independent runs, whereas values approaching the theoretical lower bound indicate substantial variation in the neighboring classes recovered by different executions of the model. Intermediate values indicate partial reproducibility of class-neighborhood relationships.

\noindent{\textbf{Illustrative example:}}
Consider the experimental setting with $k=3$ nearest class neighbors and $R=5$ independent runs, yielding $kR=15$ neighbor selections per class.

\noindent{\textbf{Maximum value}} If the same three neighboring classes are recovered in every run, each neighbor appears $R=5$ times, giving\vspace{-0.25 cm}
$ \sum_{u \in \mathcal{U}_j} n(u)^2 = 5^2 + 5^2 + 5^2 = 75,
\qquad S_j = \frac{75}{3 \times 25} = 1$.
\vspace{0.1 cm}

\noindent{\textbf{Minimum value:}}
Since neighbor selections are drawn from the $M-1$ classes other than the reference class $j$, the number of distinct neighbors cannot exceed $M-1$. Let $T=kR$ denote the total number of neighbor selections across all runs and define
$L=\min(T,M-1)$, which is the maximum possible number of distinct neighbors in $\mathcal{M}_j$. The minimum value of $\sum_{u\in\mathcal U_j}n(u)^2$ is attained when the $T$ selections are distributed as evenly as possible across these $L$ neighbors. Writing
$T=aL+b,\qquad a=\Bigl\lfloor\frac{T}{L}\Bigr\rfloor,\qquad 0\le b<L,$
the most even allocation consists of $L-b$ neighbors appearing $a$ times and $b$ neighbors appearing $a+1$ times. Therefore, $S_j^{\min}=\frac{(L-b)a^2 + b(a+1)^2}{kR^2}.$ This minimum value corresponds to the most dispersed possible neighborhood assignments and therefore represents the least reproducible class-neighborhood configuration for the given values of $k$, $R$, and $M$.

\noindent{\textbf{Illustrative examples:}}
For the experimental setting used in this paper ($k=3$, $R=5$, hence $T=15$), the lower bound depends on the number of classes. If $M \ge 16$ (e.g., the Mouse Bladder dataset, where $M-1 = 15$), then $L = T = 15$, yielding $a=1$ and $b=0$. In this case, all neighbors can be distinct and$S_j^{\min} = \frac{15}{3 \times 25} = 0.2$. In contrast, for datasets with fewer classes such as USPS ($M=10$, so $M-1=9$), we have $L=9$, $a=1$, and $b=6$, leading to $S_j^{\min} = \frac{27}{3 \times 25} = 0.36$. These examples illustrate that the theoretical lower bound depends on the number of available classes as well as the choices of $k$ and $R$. Increasing $k$ shifts CNS from measuring immediate local neighborhoods toward broader class-neighborhood structure, whereas increasing $R$ improves the reliability of the stability estimate without changing the underlying neighborhood definition. Throughout this work, we fix $k=3$ to capture local class relationships and use $R=5$ as a computationally practical choice. A sensitivity analysis of CNS with respect to $k$ and $R$ is provided in Section~\ref{supp:sec:supp_cns_sensitivity}. Although the absolute CNS values vary with these parameters, the overall comparative conclusions remain stable across the tested configurations.

\noindent{\textbf{Invariance and scope:}}
CNS is invariant to global rigid transformations (translation, rotation, and reflection) because it depends only on relative centroid neighborhoods. Consequently, embeddings related solely by a rigid transformation receive identical CNS values. Unlike traditional embedding quality measures and Procrustes-based approaches \citep{iwata2026making}, which assess the quality or geometric similarity of embeddings, CNS measures the reproducibility of class-level neighborhood structure across independent runs. Thus, CNS complements rather than replaces geometry-preservation measures such as Kruskal Stress, Pearson Correlation, and Procrustes Disparity. A controlled synthetic experiment presented in Section~\ref{supp:sec:supp_cns_procrustes} demonstrates this distinction: identical geometric embeddings with permuted class identities yield zero Procrustes Disparity but substantially different CNS values, confirming that the two metrics capture complementary rather than redundant properties of embedding similarity.

\subsection{CNS versus Procrustes Disparity: A Controlled Experiment}
\label{supp:sec:supp_cns_procrustes}

To prove our hypothesis that CNS provides information beyond existing embedding similarity measures such as Procrustes Disparity. We designed a controlled synthetic experiment that isolates the fundamental difference between the two metrics.

\begin{figure}[!ht]    
    \includegraphics[width=15.0 cm, height=15.0 cm]{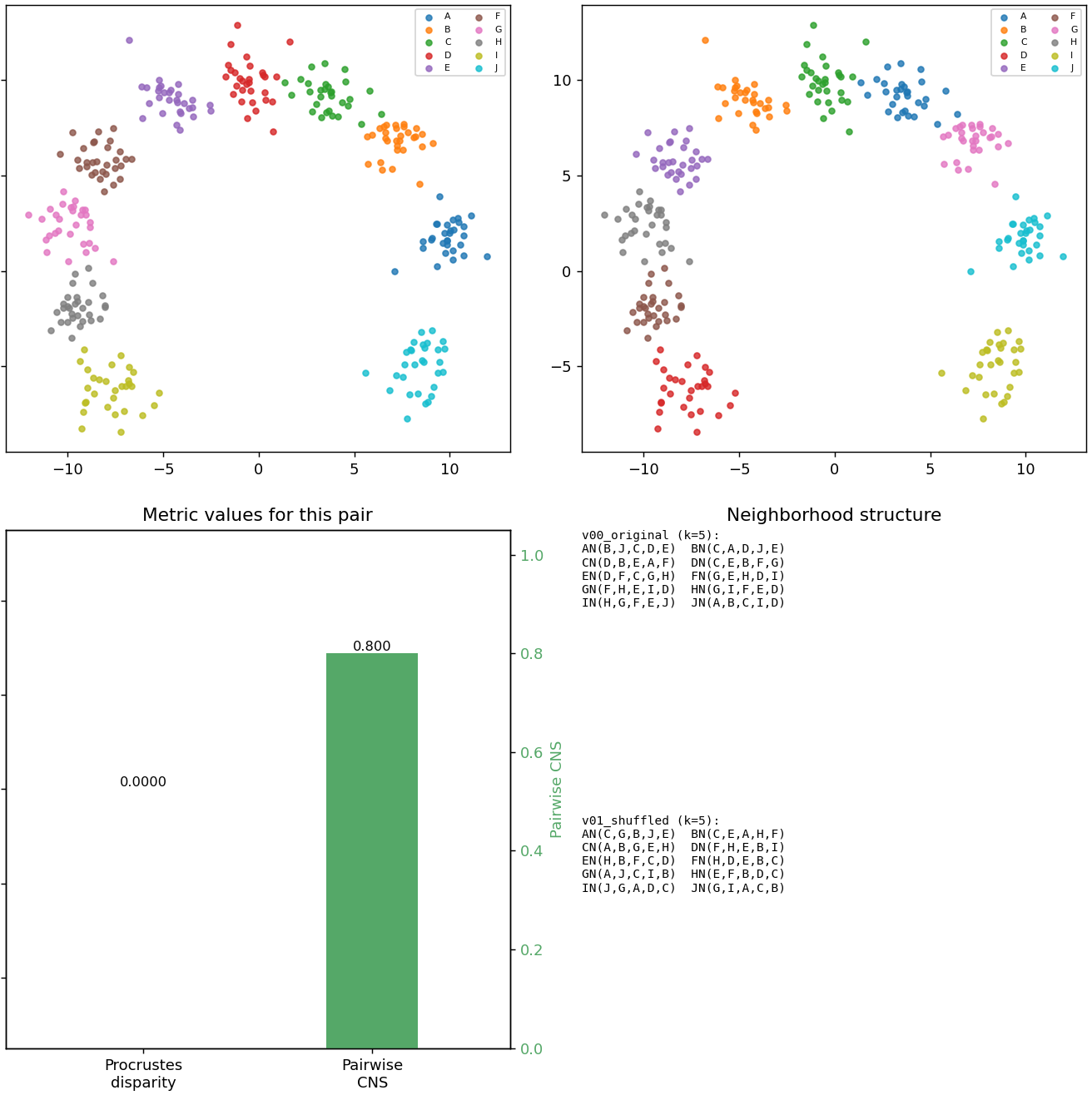}
    \caption{Representative comparison from the controlled synthetic experiment designed to distinguish Procrustes Disparity and CNS. Two configurations with identical embedding coordinates but different class-label assignments are compared. (a) Reference configuration. (b) Configuration obtained by permuting the labels assigned to entire clusters while keeping all sample coordinates fixed. (c) Corresponding Procrustes Disparity and CNS values. (d) $k=5$ nearest-centroid neighborhood graphs illustrating how class-label permutations alter class-neighborhood relationships despite identical embedding geometry. Procrustes Disparity remains zero because the geometric configuration is unchanged, whereas CNS captures the resulting change in class-level neighborhood topology.}
    \label{supp:fig:cns_procrustes_example}
\end{figure}

To obtain a comprehensive comparison, every pair of configurations was evaluated, resulting in
\[
\binom{10}{2}=45
\]

A two-dimensional synthetic dataset consisting of ten well-separated classes (A--J) was first generated, where the class centroids were arranged approximately on a circle. This reference embedding was then used to construct nine additional configurations. Rather than modifying the embedding coordinates, the class labels assigned to entire clusters were progressively permuted while keeping every sample at exactly the same location. Consequently, all ten configurations possessed identical geometric embeddings, differing only in the semantic assignment of class labels to the fixed clusters.

Figure~\ref{supp:fig:cns_procrustes_example} illustrates one representative comparison. The left two panels show two configurations with identical sample coordinates but different class-label assignments. The third panel reports the corresponding Procrustes Disparity and CNS values, while the fourth panel visualizes the $k=5$ nearest-centroid neighborhood graph for both configurations.

pairwise comparisons. For each pair, Procrustes Disparity was computed directly from the embedding coordinates using the same implementation employed throughout this work. CNS was computed by comparing the $k=5$ nearest-centroid neighborhoods between the two configurations.

Importantly, the coordinates of every embedded sample remained identical across all configurations; only the class labels assigned to entire clusters were permuted. Consequently, Procrustes Disparity remained zero (up to numerical precision) for every one of the 45 pairwise comparisons, indicating perfect geometric alignment.

In contrast, CNS exhibited substantial variation across the same set of comparisons, with a mean value of 0.784, a standard deviation of 0.031, and values ranging from 0.73 to 0.87. The observed variation reflects the extent to which each label permutation altered the neighborhood relationships among identified classes.

These results demonstrate that Procrustes Disparity depends solely on the embedding coordinates and is therefore inherently label-agnostic, measuring only geometric alignment. In contrast, CNS is label-aware and measures the consistency of class-neighborhood topology across embeddings. The two metrics therefore quantify complementary rather than redundant properties of embedding similarity.

This experiment intentionally decouples geometric structure from semantic class relationships. Since the geometry remains unchanged throughout the experiment, Procrustes Disparity correctly reports perfect alignment for every comparison. Conversely, because the neighborhood relationships among identified classes are systematically altered through class-label permutations, CNS responds proportionally to these changes. The experiment therefore provides direct evidence that CNS does not simply re-express geometric alignment measured by Procrustes Disparity, but instead quantifies a distinct aspect of embedding consistency based on class-level neighborhood topology. Finally, CNS compares unordered $k$-nearest-neighbor sets and therefore evaluates neighborhood membership rather than neighbor ordering, making it well suited for assessing the reproducibility of stochastic embedding algorithms.

\section{Experimental Setup}
\label{exp_setup}
This section describes the datasets, baseline methods, experimental protocol, and benchmarking results for the proposed ICE.
\subsection{Datasets}
We use a diverse set of publicly available datasets from multiple domains, summarized in Table~\ref{table:dataDescription}. The image datasets (MNIST, FMNIST, and USPS) provide well-established benchmarks for evaluating supervised dimensionality reduction, whereas the biological datasets are high-dimensional with a larger number of classes, offering a more challenging assessment of scalability, stability, and class-geometry preservation.

\begin{table}[!ht]
	\centering
	\caption{Descriptions of datasets used for benchmarking.}
	\vspace{0.1 cm}
	\begin{tabular} {|c|c|c|c|c|}	
		\hline
		Dataset & \#Features & \#Classes & \#Samples & Domain \\
		\hline
        USPS & 256 & 10 & 9298 & Image \\
        MNIST & 784 & 10 & 60000 & Image \\
        FMNIST & 784 & 10 & 60000 & Image \\
        10X PBMC & 16653 & 8 & 4271 & RNA-Seq \\
        Mouse Bladder Cells & 20670 & 16 & 2746 & RNA-Seq \\
        Worm Neuron Cells & 13488 & 10 & 4186 & RNA-Seq \\
		\hline
	\end{tabular}
	\vspace{0.25 cm}
	\label{table:dataDescription}
\end{table}
\subsection{Baselines}
To evaluate the proposed ICE model, we compare it with representative supervised dimensionality reduction methods, including the neural network-based approaches Centroid-Encoder (CE) \citep{ghosh2022supervised}, Super-Encoder (SE) \citep{kamal2022super}, and IVIS \citep{szubert2019structure}, together with the supervised manifold learning methods supervised UMAP \citep{mcinnes2018umap} and class-constrained t-SNE (csTSNE) \citep{meng2023class}. These baselines span centroid-based, similarity-learning, and neighborhood-preserving approaches, providing a comprehensive comparison with ICE.
\subsection{Evaluation Metrics}
We evaluate embedding quality using complementary measures of global geometry, local neighborhood structure, and stability. In addition to the proposed Centroid Neighborhood Stability (CNS) (Section~\ref{sec:cns}), which quantifies the reproducibility of class-level neighborhood structure across independent runs, we report several standard metrics. \textbf{Kruskal's Stress} \citep{kruskal1964} measures pairwise distance preservation, \textbf{Pearson Correlation} \citep{benesty2009pearson} quantifies the linear agreement between pairwise distances in the original and latent spaces, and \textbf{Procrustes Disparity} \citep{iwata2026making} measures the geometric similarity between the original and embedded centroid configurations after optimal alignment. Local neighborhood quality is assessed using 5-nearest neighbor (\textbf{5NN}) classification accuracy in the embedding space. Together, these metrics evaluate geometric fidelity, local neighborhood preservation, and embedding stability.

\subsection{Experimental Protocol}
The ICE model is trained using the Adam optimizer with a learning rate of $10^{-3}$ for 300 epochs. A common network architecture is used across the neural network-based baselines (CE, SE, and IVIS) to ensure a fair comparison, while the remaining methods follow their recommended hyperparameter settings. Following the procedure given in Algorithm~\ref{supp:alg:ice}, and using the network architectures and dataset-specific hyperparameters that are listed in Table~\ref{supp:table:ICEHyperparameters1}. A regularization parameter $\lambda$ is selected by five-fold cross-validation on the training portion of the first 80:20 train-test split and then fixed for all 15 independent runs. To analyze model behavior, we consider three ICE variants: linear, shallow (one hidden layer with 500 units), and deep (1000, 500, 125 units). Classification performance is evaluated using a 5-nearest neighbor (5NN) classifier with an 80:20 train-test split, and all reported results are averaged over $R=15$ independent runs. Following the original implementation, csTSNE is applied to PCA-reduced features only when optimization in the original feature space is computationally prohibitive; no PCA preprocessing is used for the remaining methods.

\noindent\textbf{Implementation details for reproducibility:}
\label{supp:sec:supp_config}
The following details are required to reproduce the results in Tables~\ref{tab:2D_Results} and~\ref{tab:M-D_Results}. Table~\ref{supp:table:ICEHyperparameters1} summarizes the network architectures, activation functions, optimization settings, training epochs, learning rates, mini-batch sizes, and the corresponding values of the regularization parameter $\lambda$ used for each dataset. Unless otherwise stated, ICE was trained using the configurations listed in this table. To isolate the effect of the proposed centroid-consistency and centroid-isometry objectives, no additional regularization techniques, such as batch normalization, dropout, or weight decay, were employed in any of the reported experiments.

\begin{table}[!ht]
	\centering
	\begin{adjustbox}{max width=\linewidth}
		\begin{tabular} {|c|c|c|c|c|c|}	%
		\hline
		Dataset & Network topology & Activation & Epoch & Learning Rate & Mini Batch\\
		
		\hline
		MNIST & $d \rightarrow1000\rightarrow500\rightarrow125\rightarrow 2$ & relu & 300 & 0.001 & 512\\
		USPS & $ d\rightarrow1000\rightarrow500\rightarrow125\rightarrow 2$ & relu & 300 & 0.001 & 512\\
		FMNIST & $ d\rightarrow1000\rightarrow500\rightarrow125\rightarrow 2$ & relu & 300 & 0.001 & 512\\
		PBMC & $ d\rightarrow1000\rightarrow500\rightarrow125\rightarrow 2$ & relu & 300 & 0.001 & 512\\
		Bladder & $ d\rightarrow1000\rightarrow500\rightarrow125\rightarrow 2$ & relu & 300 & 0.001 & 512\\
		Neuron & $ d\rightarrow1000\rightarrow500\rightarrow125\rightarrow 2$ & relu & 300 & 0.001 & 512\\
		\hline
	\end{tabular}
	\end{adjustbox}
	\caption{Network topology and hyperparameter settings used for ICE across different datasets. The input dimension $d$ corresponds to the feature dimension of each dataset and is therefore dataset-specific.}
	\label{supp:table:ICEHyperparameters1}
\end{table}

\begin{table}[!ht]
\small
\adjustbox{max width=\linewidth}{\begin{tabular}{c|cc|cc|cc|cc|cc|cc}
\hline
 & \multicolumn{2}{c|}{MNIST} 
 & \multicolumn{2}{c|}{USPS}
 & \multicolumn{2}{c|}{FMNIST}
 & \multicolumn{2}{c|}{PBMC}
 & \multicolumn{2}{c|}{Bladder}
 & \multicolumn{2}{c}{Neuron} \\
\cline{2-13}
\textbf{} & $p$ & $\lambda$ & $p$ & $\lambda$ & $p$ & $\lambda$ & $p$ & $\lambda$ & $p$ & $\lambda$ & $p$ & $\lambda$ \\
\hline
Low-dim & 2 & 0.001 & 2 & 0.001 & 2 & 0.001 & 2 & 0.001 & 2 & 0.001 & 2 & 0.001 \\
High-dim & 10 & 100 & 10 & 100 & 10 & 100 & 8 & 100 & 16 & 100 & 10 & 100 \\
\hline
\end{tabular}}
\caption{Dataset-specific values of $\lambda$ used for each embedding dimension in the experiments reported in Tables~\ref{tab:2D_Results} and~\ref{tab:M-D_Results}.}
\label{supp:table:ICEHyperparameters2}
\end{table}

\section{Model Behavior and Design Implications}
\label{sec:ICE_analysis}
This section investigates the behavior of proposed ICE and the effects of latent dimensionality and how to choose it, the regularization parameter $\lambda$, the contribution of each loss term (ablation), the sensitivity of CNS to its parameters, and the effect of PCA preprocessing.

\subsection{Effect of Latent Dimensionality on Centroid Geometry}
\label{subsec:stress_saturation}

This experiment investigates how latent dimensionality influences the preservation of class-centroid geometry in ICE and whether this behavior depends on encoder depth. According to Lemma~\ref{lemma:centroid_dim}, the centroid configuration of $M$ classes has intrinsic dimensionality at most $M-1$. We therefore expect improvements in centroid-geometry preservation to diminish once the latent dimension approaches this limit.

\noindent\textbf{Experimental setup:}
We first evaluate the deep ICE architecture (1000, 500, and 125 hidden units) while varying the latent dimension from $p=2$ to $26$. For each dimension, we compute Kruskal Stress and Pearson Correlation between pairwise class-centroid distances in the original and latent spaces. Results are averaged over five independent runs.

\begin{figure}[!ht]
    \centering
    \begin{subfigure}[b]{0.45\linewidth}
        \includegraphics[width=\linewidth]{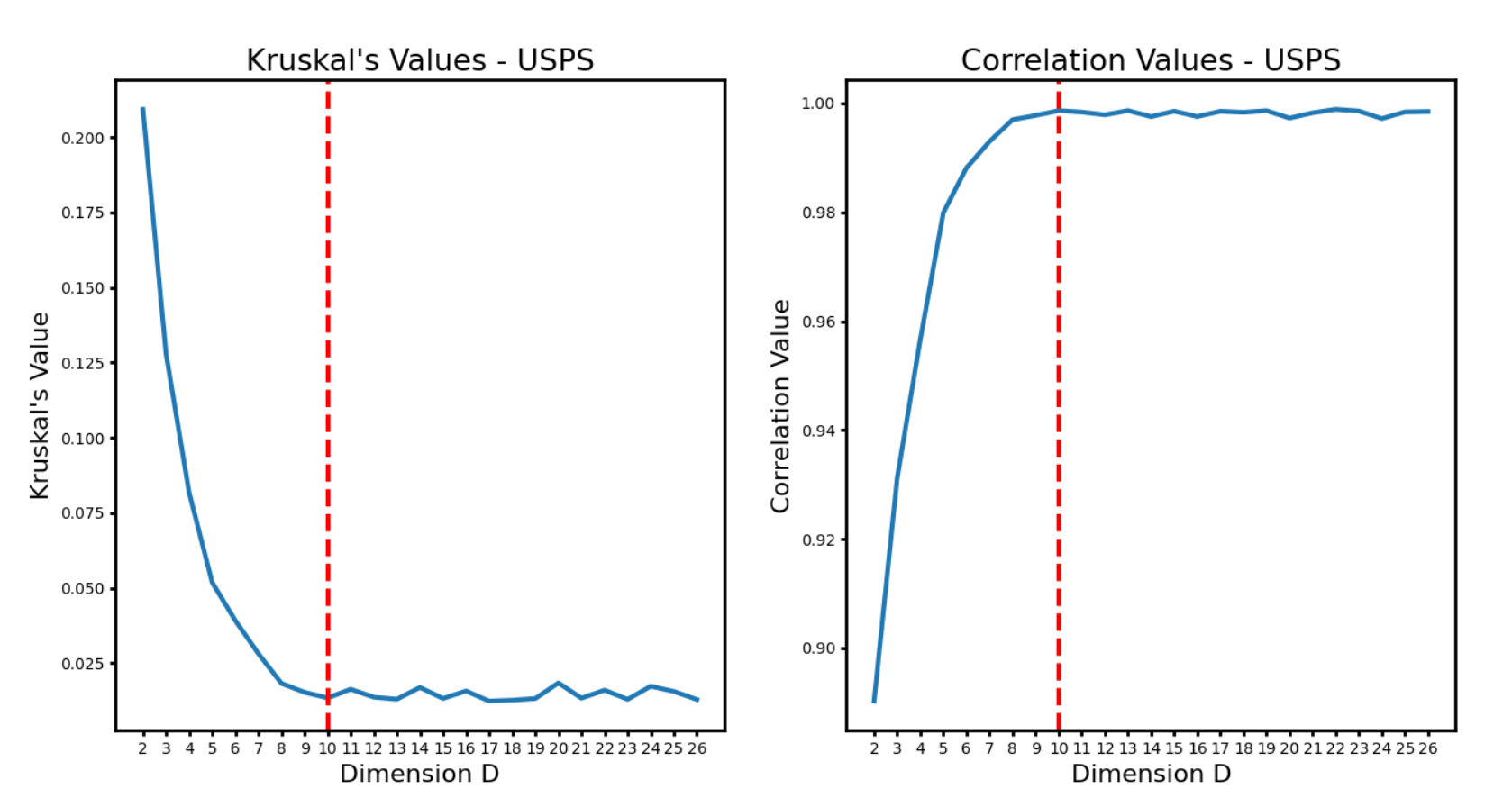}
        \caption{USPS}
        \label{fig:USPS_3HiddenLayer}
    \end{subfigure}
    \hfill
    \begin{subfigure}[b]{0.45\linewidth}
        \includegraphics[width=\linewidth]{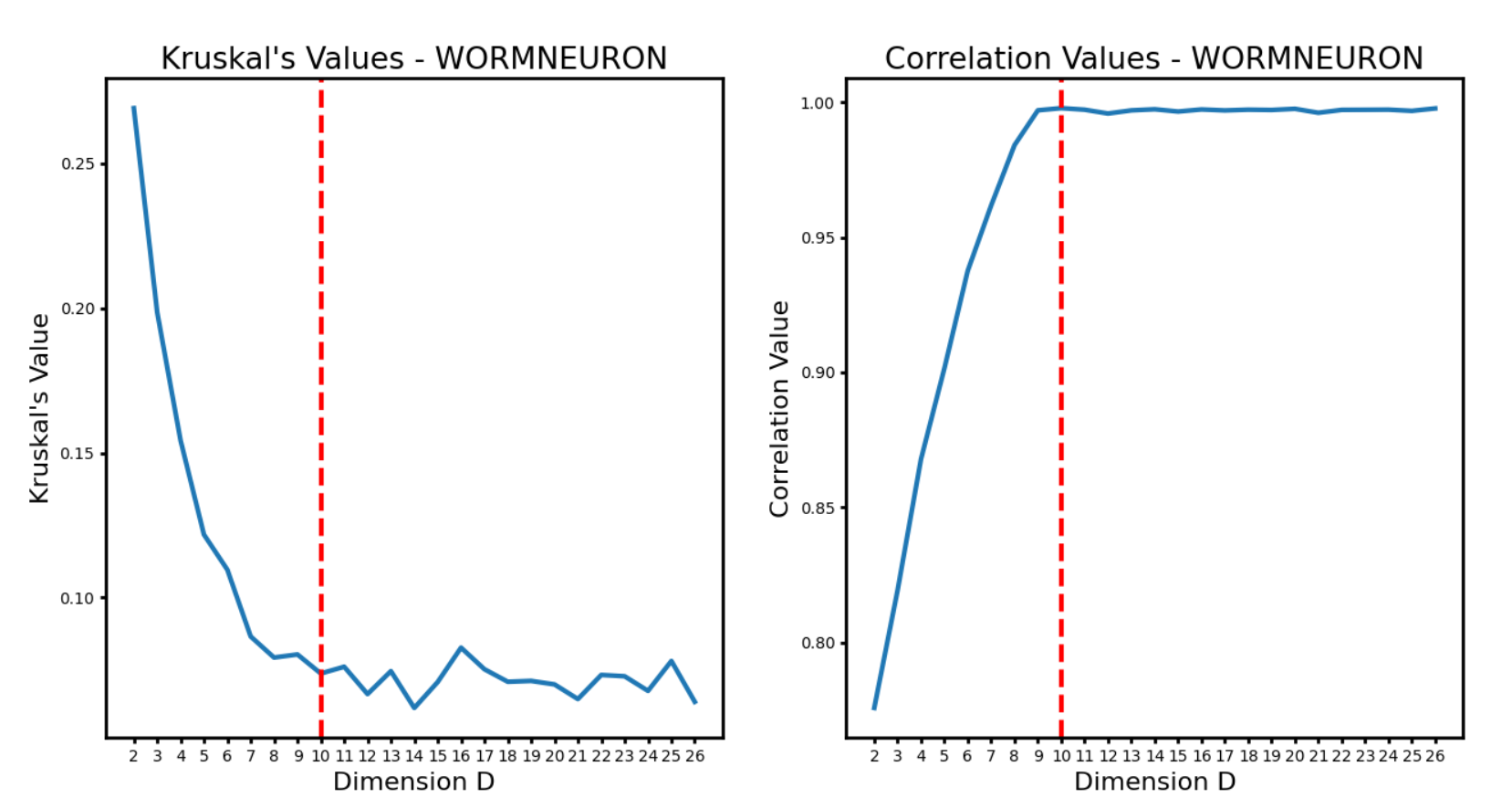}
        \caption{WORMNEURON}
        \label{fig:WORMNEURON_3HiddenLayers}
    \end{subfigure}

    \begin{subfigure}[b]{0.45\linewidth}
        \includegraphics[width=\linewidth]{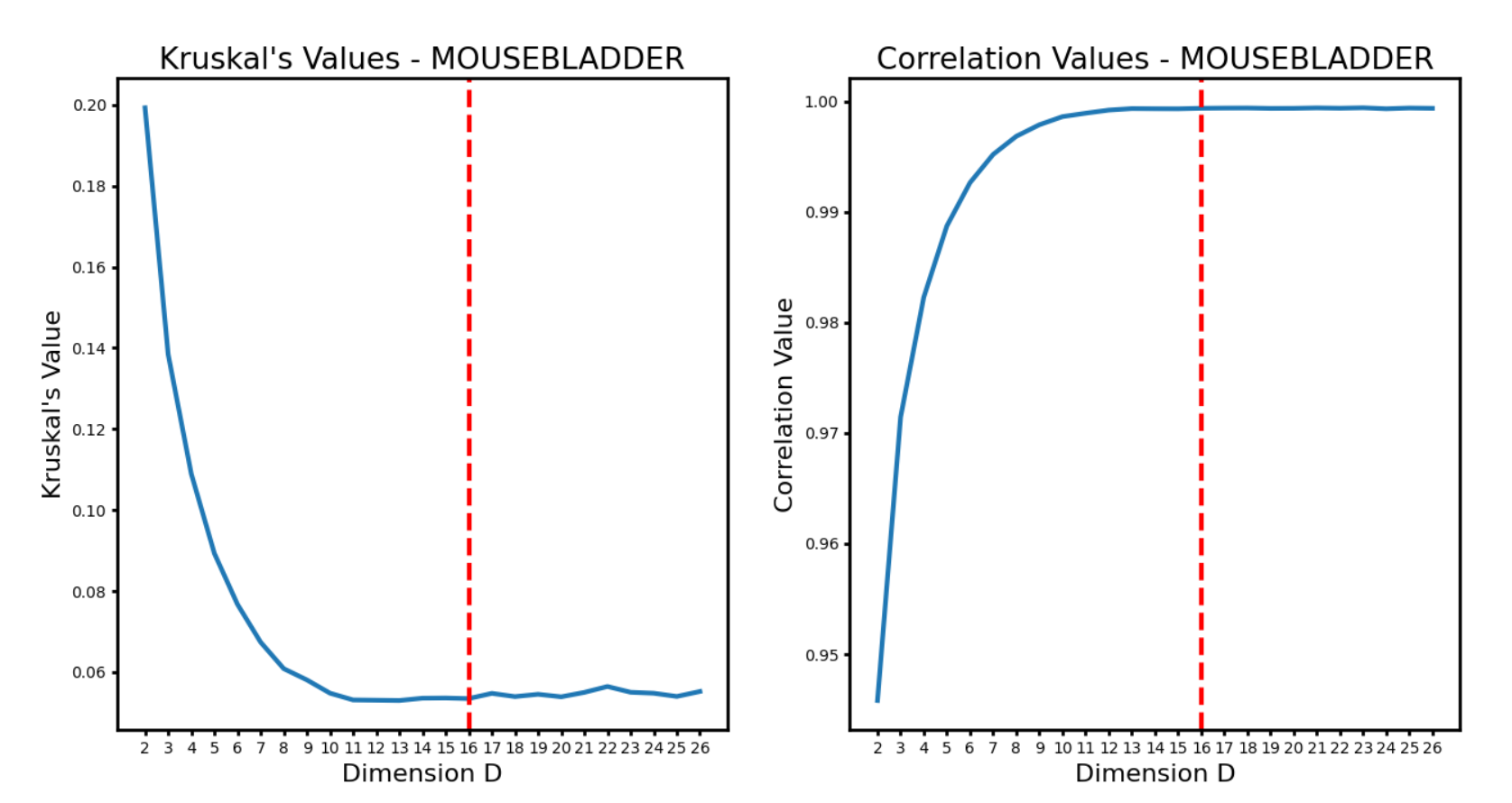}
        \caption{MOUSEBLADDER}
        \label{fig:MOUSEBLADDER_3HiddenLayer}
    \end{subfigure}
    \hfill
    \begin{subfigure}[b]{0.45\linewidth}
        \includegraphics[width=\linewidth]{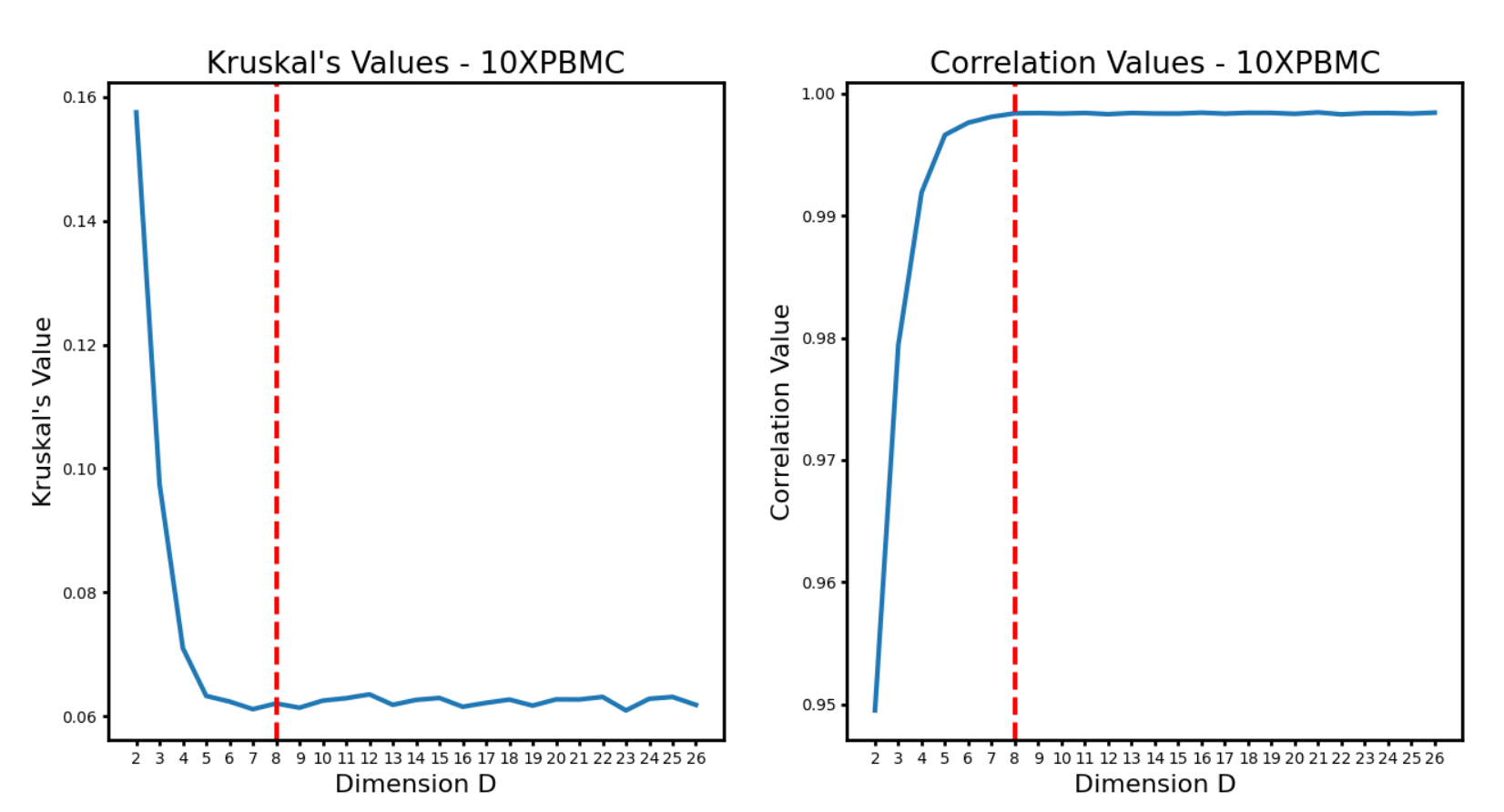}
        \caption{10XPBMC}
        \label{fig:10XPBMC_3HiddenLayer}
    \end{subfigure}

    \vspace{0.25 cm}
    \caption{Kruskal Stress and Pearson Correlation as functions of latent dimensionality for USPS, Mouse Bladder, Worm Neuron, and 10xPBMC using the deep ICE architecture (1000, 500, and 125 hidden units). Results are averaged over five independent runs. The red vertical line denotes the number of classes $M$.}
    \vspace{0.25 cm}
    \label{fig:kruskal_correlation_3HiddenLayers}
\end{figure}

\noindent\textbf{Results and interpretation:}
Figure~\ref{fig:kruskal_correlation_3HiddenLayers} shows a consistent trend across all datasets. As the latent dimension increases, Kruskal Stress decreases and Pearson Correlation increases, indicating improved preservation of class-centroid geometry. Both metrics rapidly saturate as the latent dimension approaches the number of classes, after which additional dimensions provide little improvement.

To determine whether this behavior depends on model capacity, we repeat the experiment using linear ICE (no hidden layer) and shallow ICE (one hidden layer with 500 neurons). The corresponding results are shown in Figures~\ref{supp:fig:kruskal_correlation_NoHiddenLayers} and~\ref{supp:fig:kruskal_correlation_OneHiddenLayer}. Both architectures exhibit the same saturation pattern observed for deep ICE.

\begin{figure}[!ht]
    \centering
    \begin{subfigure}[b]{0.45\linewidth}
        \includegraphics[width=\linewidth]{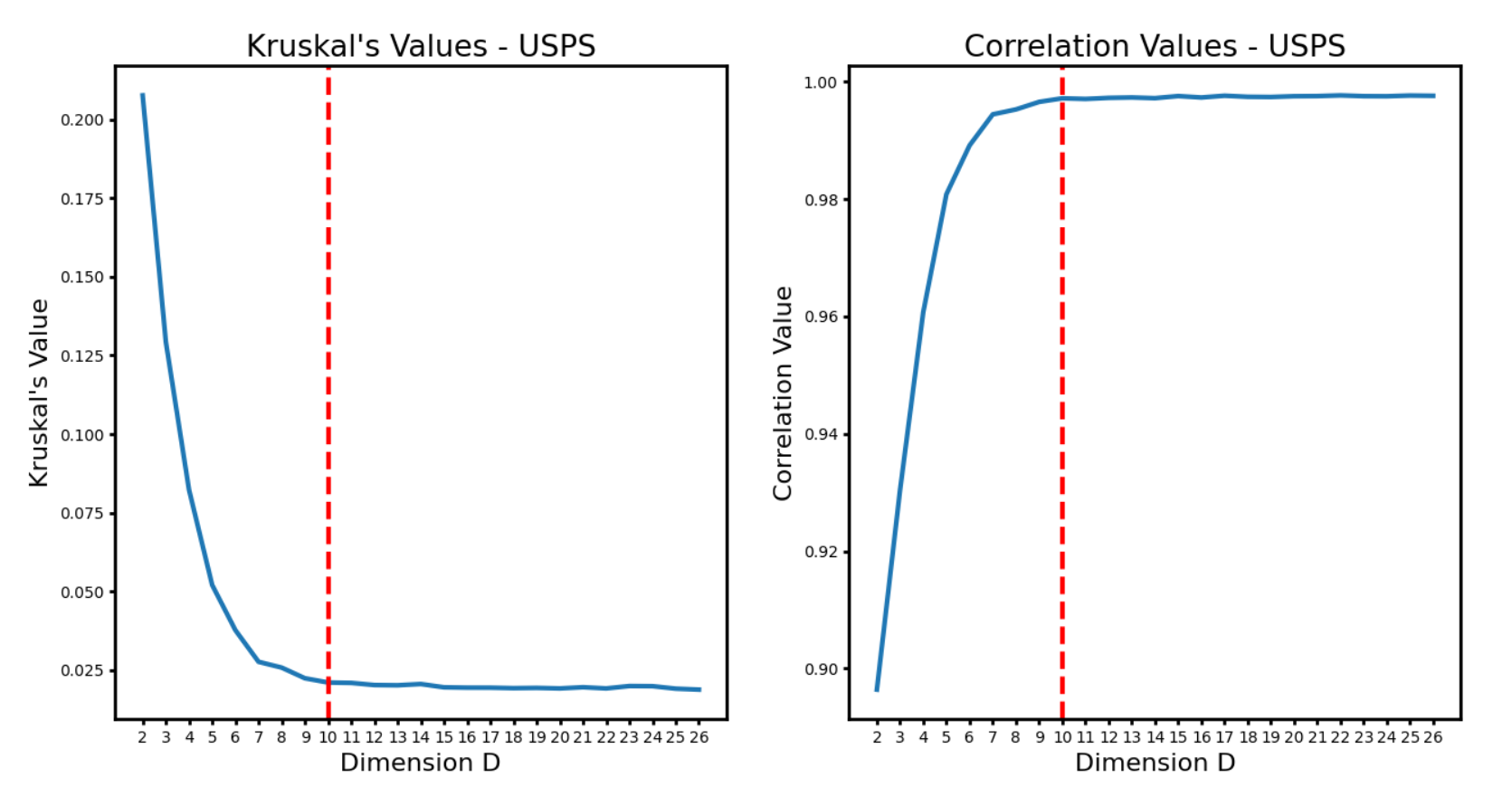}
        \caption{USPS}
        \label{supp:fig:USPS_NoHiddenLayers}
    \end{subfigure}
    \hfill
    \begin{subfigure}[b]{0.45\linewidth}
        \includegraphics[width=\linewidth]{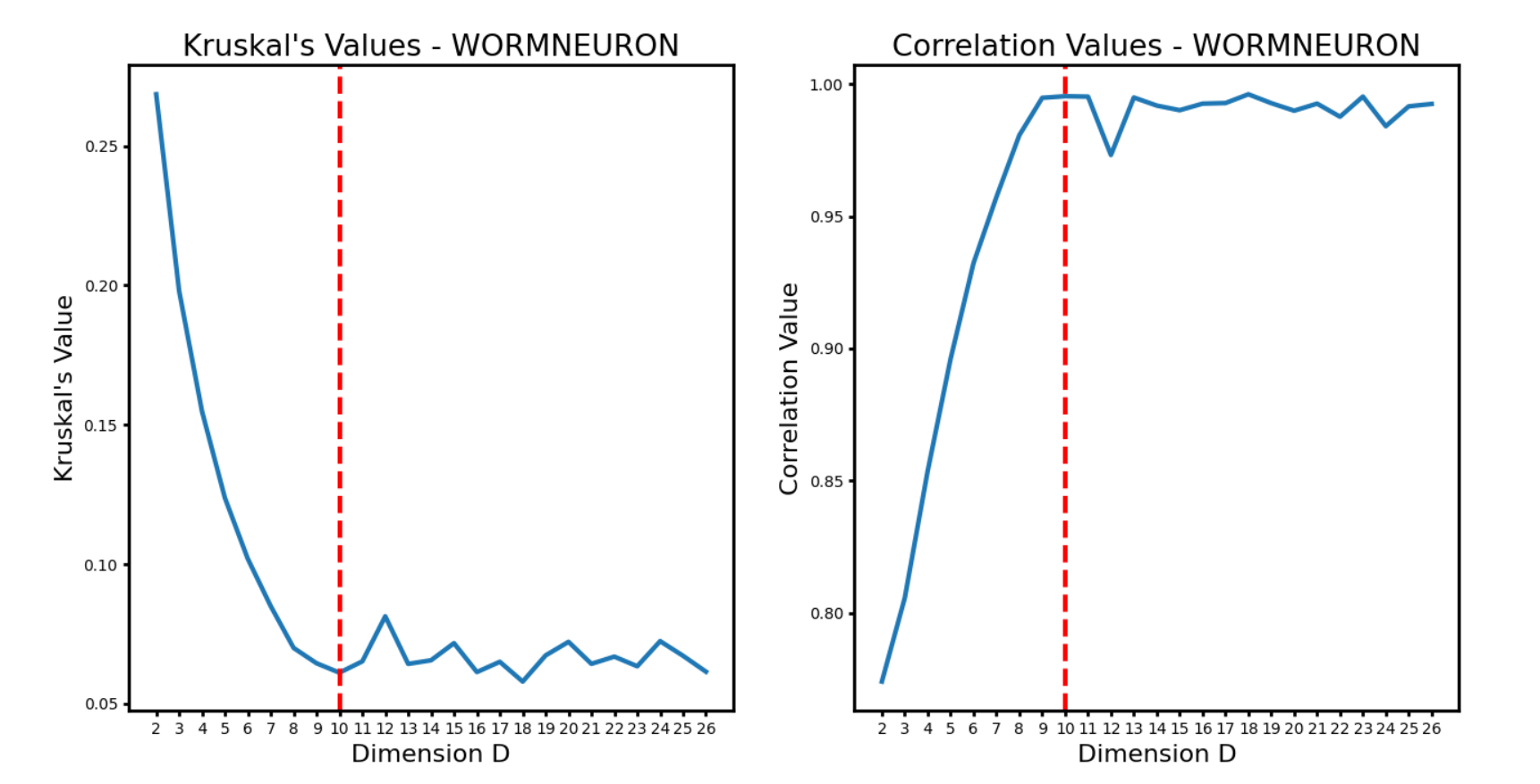}
        \caption{WORMNEURON}
        \label{supp:fig:WORMNEURON_NoHiddenLayers}
    \end{subfigure}

    \begin{subfigure}[b]{0.45\linewidth}
        \includegraphics[width=\linewidth]{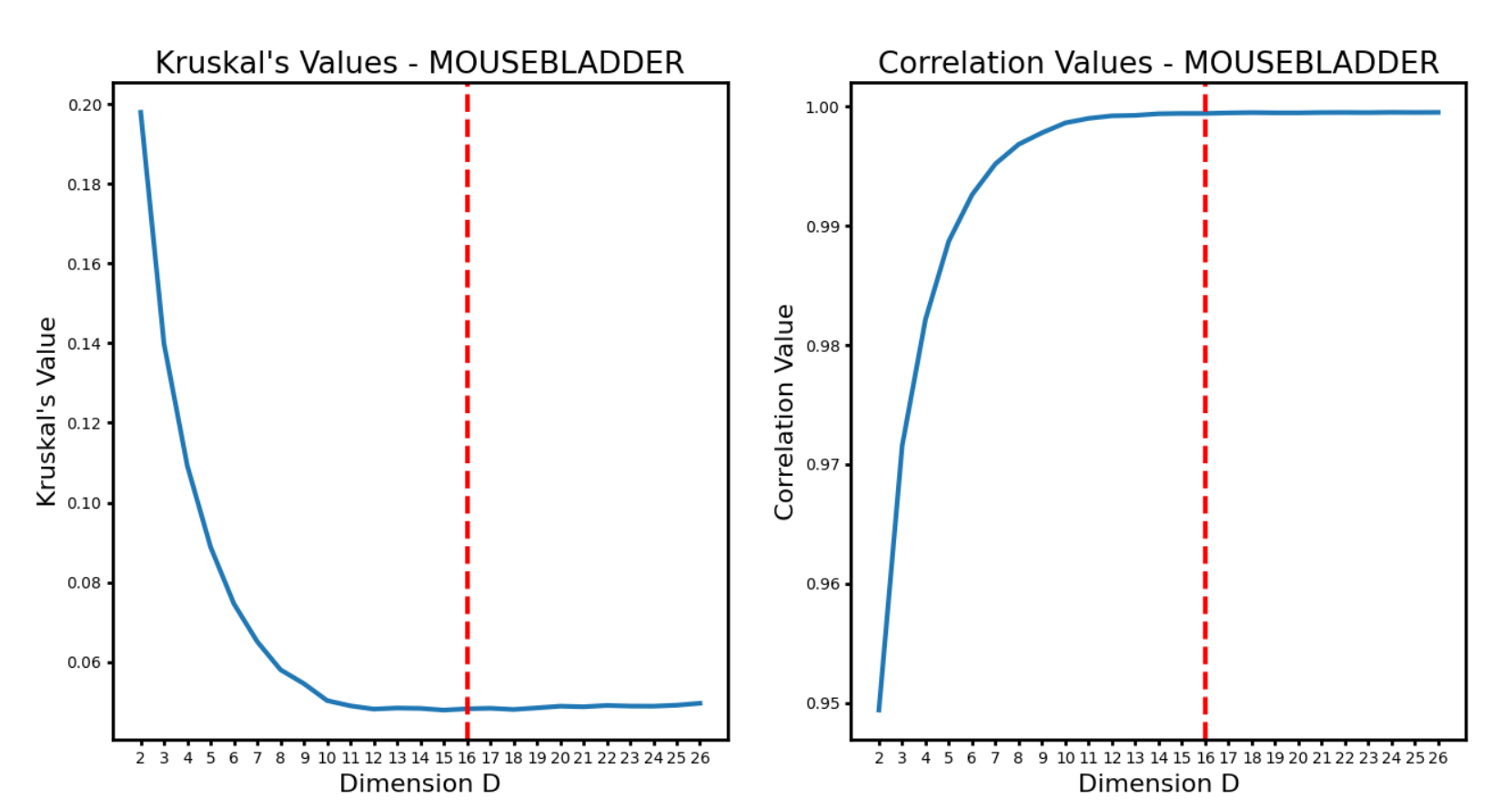}
        \caption{MOUSEBLADDER}
        \label{supp:fig:MOUSEBLADDER_NoHiddenLayers}
    \end{subfigure}
    \hfill
    \begin{subfigure}[b]{0.45\linewidth}
        \includegraphics[width=\linewidth]{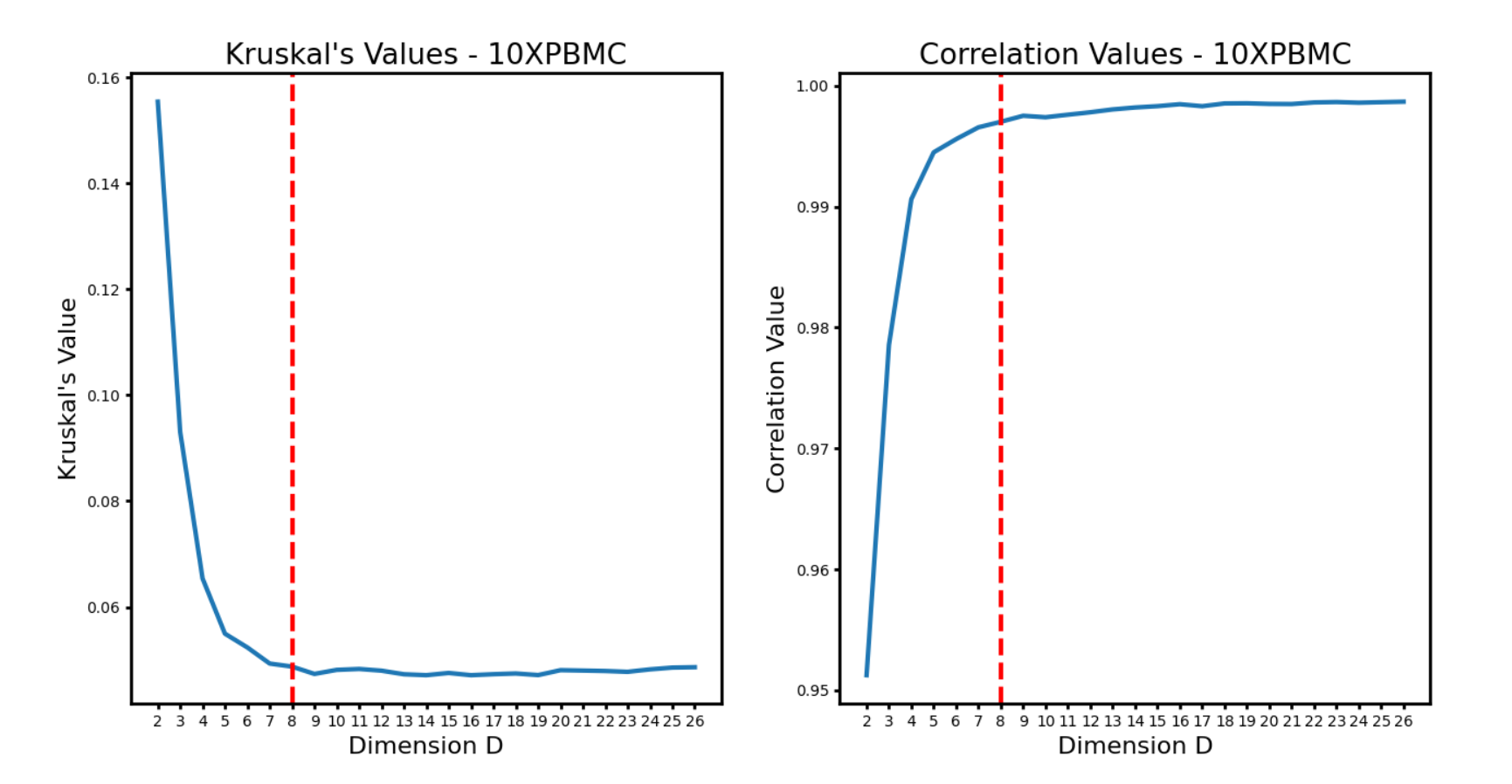}
        \caption{10XPBMC}
        \label{supp:fig:10XPBMC_NoHiddenLayers}
    \end{subfigure}

    \caption{Kruskal Stress and Pearson Correlation of class-centroid distance preservation for USPS, Mouse Bladder, Worm Neuron, and 10xPBMC using linear ICE (no hidden layers). The latent dimension is varied from 2 to 26, and results are averaged over five independent runs. The red vertical line indicates the number of classes $M$.}
    \label{supp:fig:kruskal_correlation_NoHiddenLayers}
\end{figure}

\begin{figure}[!ht]
    \centering
    \begin{subfigure}[b]{0.45\linewidth}
        \includegraphics[width=\linewidth]{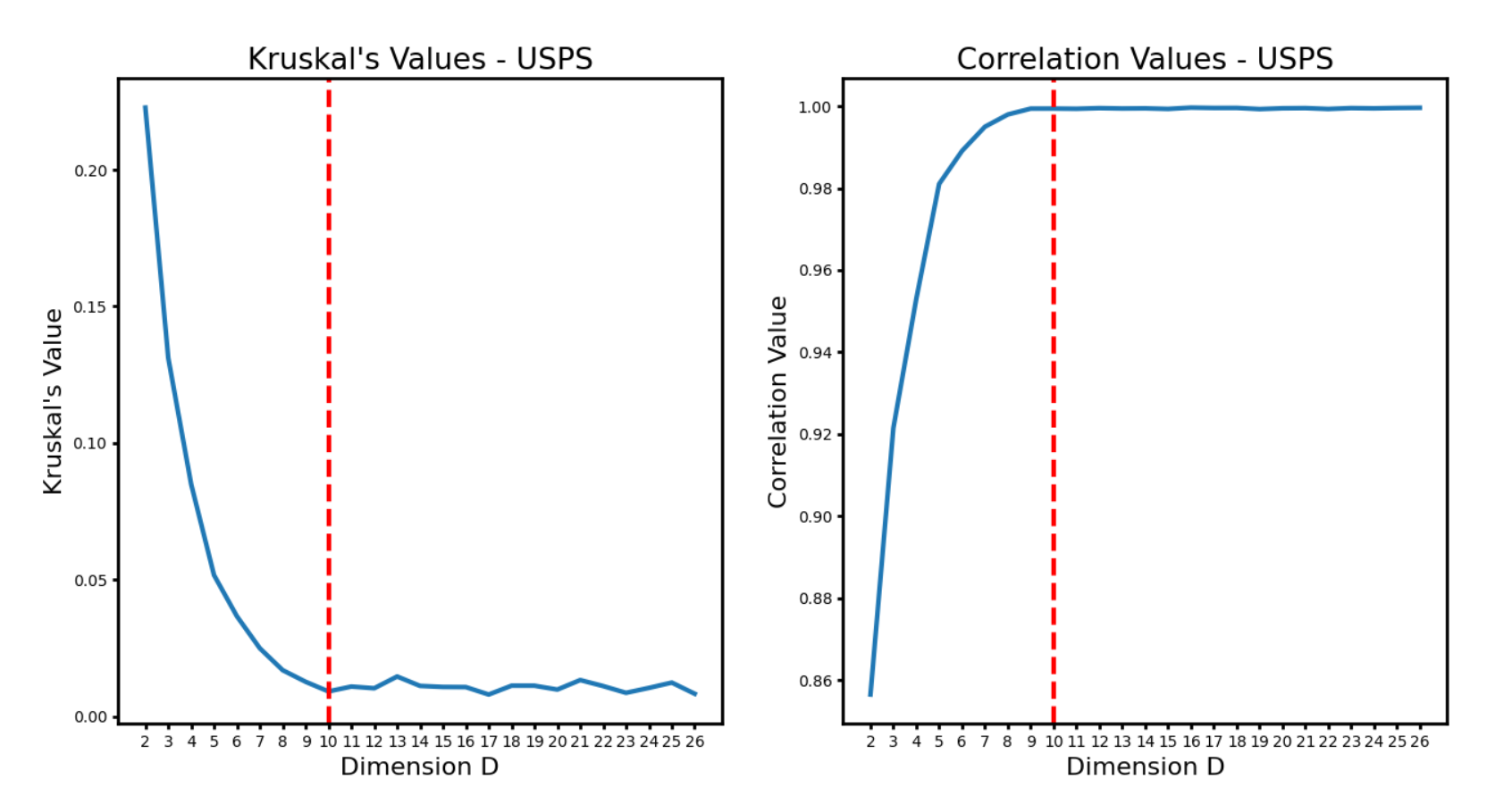}
        \caption{USPS}
        \label{supp:fig:USPS_OneHiddenLayer}
    \end{subfigure}
    \hfill
    \begin{subfigure}[b]{0.45\linewidth}
        \includegraphics[width=\linewidth]{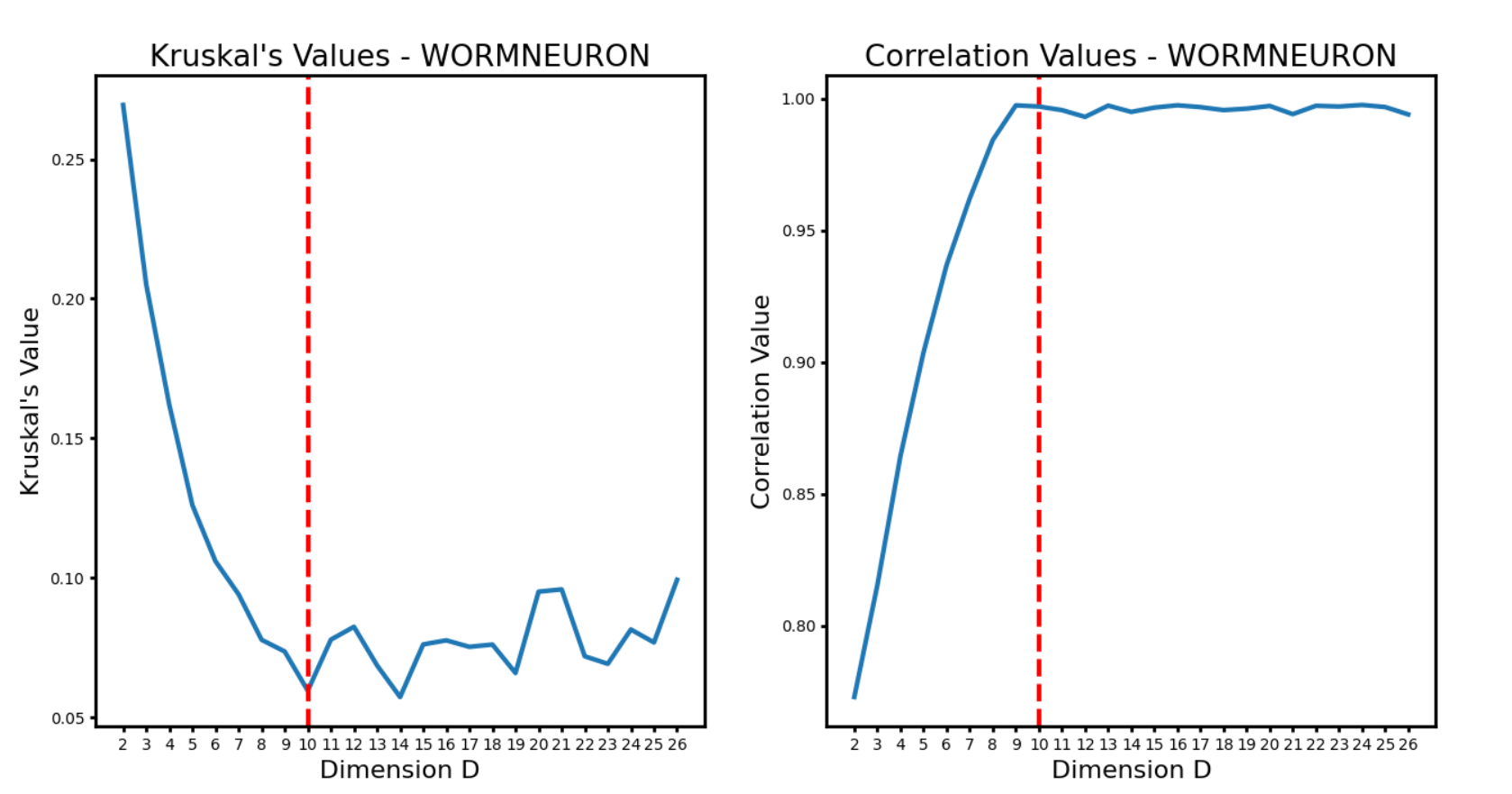}
        \caption{WORMNEURON}
        \label{supp:fig:WORMNEURON_1HiddenLayers}
    \end{subfigure}

    \begin{subfigure}[b]{0.45\linewidth}
        \includegraphics[width=\linewidth]{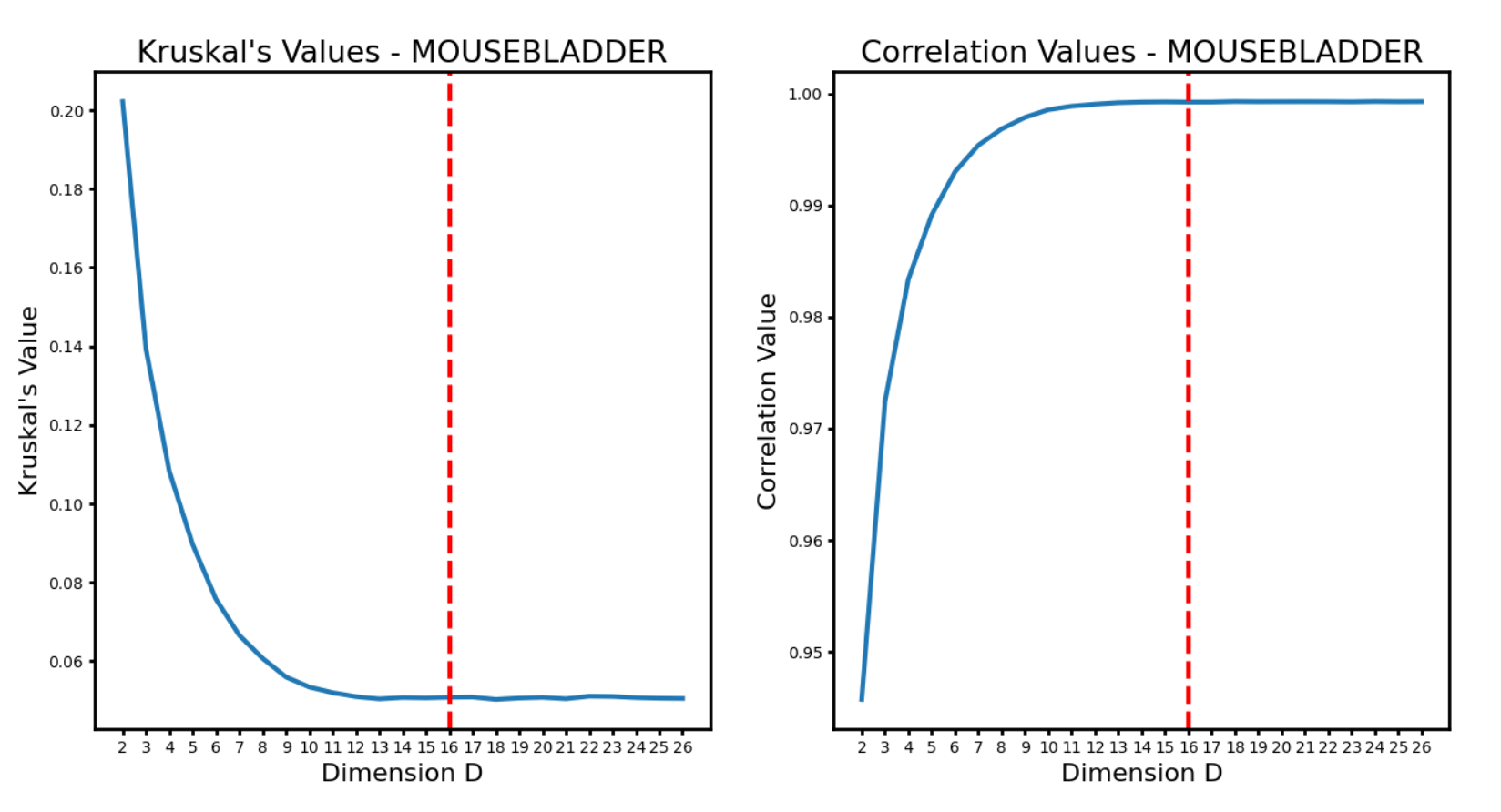}
        \caption{MOUSEBLADDER}
        \label{supp:fig:MOUSEBLADDER_OneHiddenLayer}
    \end{subfigure}
    \hfill
    \begin{subfigure}[b]{0.45\linewidth}
        \includegraphics[width=\linewidth]{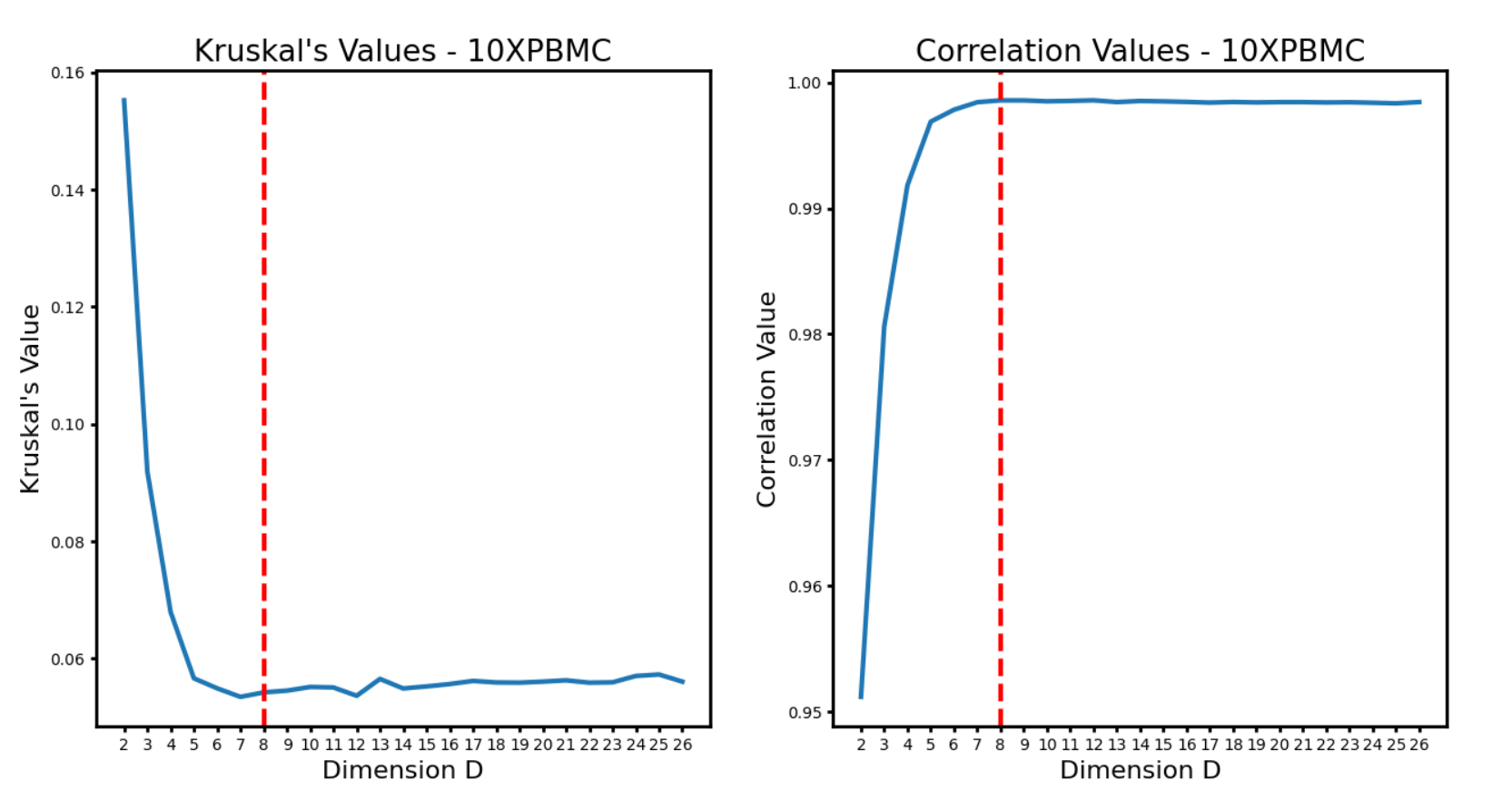}
        \caption{10XPBMC}
        \label{supp:fig:10XPBMC_OneHiddenLayer}
    \end{subfigure}

    \caption{Kruskal Stress and Pearson Correlation of class-centroid distance preservation for USPS, Mouse Bladder, Worm Neuron, and 10xPBMC using shallow ICE (one hidden layer with 500 neurons). The latent dimension ranges from 2 to 26, and results are averaged over five independent runs. The red vertical line denotes the number of classes $M$.}
    \label{supp:fig:kruskal_correlation_OneHiddenLayer}
\end{figure}

Taken together, the three architectures show that the saturation of centroid-distance preservation is largely independent of encoder depth. Increasing nonlinear capacity does not provide systematic improvement once the latent dimension approaches the intrinsic dimensionality of the centroid configuration. This observation is consistent with Lemma~\ref{lemma:centroid_dim}: the geometry of $M$ class centroids can be represented in at most $M-1$ dimensions. Thus, the observed saturation reflects a geometric limitation rather than insufficient encoder capacity.

From a practical standpoint, these results suggest that dimensions beyond $M-1$ provide little additional benefit for preserving class-centroid geometry. However, $M-1$ is an upper bound rather than a requirement; the effective centroid dimensionality may be substantially smaller. We examine this issue next and provide a data-driven procedure for selecting the latent dimension from the centroid eigenspectrum.

\subsection{Practical Selection of the Latent Dimension}
\label{sec:supp_dimension_selection}

Lemma~\ref{lemma:centroid_dim} establishes that the class-centroid configuration of a dataset containing $M$ classes can always be represented in at most $M-1$ dimensions. This result provides a theoretical upper bound on the latent dimensionality required to preserve the centroid geometry. However, it does not imply that all $M-1$ dimensions are necessary in practice. For many real-world datasets, particularly biological datasets, the effective geometric complexity of the class-centroid configuration is often substantially lower.

\begin{figure}[!ht]    
    \includegraphics[width=1.00\linewidth]{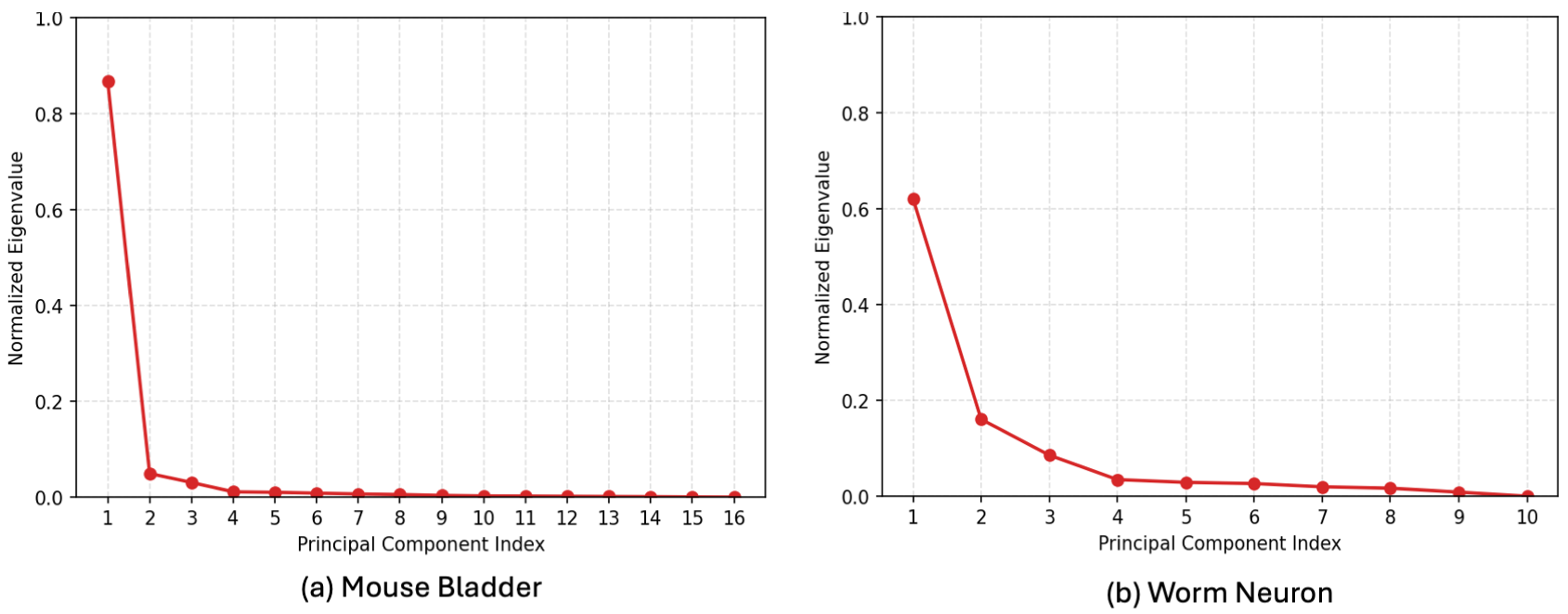}
    \caption{Normalized eigenspectra of the centered class-centroid matrix for (a) Mouse Bladder (16 cell types) and (b) Worm Neuron (10 cell types). The eigenvalues are shown in descending order after normalization by their total sum. Both datasets exhibit a rapidly decaying eigenspectrum, indicating that the dominant centroid geometry is captured by only a small number of principal directions. This suggests that the effective latent dimensionality required to preserve the class-centroid geometry can be substantially smaller than the theoretical upper bound of $M-1$. The eigenspectra were computed using Dual PCA to efficiently perform the eigendecomposition in the sample space, avoiding the construction of the high-dimensional feature covariance matrix.}
    \label{supp:fig:centroid_eigenspectrum}
\end{figure}

A practical strategy for selecting the latent dimensionality is to examine the eigenspectrum of the centered centroid matrix. Let, $\mathbf{C}_c = \mathbf{H}\mathbf{C}$, where $\mathbf{C}$ contains the class centroids in the original feature space and $\mathbf{H}$ is the centering matrix. Let, $\lambda_1 \ge \lambda_2 \ge \cdots \ge \lambda_{M-1} \ge 0$ denote the eigenvalues of $\mathbf{C}_c^\top\mathbf{C}_c$. The latent dimensionality may then be selected as the smallest integer $p^\star$ satisfying

\[
\frac{\sum_{i=1}^{p^\star}\lambda_i}
{\sum_{i=1}^{M-1}\lambda_i}
\ge \tau,
\]
where $\tau$ is a user-defined explained-variance threshold (e.g., $95\%$ or $98\%$). This criterion identifies the minimum number of latent dimensions required to preserve the dominant centroid geometry while avoiding unnecessary dimensions.

To illustrate this observation, we analyzed the eigenspectra of the class-centroid matrices for two biological datasets: Mouse Bladder (16 cell types) and Worm Neuron (10 cell types). Figure~\ref{supp:fig:centroid_eigenspectrum} shows the normalized cumulative explained variance for both datasets. Although the theoretical upper bounds are $15$ and $9$ dimensions, respectively, the centroid eigenspectra decay rapidly, indicating that only a small number of principal centroid directions explain most of the centroid variation. In particular, the Mouse Bladder dataset exhibits a highly dominant first centroid direction, while the Worm Neuron dataset requires only a few principal directions before the cumulative explained variance approaches saturation.

These observations suggest that the effective centroid dimensionality is often considerably smaller than the theoretical upper bound of $M-1$. Consequently, rather than simply setting the latent dimension to the number of classes, practitioners may estimate an appropriate latent dimensionality directly from the centroid eigenspectrum. This provides a simple, deterministic, and data-driven guideline for selecting the latent dimension while remaining consistent with the theoretical result established in Lemma~\ref{lemma:centroid_dim}.

\subsection{Effect of Regularization Parameter \texorpdfstring{$\lambda$}{lambda}}
\label{sec:lambda_main}
\begin{table}[!ht]
\centering
\caption{Effect of the regularization parameter $\lambda$ on centroid neighborhood stability (CNS), Kruskal Stress, and 5NN accuracy for MNIST in a two-dimensional embedding. Results are averaged over five independent runs.}
\vspace{0.1 cm}
\begin{tabular}{c c c c}
\hline
$\lambda$ & CNS & Kruskal Stress & 5NN Accuracy (\%) \\
\hline
0.01   & 0.8080 & $0.2243 \pm 0.0011$ & $97.27 \pm 0.05$ \\
0.1    & \textbf{0.8933} & $0.2258 \pm 0.0005$ & $96.94 \pm 0.09$ \\
1.0    & 0.8400 & $0.2253 \pm 0.0006$ & $96.39 \pm 0.20$ \\
10.0   & 0.7787 & $0.2268 \pm 0.0035$ & $96.48 \pm 0.08$ \\
100.0  & 0.7573 & $0.2378 \pm 0.0100$ & $95.93 \pm 0.12$ \\
\hline
\end{tabular}
\vspace{0.1 cm}
\label{tab:lambda_analysis}
\end{table}

\noindent Table~\ref{tab:lambda_analysis} summarizes the effect of the regularization parameter $\lambda$ on CNS, Kruskal Stress, and 5NN accuracy for MNIST in a two-dimensional embedding. Kruskal Stress remains nearly unchanged across different values of $\lambda$, indicating that the global centroid geometry is largely preserved. In contrast, CNS exhibits a non-monotonic trend, reaching its maximum for intermediate values of $\lambda$. Small values provide weak centroid-isometry regularization, resulting in greater run-to-run variability of the class-centroid configuration, whereas larger values improve geometric regularization but eventually overemphasize the isometry objective.

In a two-dimensional latent space, the centroid-isometry objective cannot generally be satisfied exactly because the centroid geometry of $M$ classes requires up to $M-1$ dimensions (Lemma~\ref{lemma:centroid_dim}). Consequently, excessively large values of $\lambda$ force the optimization toward a compromise solution, increasing the sensitivity of class neighborhoods to small perturbations and reducing CNS despite little change in global Stress.

Smaller values of $\lambda$ produce slightly higher 5NN accuracy by emphasizing the centroid-consistency objective and improving local sample neighborhoods. This illustrates that classification accuracy and CNS quantify complementary properties of the learned embedding.

\label{supp:sec:supp_lambda}
To evaluate whether this behavior generalizes across application domains, we repeat the analysis on the USPS and 10X PBMC datasets (Tables~\ref{supp:tab:lambda_usps} and~\ref{supp:tab:lambda_pbmc}). Following the same protocol, the latent dimension is fixed at two, while $\lambda$ is varied over $\{0.01, 0.1, 1, 10, 100\}$. For each setting, we report the proposed Centroid Neighborhood Stability (CNS), Kruskal Stress, and 5NN classification accuracy averaged over five independent runs.

\begin{table}[t]
\centering
\caption{Effect of the regularization parameter $\lambda$ on USPS. Results are averaged over five independent runs using an $80{:}20$ train-test split.}
\label{supp:tab:lambda_usps}
\begin{tabular}{c c c c}
\hline
$\lambda$ & CNS & Kruskal Stress & 5NN Accuracy (\%)\\
\hline
0.01  & 0.8747 & $0.2086 \pm 0.0035$ & $96.24 \pm 0.22$\\
0.1   & 0.8160 & $0.2164 \pm 0.0118$ & $96.15 \pm 0.31$\\
1.0   & \textbf{0.8933} & $\mathbf{0.2085 \pm 0.0036}$ & $96.09 \pm 0.25$\\
10.0  & 0.7893 & $0.2147 \pm 0.0075$ & \textbf{$96.80 \pm 0.14$}\\
100.0 & 0.8027 & $0.2137 \pm 0.0088$ & $95.28 \pm 0.40$\\
\hline
\end{tabular}
\end{table}

\begin{table}[t]
\centering
\caption{Effect of the regularization parameter $\lambda$ on the 10X PBMC dataset. Results are averaged over five independent runs using an $80{:}20$ train-test split.}
\label{supp:tab:lambda_pbmc}
\begin{tabular}{c c c c}
\hline
$\lambda$ & CNS & Kruskal Stress & 5NN Accuracy (\%)\\
\hline
0.01  & 0.8700 & $0.1693 \pm 0.0128$ & $82.61 \pm 1.42$\\
0.1   & 0.8967 & $0.1621 \pm 0.0065$ & $81.87 \pm 2.71$\\
1.0   & 0.9100 & $0.1655 \pm 0.0033$ & $80.98 \pm 1.27$\\
10.0  & \textbf{0.9133} & $0.1580 \pm 0.0065$ & \textbf{$83.24 \pm 1.42$}\\
100.0 & 0.9000 & $\mathbf{0.1515 \pm 0.0044}$ & $78.79 \pm 1.87$\\
\hline
\end{tabular}
\end{table}

\noindent\textbf{Results and interpretation:}

The sensitivity analysis on USPS and 10X PBMC exhibits trends similar to those observed for MNIST in Table~\ref{tab:lambda_analysis}. Across both datasets, ICE remains relatively stable over a broad range of $\lambda$ values, indicating that the method is not highly sensitive to the precise choice of the regularization parameter.

For USPS, the classification accuracy remains above $95\%$ for all values of $\lambda$, while both CNS and Kruskal Stress vary only modestly. The highest CNS is obtained at $\lambda=1$, whereas the best classification accuracy occurs at $\lambda=10$, illustrating that embedding stability and classification performance capture complementary properties of the learned representation.

A similar trend is observed on the 10X PBMC dataset. Moderate values of $\lambda$ ($1$--$10$) provide consistently high CNS together with low Kruskal Stress, while the highest classification accuracy is obtained at $\lambda=10$. Increasing $\lambda$ further to $100$ continues to improve centroid-distance preservation but leads to a reduction in classification accuracy, suggesting that excessive emphasis on the centroid-isometry objective may slightly compromise local sample organization.

Overall, these experiments demonstrate that the qualitative behavior of ICE with respect to $\lambda$ is consistent across image and biological datasets. Moderate values of $\lambda$ provide a favorable balance between the centroid-consistency and centroid-isometry objectives, yielding stable class-level geometry together with strong classification performance.

\noindent\textbf{Design implications:}
These results highlight that low centroid-distance Stress alone does not guarantee stable class neighborhoods. While Stress varies smoothly with $\lambda$, CNS is a discrete set-based measure that can change abruptly when small variations in centroid distances alter nearest-neighbor membership. Thus, $\lambda$ provides an explicit mechanism for balancing class-centroid geometry, classification accuracy, and neighborhood stability. When the latent dimensionality is sufficient ($p \gtrsim M-1$), ICE consistently achieves both low Stress and high CNS. In low-dimensional embeddings, however, $\lambda$ must be selected to balance geometric preservation and stability. These findings further demonstrate that Stress, classification accuracy, and CNS quantify complementary aspects of embedding quality.

\subsection{Ablation Study}
\label{subsec:ablation}
\label{supp:sec:supp_ablation}

The ICE objective consists of two components: the centroid-consistency loss $\mathcal{L}_{\mathrm{cons}}$, which promotes within-class compactness, and the centroid-isometry loss $\mathcal{L}_{\mathrm{iso}}$, which preserves the prescribed pairwise geometry among latent class centroids. To isolate the role of these components, we consider two complementary ablations:
(i) consistency-only ICE, obtained by setting $\lambda=0$ and removing the centroid-isometry term, and (ii) isometry-only ICE, obtained by removing the centroid-consistency term. 

\subsubsection{Consistency-Only ICE: Removing the Centroid-Isometry Term}

We first remove the centroid-isometry term by setting $\lambda=0$. Experiments
were performed on MNIST and FMNIST in both the two-dimensional ($p=2$) and
class-dimensional ($p=M=10$) settings.

\begin{table}[!ht]
\begin{tabular}{llcccc}
\hline
Dataset & Latent Dim. & Stress & Correlation & CNS & 5NN Accuracy (\%)\\
\hline
MNIST  & $2$  & $1.00 \pm 0.00$ & $-0.00 \pm 0.19$ & $0.46 \pm 0.01$ & $14.12 \pm 3.68$\\
MNIST  & $10$ & $1.00 \pm 0.00$ & $0.00 \pm 0.12$  & $0.48 \pm 0.01$ & $9.91 \pm 0.39$\\
FMNIST & $2$  & $1.00 \pm 0.00$ & $0.13 \pm 0.32$  & $0.50 \pm 0.02$ & $20.65 \pm 6.13$\\
FMNIST & $10$ & $1.00 \pm 0.00$ & NA                  & $0.77 \pm 0.04$ & $10.00 \pm 0.01$\\
\hline
\end{tabular}
\caption{Ablation results obtained by removing the centroid-isometry term
($\lambda=0$). Lower Kruskal Stress is better, whereas higher Correlation,
CNS, and 5NN accuracy are better.}
\label{supp:tab:ablation_cons_only}
\end{table}

Removing the isometry term produces a degenerate solution in all four configurations. Kruskal Stress reaches $1.00$, correlation becomes negligible or undefined, and classification accuracy approaches the chance level of $10\%$ for the two ten-class datasets. This behavior follows directly from the
structure of $\mathcal{L}_{\mathrm{cons}}$. Since the latent centroids are dynamically computed from the encoder outputs, a configuration in which all samples and all class centroids collapse to a common location can minimize the centroid-consistency objective while containing no meaningful class-level geometry.

For FMNIST with $p=10$, Pearson correlation is reported as NA because one run produced numerically identical latent class centroids. Consequently, the encoded pairwise centroid-distance vector had zero variance, for which Pearson correlation is undefined. The remaining runs also exhibited near-zero inter-centroid distances, confirming the same degenerative behavior.

These results indicate that the centroid-consistency objective alone is insufficient to determine the relative arrangement of the latent classes. The centroid-isometry term is therefore a key component of ICE for preventing trivial collapse and preserving the prescribed class-centroid geometry.

\subsubsection{Isometry-Only ICE: Removing the Centroid-Consistency Term}

We next perform the complementary ablation by removing
$\mathcal{L}_{\mathrm{cons}}$ and optimizing only the centroid-isometry objective. This experiment was conducted on MNIST, FMNIST, 10X PBMC, and Mouse Bladder using both $p=2$ and $p=M$. The coefficient of the isometry objective was fixed at $\lambda=1$.

\begin{table}[!ht]
\centering
\caption{Performance of isometry-only ICE after removing the
centroid-consistency term. Results are averaged over 15 independent runs.}
\label{supp:tab:ablation_iso_only}
\resizebox{\textwidth}{!}{
\begin{tabular}{llcccc}
\hline
Dataset & Latent Dim. & Kruskal Stress & Correlation & CNS & 5NN Accuracy (\%)\\
\hline
MNIST         & $2$  & $0.23 \pm 0.01$ & $0.83 \pm 0.03$ & $0.72 \pm 0.02$ & $94.26 \pm 0.32$\\
FMNIST        & $2$  & $0.15 \pm 0.02$ & $0.94 \pm 0.03$ & $0.95 \pm 0.02$ & $84.69 \pm 0.41$\\
10X PBMC      & $2$  & $0.16 \pm 0.01$ & $0.94 \pm 0.01$ & $0.91 \pm 0.03$ & $85.55 \pm 2.10$\\
Mouse Bladder & $2$  & $0.20 \pm 0.01$ & $0.94 \pm 0.00$ & $0.77 \pm 0.02$ & $65.71 \pm 1.69$\\
\hline
MNIST         & $M$  & $0.01 \pm 0.00$ & $1.00 \pm 0.00$ & $0.95 \pm 0.02$ & $97.39 \pm 0.10$\\
FMNIST        & $M$  & $0.02 \pm 0.01$ & $1.00 \pm 0.00$ & $0.99 \pm 0.01$ & $87.83 \pm 0.17$\\
10X PBMC      & $M$  & $0.06 \pm 0.01$ & $1.00 \pm 0.00$ & $0.97 \pm 0.01$ & $88.49 \pm 0.54$\\
Mouse Bladder & $M$  & $0.05 \pm 0.02$ & $1.00 \pm 0.00$ & $0.94 \pm 0.01$ & $80.28 \pm 1.15$\\
\hline
\end{tabular}
}
\end{table}

In contrast to the consistency-only ablation, the isometry-only model retains strong class-level geometric structure. In the class-dimensional setting, Kruskal Stress remains between $0.01$ and $0.06$, Pearson correlation is approximately $1.00$, and CNS ranges from $0.94$ to $0.99$ across all four datasets. These values are close to those obtained by the complete ICE model in Table~\ref{tab:M-D_Results}. The result indicates that the
centroid-isometry objective appears to be the primary mechanism responsible for preserving and stabilizing the class-centroid configuration.

Removing the centroid-consistency term has a more noticeable effect on local sample organization. For example, in the two-dimensional setting, MNIST accuracy decreases from $97.26\%$ for full ICE to $94.26\%$ for the isometry-only model, while FMNIST decreases from $88.95\%$ to $84.69\%$. The corresponding reductions are smaller in the class-dimensional setting,
where sufficient latent capacity allows the prescribed centroid geometry to be represented more faithfully. Similar behavior is observed on the biological datasets: geometry-preservation and CNS remain strong, while classification performance benefits from including the consistency term.

Importantly, $\mathcal{L}_{\mathrm{iso}}$ does not explicitly minimize the distance between individual samples and their corresponding latent centroid. Rather, each latent sample contributes to the dynamically computed class centroid, and the resulting gradient adjusts the encoder so that the pairwise centroid configuration approaches the prescribed reference geometry. The centroid-consistency term provides the complementary sample-level constraint by explicitly encouraging observations from the same class to remain close to their latent centroid.

\noindent\textbf{Overall interpretation:}
The two ablations reveal an asymmetric but complementary role of the two ICE objectives. The centroid-consistency term alone admits trivial collapse because it does not constrain the relative locations of different classes. In contrast, the centroid-isometry term establishes the global class-centroid geometry and produces highly stable embeddings even when used independently. Adding the
centroid-consistency term further improves within-class localization and local sample organization while retaining the geometry induced by the isometry objective. Thus, the complete ICE objective combines class-level geometric preservation with explicit within-class compactness rather than relying on either
property alone.

\subsection{Sensitivity of CNS to Its Parameters}
\label{supp:sec:supp_cns_sensitivity}

The proposed Centroid Neighborhood Stability (CNS) metric depends on two user-defined parameters: the neighborhood size $k$, which determines the number of nearest neighboring classes considered, and the number of independent runs $R$, over which neighborhood consistency is evaluated. Since these parameters influence the notion of neighborhood stability being measured, it is important to examine whether the comparative conclusions drawn using CNS remain consistent across reasonable parameter choices. We use $k=3$ as an intermediate neighborhood size and $R=5$ independent runs for all experiments. Here we investigate the sensitivity of CNS to varying $k$ and $R$ on the MNIST dataset using both two-dimensional and class-dimensional embeddings. For each $(k,R)$ configuration, Table~\ref{supp:tab:cns_sensitivity_2d} and Table~\ref{supp:tab:cns_sensitivity_cd} report the average CNS values with the relative ranking of the methods in $p=2$ and $p=M$ dimensional space.

\begin{table}[!ht]
\centering
\small
\begin{tabular}{c c l}
\hline
$k$ & $R$ & CNS ranking (higher is better) \\
\hline
1 & 3  & IVIS (0.8667) $>$ ICE (0.7867) $>$ SE (0.7778) $>$ CE (0.6444) $>$ SUMAP (0.6044) \\
1 & 5  & IVIS (0.8507) $>$ ICE (0.7600) $>$ SE (0.7387) $>$ CE (0.5680) $>$ SUMAP (0.5627) \\
1 & 15 & IVIS (0.8267) $>$ ICE (0.6853) $>$ SE (0.6836) $>$ SUMAP (0.4996) $>$ CE (0.4844) \\
3 & 3  & IVIS (0.9111) $>$ SE (0.8281) $>$ ICE (0.7793) $>$ CE (0.7689) $>$ SUMAP (0.7437) \\
3 & 5  & IVIS (0.8969) $>$ SE (0.7796) $>$ ICE (0.7422) $>$ SUMAP (0.7271) $>$ CE (0.7067) \\
3 & 15 & IVIS (0.8498) $>$ SE (0.7390) $>$ SUMAP (0.6942) $>$ ICE (0.6833) $>$ CE (0.6361) \\
5 & 3  & IVIS (0.9076) $>$ SE (0.8827) $>$ ICE (0.8791) $>$ CE (0.8427) $>$ SUMAP (0.8391) \\
5 & 5  & IVIS (0.8821) $>$ ICE (0.8619) $>$ SE (0.8549) $>$ SUMAP (0.8277) $>$ CE (0.8016) \\
5 & 15 & IVIS (0.8491) $>$ ICE (0.8231) $>$ SE (0.8229) $>$ SUMAP (0.8027) $>$ CE (0.7639) \\
\hline
\end{tabular}
\caption{Sensitivity of CNS to neighborhood size $k$ and number of runs $R$ for the MNIST two-dimensional embedding. Methods are ranked according to their average CNS for each $(k,R)$ configuration.}
\label{supp:tab:cns_sensitivity_2d}
\end{table}

\textbf{Results and interpretation:} Tables~\ref{supp:tab:cns_sensitivity_2d} and~\ref{supp:tab:cns_sensitivity_cd} examine the sensitivity of CNS to the neighborhood size $k$ and the number of independent runs $R$. Two consistent trends are observed across almost all methods. First, for a fixed neighborhood size $k$, the absolute CNS values gradually decrease as the number of runs $R$ increases. This behavior is expected because additional independent runs provide more opportunities for variation in the recovered neighborhood sets. Second, for a fixed 
value of $R$, CNS generally increases with increasing neighborhood size $k$, since larger neighborhoods are less sensitive to changes in individual nearest neighbors.

The parameter $k$ determines the scale at which neighborhood stability is evaluated: $k=1$ measures consistency of the nearest class relationship, whereas larger values assess increasingly broader class neighborhoods. Because CNS operates on unordered neighbor sets, it measures whether the same neighboring classes are repeatedly recovered, independent of their internal ranking. Consequently, changing $k$ changes the scale at which neighborhood stability is evaluated and may therefore alter both the absolute CNS values and the relative ordering of closely performing methods.

\begin{table}[t]
\centering
\small
\begin{tabular}{c c l}
\hline
$k$ & $R$ & CNS ranking (higher is better) \\
\hline
1 & 3  & ICE (0.9911) $>$ SE (0.9422) $>$ IVIS (0.9289) $>$ CE (0.7733) $>$ SUMAP (0.6222) \\
1 & 5  & ICE (0.9893) $>$ SE (0.9360) $>$ IVIS (0.9067) $>$ CE (0.7173) $>$ SUMAP (0.5547) \\
1 & 15 & ICE (0.9852) $>$ SE (0.9262) $>$ IVIS (0.8927) $>$ CE (0.6539) $>$ SUMAP (0.4793) \\
3 & 3  & ICE (0.9674) $>$ IVIS (0.9111) $>$ SE (0.8993) $>$ CE (0.8267) $>$ SUMAP (0.7748) \\
3 & 5  & ICE (0.9609) $>$ IVIS (0.8969) $>$ SE (0.8791) $>$ CE (0.7813) $>$ SUMAP (0.7271) \\
3 & 15 & ICE (0.9495) $>$ IVIS (0.8498) $>$ SE (0.8407) $>$ CE (0.7146) $>$ SUMAP (0.6942) \\
5 & 3  & ICE (0.9742) $>$ SE (0.9582) $>$ IVIS (0.9324) $>$ CE (0.8658) $>$ SUMAP (0.8604) \\
5 & 5  & ICE (0.9675) $>$ SE (0.9456) $>$ IVIS (0.9248) $>$ CE (0.8416) $>$ SUMAP (0.8363) \\
5 & 15 & ICE (0.9543) $>$ SE (0.9246) $>$ IVIS (0.8998) $>$ CE (0.7893) $>$ SUMAP (0.7883) \\
\hline
\end{tabular}
\caption{Sensitivity of CNS to neighborhood size $k$ and number of runs $R$ for the MNIST class-dimensional embedding ($p=M$). Methods are ranked according to their average CNS for each $(k,R)$ configuration.}
\label{supp:tab:cns_sensitivity_cd}
\end{table}

Increasing $R$ generally decreases the absolute CNS values across all methods because additional independent runs provide more opportunities for variation in neighborhood membership. Importantly, this effect is systematic rather than method-specific. In the two-dimensional setting, where the latent space is severely constrained, some changes in method ordering occur as $k$ and $R$ vary. This is expected because small changes in centroid positions can alter neighborhood membership when the complete class geometry cannot be faithfully represented.

In contrast, the class-dimensional setting exhibits substantially more consistent behavior. ICE achieves the highest CNS for every tested combination of $k \in \{1,3,5\}$ and 
$R \in \{3,5,15\}$, while the relative ordering of the remaining methods changes only modestly. These results indicate that when sufficient latent dimensionality is available, 
the class-neighborhood structure learned by ICE remains highly reproducible across both local and broader neighborhood definitions and across different numbers of independent 
runs.
We used $k=3$ as an intermediate neighborhood size that captures local class relationships without restricting CNS to only the single nearest neighbor. 
We use $R=5$ so that the 15 independent experimental runs can be partitioned into three non-overlapping groups, allowing both the mean and variability of CNS to be reported. 
These settings were fixed across all methods and datasets.

\subsection{Effect of PCA Preprocessing on High-Dimensional Biological Data}
\label{supp:sec:supp_pca_ice}

Because the reference geometry of ICE is computed from Euclidean centroid distances in the ambient space, a natural question is whether these distances remain reliable in the presence of noisy, redundant, or weakly informative features. To investigate this issue, we evaluated ICE after applying PCA to the three biological datasets. PCA was fitted using only the training samples, and the smallest number of principal components preserving at least $98\%$ of the training-set variance was retained. The test samples were then centered and projected using the training-derived mean and principal directions. ICE was subsequently trained using the same experimental protocol employed for the original-feature experiments. Table~\ref{supp:tab:pca_ice_biological} compares ICE with and without PCA preprocessing in both the direct two-dimensional setting ($p=2$) and the class-dimensional setting ($p=M$). All results are reported as mean $\pm$ standard deviation over 15 independent runs. CNS was computed from three non-overlapping groups of five runs, as described in the experimental protocol.

\begin{table}[!ht]

\adjustbox{max width=\linewidth}{\begin{tabular}{llcccc}
\hline
\textbf{Latent Dim.} &
\textbf{Dataset} &
\textbf{Input Representation} &
\textbf{Kruskal Stress} &
\textbf{CNS} &
\textbf{5NN Accuracy (\%)} \\
\hline

\multirow{6}{*}{$p=2$}
& \multirow{2}{*}{10X PBMC}
& Original features & $0.15 \pm 0.01$ & $\mathbf{0.94 \pm 0.04}$ & $85.70 \pm 1.27$ \\
& & PCA (98\%) & $\mathbf{0.14 \pm 0.00}$ & $0.92 \pm 0.06$ & $\mathbf{87.54 \pm 1.35}$ \\

& \multirow{2}{*}{Mouse Bladder}
& Original features & $0.20 \pm 0.01$ & $0.79 \pm 0.01$ & $67.03 \pm 1.69$ \\
& & PCA (98\%) & $\mathbf{0.09 \pm 0.00}$ & $\mathbf{0.91 \pm 0.01}$ & $\mathbf{79.81 \pm 1.10}$ \\

& \multirow{2}{*}{Worm Neuron}
& Original features & $0.27 \pm 0.00$ & $0.65 \pm 0.01$ & $69.86 \pm 2.73$ \\
& & PCA (98\%) & $\mathbf{0.16 \pm 0.00}$ & $\mathbf{0.73 \pm 0.04}$ & $\mathbf{75.53 \pm 2.90}$ \\
\hline

\multirow{6}{*}{$p=M$}
& \multirow{2}{*}{10X PBMC}
& Original features & $0.06 \pm 0.01$ & $\mathbf{0.99 \pm 0.01}$ & $89.27 \pm 0.45$ \\
& & PCA (98\%) & $\mathbf{0.03 \pm 0.00}$ & $0.98 \pm 0.00$ & $\mathbf{91.04 \pm 0.42}$ \\

& \multirow{2}{*}{Mouse Bladder}
& Original features & $0.07 \pm 0.01$ & $0.91 \pm 0.01$ & $78.12 \pm 0.95$ \\
& & PCA (98\%) & $\mathbf{0.03 \pm 0.00}$ & $\mathbf{0.98 \pm 0.00}$ & $\mathbf{83.86 \pm 0.67}$ \\

& \multirow{2}{*}{Worm Neuron}
& Original features & $0.07 \pm 0.01$ & $\mathbf{0.99 \pm 0.02}$ & $87.14 \pm 1.00$ \\
& & PCA (98\%) & $\mathbf{0.04 \pm 0.01}$ & $0.95 \pm 0.02$ & $\mathbf{93.79 \pm 1.02}$ \\
\hline
\end{tabular}}
\caption{Performance of ICE on high-dimensional biological datasets with and without PCA preprocessing. PCA retains at least $98\%$ of the variance estimated from the training data. Lower Kruskal Stress is better, whereas higher CNS and 5NN accuracy are better.}
\label{supp:tab:pca_ice_biological}
\end{table}

\noindent\textbf{Results and interpretation:}
PCA preprocessing improves classification performance across all three biological datasets and generally improves centroid-distance preservation. In the direct two-dimensional setting, the effect is particularly pronounced for Mouse Bladder and Worm Neuron. For Mouse Bladder, Kruskal Stress decreases from $0.20$ to $0.09$ and 5NN accuracy increases from $67.03\%$ to $79.81\%$. For Worm Neuron, Stress decreases from $0.27$ to $0.16$ and accuracy increases from $69.86\%$ to $75.53\%$. The improvement for 10X PBMC is more modest. These results suggest that removing redundant or weakly informative feature directions can provide a cleaner representation from which the reference centroid geometry is estimated, which is particularly beneficial when that geometry must be represented in a severely constrained two-dimensional latent space.

A different pattern emerges in the class-dimensional setting. PCA continues to improve 5NN accuracy across all three datasets, whereas the changes in Kruskal Stress and CNS are generally smaller because ICE already achieves low Stress and high CNS from the original features when sufficient latent dimensionality is available. For example, on Worm Neuron, accuracy increases from $87.14\%$ to $93.79\%$, while Stress changes from $0.07$ to $0.04$ and CNS remains high ($0.99$ versus $0.95$). Similarly, on 10X PBMC, accuracy increases from $89.27\%$ to $91.04\%$, while CNS remains nearly unchanged ($0.99$ versus $0.98$). Mouse Bladder shows a larger improvement in both geometry and classification, with Stress decreasing from $0.07$ to $0.03$, CNS increasing from $0.91$ to $0.98$, and accuracy increasing from $78.12\%$ to $83.86\%$.

Overall, these results show that PCA preprocessing can improve both the reference centroid geometry and sample-level discrimination, particularly for high-dimensional biological data. However, its effect depends on the available latent dimensionality. Under the severe constraint of $p=2$, improving the input representation can substantially benefit both geometric preservation and classification. At $p=M$, ICE already preserves class-centroid geometry with low Stress and high CNS on the original features, while PCA can still yield substantial gains in 5NN accuracy. This distinction further illustrates that classification performance and class-level stability capture complementary properties: 5NN accuracy evaluates local sample discrimination within an embedding, whereas CNS measures the reproducibility of class-neighborhood relationships across independent runs.

\section{Comparative Benchmark}
\label{sec:benchmark}
This section compares ICE with representative supervised dimensionality reduction methods in two-dimensional and $M$-dimensional spaces. Performance is evaluated using complementary measures of local structure and global geometry. Since the results for $M-1$ and $M$ dimensions are nearly identical, we report only the latter. We also present additional visualizations on FMNIST and 10XPBMC and compare the computational cost of ICE and CE.

\subsection{Results in Low Dimension (p = 2)}
Table~\ref{tab:2D_Results} summarizes performance in the two-dimensional embedding space, a highly constrained regime where class structure and geometry must be preserved under severe dimensional limitations.
\begin{table}[!ht]
\small
\caption{Results on 2D embedding across datasets. Numbers for csTSNE on MNIST and FMNIST are not reported due to computational limitations on full datasets.}
\vspace{0.1 cm}
\adjustbox{max width=\linewidth}{\begin{tabular}{llcccccc}
\toprule
\textbf{Metric} & \textbf{Model} & \multicolumn{6}{c}{\textbf{Datasets}} \\
\cmidrule(lr){3-8}
 &  & MNIST & FMNIST & USPS & PBMC & BLADDER & NEURON \\
\midrule

\multirow{6}{*}{CNS}
& CE  & 0.72$\pm$0.02 & 0.83$\pm$0.02 & 0.74$\pm$0.02 & 0.81$\pm$0.01 & \textbf{0.83}$\pm$\textbf{0.03} & 0.64$\pm$0.05 \\
& SE  & \textbf{0.90}$\pm$\textbf{0.07} & 0.86$\pm$0.08 & 0.84$\pm$0.11 & 0.86$\pm$0.00 & 0.69$\pm$0.09 & 0.54$\pm$0.06 \\
& ICE & 0.78$\pm$0.05 & \textbf{0.94}$\pm$\textbf{0.01} & \textbf{0.84}$\pm$\textbf{0.04} & 0.94$\pm$0.04 & 0.79$\pm$0.01 & 0.65$\pm$0.01 \\
& IVIS & 0.83$\pm$0.09 & 0.85$\pm$0.65 & 0.77$\pm$0.10 & 0.87$\pm$0.07 & 0.71$\pm$0.13 & 0.63$\pm$0.12 \\
& csTSNE & NA & NA & 0.68$\pm$0.02 & \textbf{0.95}$\pm$\textbf{0.01} & 0.76$\pm$0.03 & \textbf{0.73}$\pm$\textbf{0.04} \\
& SUMAP & 0.77$\pm$0.05 & 0.93$\pm$0.01 & 0.68$\pm$0.04 & 0.88$\pm$0.03 & 0.63$\pm$0.02 & 0.64$\pm$0.02 \\

\midrule

\multirow{6}{*}{Stress}
& CE  & 0.29$\pm$0.02 & 0.31$\pm$0.02 & 0.27$\pm$0.02 & 0.29$\pm$0.03 & 0.80$\pm$0.11 & 0.86$\pm$0.05 \\
& SE  & 0.83$\pm$0.01 & 0.72$\pm$0.03 & \textbf{0.19}$\pm$\textbf{0.03} & 0.89$\pm$0.01 & 0.95$\pm$0.00 & 0.90$\pm$0.01 \\
& ICE & \textbf{0.23}$\pm$\textbf{0.01} & \textbf{0.17}$\pm$\textbf{0.02} & 0.21$\pm$0.01 & \textbf{0.15}$\pm$\textbf{0.01} & \textbf{0.20}$\pm$\textbf{0.01} & \textbf{0.27}$\pm$\textbf{0.00} \\
& IVIS & 0.44$\pm$0.05 & 0.47$\pm$0.03 & 0.45$\pm$0.02 & 0.62$\pm$0.04 & 0.63$\pm$0.02 & 0.62$\pm$0.01 \\
& csTSNE & NA & NA & 2.12$\pm$0.06 & 0.92$\pm$0.00 & 0.92$\pm$0.00 & 0.65$\pm$0.03 \\
& SUMAP & 2.91$\pm$0.04 & 1.58$\pm$0.04 & 1.32$\pm$0.04 & 0.92$\pm$0.00 & 0.92$\pm$0.00 & 0.54$\pm$0.02 \\

\midrule

\multirow{6}{*}{Procrustes}
& CE  & 0.40$\pm$0.24 & 0.17$\pm$0.07 & \textbf{0.05}$\pm$\textbf{0.04} & 0.20$\pm$0.11 & 0.07$\pm$0.03 & 0.47$\pm$0.16 \\
& SE  & 0.36$\pm$0.15 & 0.54$\pm$0.32 & 0.29$\pm$0.16 & \textbf{0.04}$\pm$\textbf{0.01} & 0.17$\pm$0.07 & 0.35$\pm$0.06 \\
& ICE & 0.39$\pm$0.25 & \textbf{0.02}$\pm$\textbf{0.00} & 0.22$\pm$0.17 & 0.31$\pm$0.17 & 0.18$\pm$0.08 & 0.72$\pm$0.17 \\
& IVIS & \textbf{0.24}$\pm$\textbf{0.16} & 0.09$\pm$0.04 & 0.08$\pm$0.11 & 0.09$\pm$0.02 & \textbf{0.04}$\pm$\textbf{0.05} & \textbf{0.33}$\pm$\textbf{0.13} \\
& csTSNE & NA & NA & 0.12$\pm$0.08 & 0.22$\pm$0.12 & 0.15$\pm$ 0.07 & 0.63$\pm$0.12 \\
& SUMAP & 0.33$\pm$0.17 & 0.13$\pm$0.07 & 0.37$\pm$0.19 & 0.12$\pm$0.03 & 0.23$\pm$0.09 & 0.77$\pm$0.12 \\

\midrule

\multirow{6}{*}{Correlation}
& CE  & 0.64$\pm$0.05 & 0.81$\pm$0.03 & 0.70$\pm$0.07 & 0.70$\pm$0.09 & 0.31$\pm$0.12 & -0.16$\pm$0.08 \\
& SE  & 0.67$\pm$0.04 & 0.87$\pm$0.02 & 0.75$\pm$0.04 & 0.26$\pm$0.05 & 0.45$\pm$0.13 & -0.24$\pm$0.06 \\
& ICE & \textbf{0.85}$\pm$\textbf{0.02} & \textbf{0.93}$\pm$\textbf{0.03} & \textbf{0.89}$\pm$\textbf{0.02} & \textbf{0.95}$\pm$\textbf{0.00} & \textbf{0.94}$\pm$\textbf{0.01} & 0.78$\pm$0.01 \\
& IVIS & 0.32$\pm$0.08 & 0.42$\pm$0.03 & 0.39$\pm$0.03 & 0.08$\pm$0.04 & -0.02$\pm$0.04 & 0.09$\pm$0.02 \\
& csTSNE & NA & NA & 0.46$\pm$ 0.14 & 0.63$\pm$0.02 & 0.01$\pm$0.04 & 0.79$\pm$0.02 \\
& SUMAP & 0.52$\pm$0.06 & 0.71$\pm$0.04 & 0.57$\pm$0.06 & 0.46$\pm$0.08 & 0.33$\pm$0.08 & \textbf{0.82}$\pm$\textbf{0.06} \\

\midrule
\multirow{6}{*}{Accuracy}
& CE  & 96.26$\pm$0.14 & 88.02$\pm$0.43 & 96.42$\pm$0.23 & 86.14$\pm$ 1.13 & 70.16$\pm$2.05 & 49.16$\pm$8.65 \\
& SE  & 94.93$\pm$0.24 & 81.65$\pm$0.72 & 96.19$\pm$0.20 & \textbf{88.44}$\pm$\textbf{0.63} & 74.18$\pm$1.44 & \textbf{76.52}$\pm$\textbf{1.80} \\
& ICE & \textbf{97.26}$\pm$\textbf{0.06} & \textbf{88.95}$\pm$\textbf{0.61} & \textbf{96.65}$\pm$\textbf{0.20} & 85.70$\pm$1.27 & 67.03$\pm$ 1.69 & 69.86$\pm$2.73 \\
& IVIS & 68.77$\pm$2.82 & 67.20$\pm$6.50 & 76.33$\pm$2.29 & 74.93$\pm$3.05 & 62.01$\pm$4.28 & 64.71$\pm$1.64 \\
& csTSNE & NA & NA & 96.23$\pm$0.19 & 82.91$\pm$0.12 & \textbf{83.93}$\pm$\textbf{0.32} & 67.83$\pm$1.14 \\
& SUMAP & 94.37$\pm$0.14 & 83.71$\pm$0.19 & 92.01$\pm$0.42 & 80.89$\pm$0.53 & 78.97$\pm$1.09 & 66.72$\pm$1.20 \\

\bottomrule
\end{tabular}}
\label{tab:2D_Results}
\end{table}

\noindent\textbf{Centroid geometry preservation:}
ICE consistently achieves the lowest Kruskal Stress and highest Pearson Correlation across most datasets, indicating improved preservation of centroid-distance relationships in 2D. The newly added Procrustes Disparity further shows that ICE remains highly competitive in terms of overall centroid geometry, although the best Procrustes score varies across datasets. This advantage is particularly evident on PBMC, BLADDER, and NEURON, where maintaining global class geometry is more challenging due to the larger number of classes. Methods such as CE and SE, which do not explicitly model inter-class relationships, exhibit higher Stress and lower Correlation, reflecting distortion of centroid geometry. Similarly, manifold-based approaches (SUMAP and csTSNE), which prioritize local neighborhood structure, show less consistent global distance preservation.

\noindent\textbf{Stability (CNS):}
Across the six datasets, ICE achieves the highest CNS on FMNIST and USPS. On higher-dimensional biological datasets (PBMC, BLADDER, NEURON), ICE remains consistently close to the best method despite the increased number of classes (8, 16, and 10, respectively), where preserving neighborhood structure in 2D becomes inherently difficult. In these datasets, csTSNE performs best in two cases and CE in one, while ICE consistently ranks second. Notably, csTSNE operates on PCA-reduced data, which may improve stability. In contrast, ICE achieves competitive CNS without such preprocessing while enforcing centroid geometry.

\noindent\textbf{Local neighborhood structure (KNN accuracy):}
All methods achieve competitive KNN accuracy in 2D, with ICE attaining the highest accuracy on MNIST, FMNIST, and USPS. This indicates that ICE preserves local neighborhood structure while maintaining global geometry. On more complex biological datasets, methods such as SE and csTSNE show comparable or higher accuracy, suggesting that strong local structure can yield good classification performance even when global geometry is distorted.

\noindent\textbf{Discussion:}
The two-dimensional results highlight an inherent trade-off between preserving local neighborhood structure and maintaining the global arrangement of class centroids. Although ICE consistently preserves centroid geometry better than competing methods, the limited representational capacity of a two-dimensional latent space prevents any method from faithfully capturing the complete centroid configuration. Consequently, all approaches must compromise between local discrimination and global geometric preservation. This observation motivates the evaluation of higher-dimensional embeddings, where the intrinsic centroid geometry can be represented more faithfully.

\subsection{Results in Higher Dimension (p = M)}
Table~\ref{tab:M-D_Results} presents performance in the $M$-dimensional embedding space, where the latent dimension is sufficient to represent intrinsic centroid geometry and assess preservation of global structure and local neighborhoods.

\noindent\textbf{Centroid geometry preservation:}
ICE achieves consistently low Kruskal Stress, Pearson Correlation close to one, and the lowest Procrustes Disparity on five of the six datasets, ranking a close second on NEURON (0.07 versus 0.06 for csTSNE). Together, these complementary metrics indicate that the learned embeddings faithfully preserve both pairwise centroid distances and the overall centroid configuration. This consistency across image and biological datasets demonstrates that the centroid-isometry constraint successfully recovers the original class geometry when sufficient latent dimensionality is available. In contrast, competing methods exhibit higher stress, lower correlation, and larger Procrustes Disparity, suggesting that increasing the embedding dimension alone is insufficient to recover the underlying centroid geometry without explicitly modeling inter-centroid relationships.

\noindent\textbf{Stability (CNS):}
ICE attains the highest CNS values across most datasets, with scores approaching one, indicating highly stable class-level neighborhoods. This improvement over the 2D case reflects the role of sufficient dimensionality in stabilizing embeddings. While csTSNE performs competitively on some biological datasets, ICE remains consistently strong without requiring preprocessing.

\begin{table}[!ht]
\small
\caption{Results on M-dimensional embedding across datasets. Numbers for csTSNE on MNIST and FMNIST are not reported due to computational limitations on full datasets.}
\vspace{0.1 cm}
\adjustbox{max width=\linewidth}{\begin{tabular}{llcccccc}
\toprule
\textbf{Metric} & \textbf{Model} & \multicolumn{6}{c}{\textbf{Datasets}} \\
\cmidrule(lr){3-8}
 &  & MNIST & FMNIST & USPS & PBMC & BLADDER & NEURON \\
\midrule

\multirow{6}{*}{CNS}
& CE  & 0.86$\pm$0.04 & 0.88$\pm$0.04 & 0.84$\pm$0.03 & 0.89$\pm$0.03 & 0.79$\pm$0.02 & 0.68$\pm$0.11 \\
& SE  & 0.87$\pm$0.12 & 0.98$\pm$0.04 & 0.68$\pm$0.09 & 0.94$\pm$0.00 & 0.79$\pm$0.14 & 0.62$\pm$0.09 \\
& ICE & \textbf{0.98}$\pm$\textbf{0.01} & \textbf{0.99}$\pm$\textbf{0.01} & \textbf{1.00}$\pm$\textbf{0.00} & \textbf{0.99}$\pm$\textbf{0.01} & 0.91$\pm$0.01 & \textbf{0.99}$\pm$\textbf{0.02} \\
& IVIS & 0.83$\pm$0.09 & 0.86$\pm$0.10 & 0.77$\pm$0.10 & 0.87$\pm$0.07 & 0.72$\pm$0.13 & 0.63$\pm$0.12 \\
& csTSNE & NA & NA & 0.90$\pm$0.02 & 0.98$\pm$0.02 & \textbf{0.97}$\pm$\textbf{0.01} & 0.94$\pm$0.00 \\
& SUMAP & 0.71$\pm$0.02 & 0.93$\pm$0.01 & 0.73$\pm$0.03 & 0.93$\pm$0.01 & 0.84$\pm$0.01 & 0.87$\pm$0.00 \\

\midrule

\multirow{6}{*}{Stress}
& CE  & 0.21$\pm$0.02 & 0.19$\pm$0.05 & 0.23$\pm$0.05 & 0.27$\pm$0.03 & 0.85$\pm$0.08 & 0.98$\pm$0.13 \\
& SE  & 0.50$\pm$0.05 & 0.75$\pm$0.00 & 0.55$\pm$0.02 & 0.93$\pm$0.01 & 0.98$\pm$0.00 & 0.97$\pm$0.00 \\
& ICE & \textbf{0.01}$\pm$\textbf{0.00} & \textbf{0.01}$\pm$\textbf{0.00} & \textbf{0.01}$\pm$\textbf{0.00} & \textbf{0.06}$\pm$\textbf{0.01} & \textbf{0.07}$\pm$\textbf{0.01} & \textbf{0.07}$\pm$\textbf{0.01} \\
& IVIS & 0.47$\pm$0.07 & 0.47$\pm$0.03 & 0.45$\pm$0.02 & 0.62$\pm$0.04 & 0.63$\pm$0.03 & 0.62$\pm$0.01 \\
& csTSNE & 0.73$\pm$0.00 & 0.70$\pm$0.00 & 1.93$\pm$ 0.02 & 0.92$\pm$0.01 & 0.92$\pm$0.00 & 0.68$\pm$0.03 \\
& SUMAP & 2.35$\pm$0.02 & 1.19$\pm$0.01 & 1.07$\pm$0.02 & 0.94$\pm$0.00 & 0.94$\pm$0.00 & 0.64$\pm$0.01 \\

\midrule

\multirow{6}{*}{Procrustes}
& CE  & 0.07$\pm$0.01 & 0.28$\pm$0.05 & 0.09$\pm$0.02 & 0.29$\pm$0.09 & 0.09$\pm$0.01 & 0.17$\pm$0.03 \\
& SE  & 0.16$\pm$0.03 & 0.19$\pm$0.04 & 0.12$\pm$0.03 & 0.08$\pm$0.03 & 0.33$\pm$0.08 & 0.65$\pm$0.06 \\
& ICE & \textbf{0.02}$\pm$\textbf{0.01} & \textbf{0.02}$\pm$\textbf{0.00} & \textbf{0.02}$\pm$\textbf{0.01} & \textbf{0.02}$\pm$\textbf{0.00} & \textbf{0.01}$\pm$\textbf{0.00} & 0.07$\pm$0.01 \\
& IVIS & 0.12$\pm$0.09 & 0.16$\pm$0.12 & 0.09$\pm$0.13 & 0.14$\pm$0.05 & 0.20$\pm$0.11 & 0.48$\pm$0.18 \\
& csTSNE & NA & NA & 0.03$\pm$0.01 & 0.09$\pm$0.06 & 0.01$\pm$0.01 & 0.06$\pm$0.02 \\
& SUMAP & 0.05$\pm$0.01 & \textbf{0.02}$\pm$\textbf{0.00} & 0.15$\pm$0.03 & 0.03$\pm$0.01 & 0.04$\pm$0.01 & \textbf{0.06}$\pm$\textbf{0.01} \\

\midrule

\multirow{6}{*}{Correlation}
& CE  & 0.67$\pm$0.03 & 0.88$\pm$0.02 & 0.77$\pm$0.06 & 0.74$\pm$0.09 & 0.37$\pm$0.07 & $-$0.16$\pm$0.09 \\
& SE  & 0.81$\pm$0.02 & 0.89$\pm$0.01 & 0.69$\pm$0.06 & 0.57$\pm$0.06 & 0.75$\pm$0.09 & -0.35$\pm$0.03 \\
& ICE & \textbf{1.00}$\pm$\textbf{0.00} & \textbf{1.00}$\pm$\textbf{0.00} & \textbf{1.00}$\pm$\textbf{0.00} & \textbf{1.00}$\pm$\textbf{0.00} & \textbf{1.00}$\pm$\textbf{0.00} & \textbf{1.00}$\pm$\textbf{0.00} \\
& IVIS & 0.26$\pm$0.11 & 0.42$\pm$0.04 & 0.39$\pm$0.03 & 0.08$\pm$0.04 & -0.02$\pm$0.04 & 0.09$\pm$0.02 \\
& csTSNE & 0.26$\pm$0.00 & 0.45$\pm$0.00 & 0.81$\pm$0.04 & 0.66$\pm$0.01 & 0.06$\pm$0.00 & 0.92$\pm$0.01 \\
& SUMAP & $-$0.06$\pm$0.09 & 0.61$\pm$0.02 & 0.00$\pm$0.17 & 0.52$\pm$0.04 & 0.16$\pm$0.10 & 0.88$\pm$0.02 \\

\midrule

\multirow{6}{*}{Accuracy}
& CE  & 97.41$\pm$0.15 & 88.80$\pm$0.19 & 97.34$\pm$0.21 & \textbf{89.89}$\pm$\textbf{0.86} & \textbf{81.74}$\pm$\textbf{1.15} & 33.40$\pm$7.04 \\
& SE  & 90.20$\pm$0.78 & 88.65$\pm$0.12 & 94.62$\pm$0.48 & 81.85$\pm$1.48 & 59.35$\pm$1.98 & 64.73$\pm$2.08 \\
& ICE & \textbf{97.92}$\pm$\textbf{0.08} & \textbf{89.23}$\pm$\textbf{0.22} & \textbf{97.72}$\pm$\textbf{0.15} & 89.27$\pm$0.45 & 78.12$\pm$0.95 & \textbf{87.14}$\pm$\textbf{1.00} \\
& IVIS & 68.77$\pm$2.82 & 67.23$\pm$6.39 & 76.33$\pm$2.29 & 74.93$\pm$3.05 & 62.01$\pm$4.28 & 64.71$\pm$1.64 \\
& csTSNE & 87.33$\pm$0.00 & 78.56$\pm$0.03 & 95.59$\pm$0.04 & 82.42$\pm$0.07 & 81.63$\pm$0.06 & 75.11$\pm$0.47 \\
& SUMAP & 93.98$\pm$0.12 & 83.33$\pm$0.06 & 91.30$\pm$0.45 & 80.56$\pm$0.35 & 77.13$\pm$0.24 & 69.00$\pm$0.60 \\

\bottomrule
\end{tabular}}
\label{tab:M-D_Results}
\end{table}

\noindent\textbf{Local neighborhood structure (KNN accuracy):}
All methods show improved KNN accuracy compared to 2D, reflecting better local neighborhood preservation. ICE achieves the highest accuracy on MNIST, FMNIST, USPS, and NEURON, and remains competitive on PBMC and BLADDER, indicating that it preserves local structure while maintaining global geometry.

\noindent\textbf{On the role of CNS as a stability metric:}
The benchmark results illustrate the complementary roles of geometric preservation and embedding stability. Kruskal Stress, Pearson Correlation, and Procrustes Disparity quantify different aspects of agreement between the original and learned centroid geometries, whereas CNS evaluates whether class-level neighborhood relationships remain consistent across independent runs. This distinction is particularly evident in the 2D results, where ICE achieves strong Stress and Correlation but does not always attain the highest CNS. For example, csTSNE achieves higher CNS on some biological datasets despite weaker preservation of the original centroid geometry. Thus, accurate recovery of centroid geometry does not necessarily imply reproducible class neighborhoods across runs.

The comparison between the 2D and $M$-dimensional results further highlights the effect of representational capacity. In 2D, ICE is competitive in terms of Procrustes Disparity but does not consistently outperform the competing methods. At $p=M$, however, ICE achieves the lowest Procrustes Disparity on nearly every dataset while also attaining low Stress, high Correlation, and generally high CNS. This behavior is consistent with Lemma~\ref{lemma:centroid_dim}: when sufficient latent dimensionality is available, the centroid-isometry objective can more accurately represent the prescribed class-centroid geometry. CNS complements these geometric measures by additionally quantifying whether the resulting class-neighborhood structure is reproducible across independent runs.

\begin{figure}[!h]
    \centering
    \includegraphics[width=17.5 cm, height=16.5 cm]{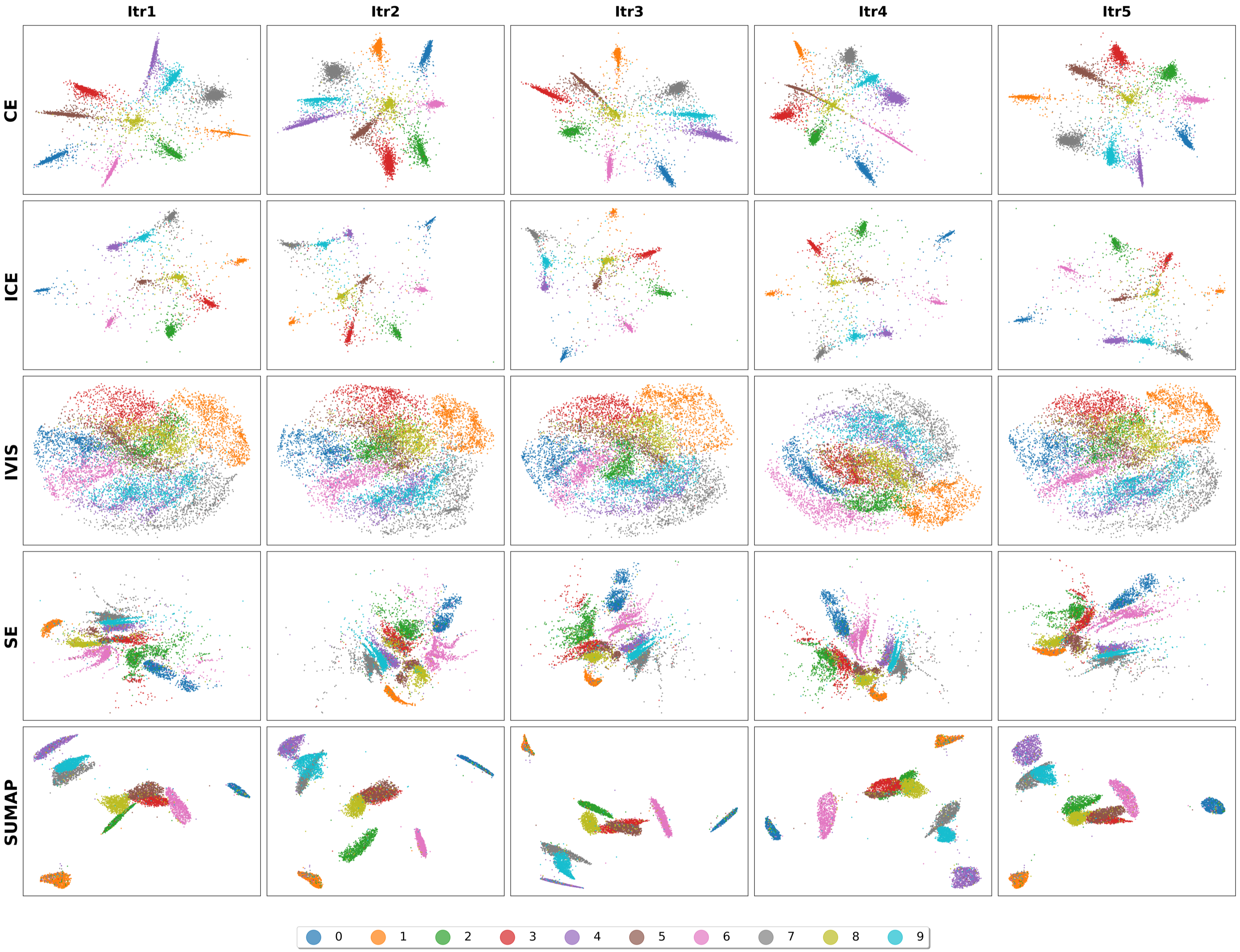}
    \caption{Two-dimensional scatter plots of MNIST test digits across five independent runs for each method.}
    \label{fig:embding_comparison_2D}
    \vspace{0.25 cm}
\end{figure}

\noindent\textbf{Visualization protocol:}
We report both direct two-dimensional embeddings and two-dimensional projections obtained via PCA from the corresponding $M$-dimensional representations. The direct embeddings reflect how each method operates under strict low-dimensional constraints, while the PCA-based projections provide a qualitative visualization of the learned higher-dimensional geometry. PCA was selected because it is deterministic, parameter-free, and does not require additional optimization or hyperparameter tuning, thereby avoiding the variability introduced by nonlinear visualization methods such as t-SNE or UMAP. The PCA projections are used solely for visualization and are not used for quantitative evaluation. All quantitative conclusions are based on the reported CNS, Kruskal Stress, Pearson correlation, kNN accuracy, and Procrustes Disparity. The inclusion of both views ensures that comparisons are not biased toward methods optimized specifically for two-dimensional layouts, but instead reflect the consistency of the learned representations.

Figure \ref{fig:embding_comparison_2D} presents the direct two-dimensional embeddings of MNIST test digits across five independent runs. CE, ICE, and SUMAP form visually distinct clusters, whereas IVIS fails to produce meaningful class separation. Although CE separates classes reasonably well, its class neighborhood structure varies across runs. For example, digits 0 (blue) and 1 (orange) are far apart in run 1 but appear adjacent in run 2; similarly, the neighborhood of digits such as 2 (green), 3 (red), and 5 (brown) changes across runs. SUMAP shows similar behavior: while it produces compact clusters, the relative arrangement of classes is unstable, with groups such as 2, 3, 5, and 6 merging in some runs and separating in others.

Compared to these methods, ICE produces more consistent embeddings across runs. The relative arrangement of classes is largely preserved, with visually similar digits such as 4, 7, and 9 consistently appearing as neighbors, and dissimilar digits such as 0 and 1 remaining well separated. Some variability remains in low-dimensional embeddings, particularly among digits such as 2, 3, and 5. Super-Encoder (SE) achieves the highest CNS in this setting, but its embeddings collapse class structure, failing to preserve meaningful inter-class geometry, as reflected by high stress and lower correlation. IVIS produces diffuse, near-spherical distributions with limited separation.

\begin{figure}[!ht]
    \centering
    \includegraphics[width=17.5 cm, height=16.5 cm]{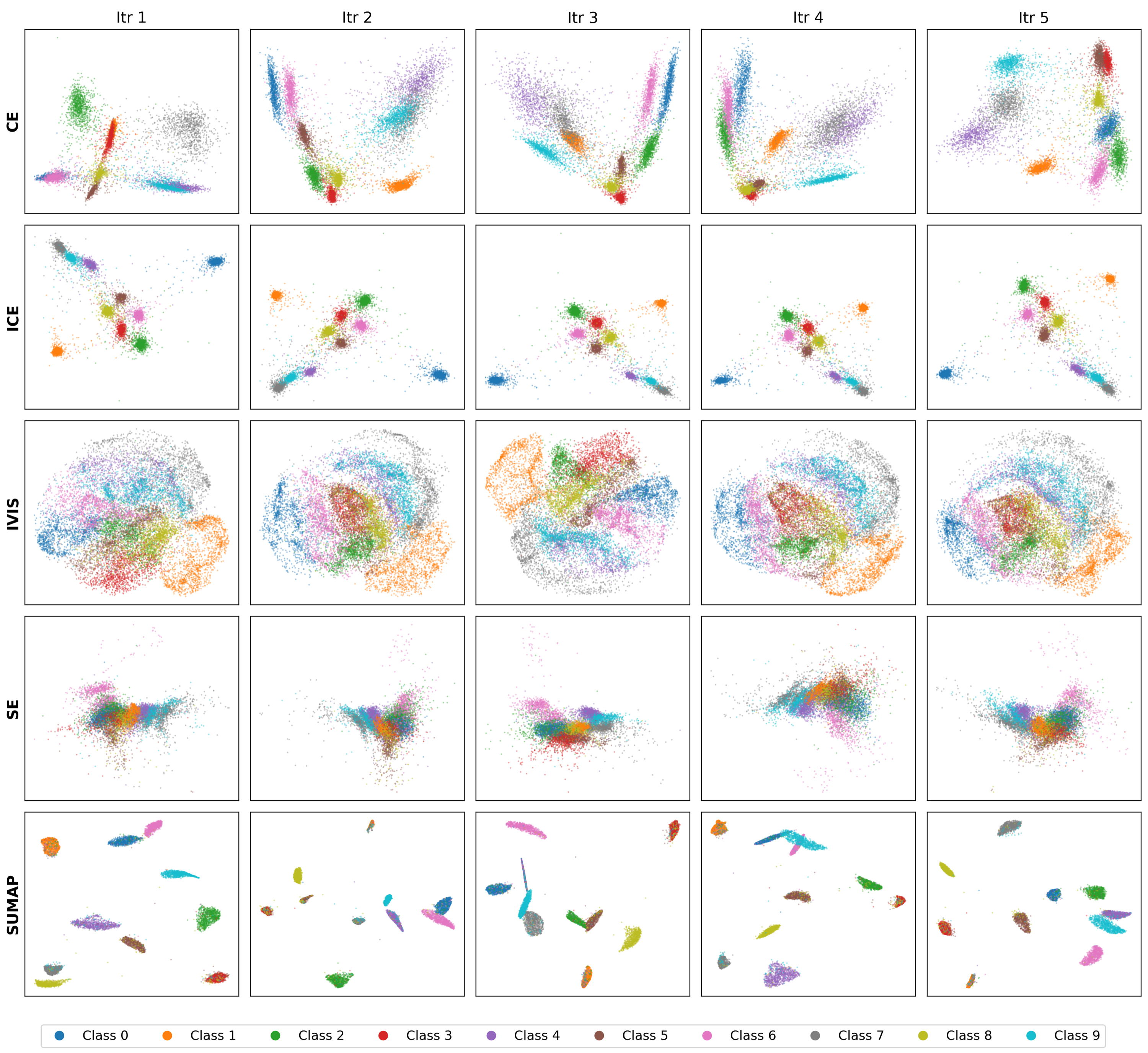}
    \caption{Two-dimensional scatter plots of MNIST test digits across five independent runs for each method. The 2D embeddings are obtained by applying PCA to the corresponding $M$-dimensional representations. }
    \vspace{0.25 cm}
    \label{fig:embding_comparison_M-D}
\end{figure}

Figure \ref{fig:embding_comparison_M-D} shows two-dimensional visualizations obtained via PCA from the $M$-dimensional embeddings. Here, the plots reflect how well each method preserves learned geometry rather than how it optimizes directly in 2D. CE exhibits increased overlap compared to its direct 2D embedding, indicating that centroid reconstruction alone does not ensure stable global geometry. SUMAP continues to form distinct clusters, but the relative arrangement varies significantly across runs. SE maintains similar CNS across runs, but its PCA-based embeddings show increased overlap between classes, indicating a collapse of inter-class geometry. IVIS remains diffuse and poorly separated.

In contrast, ICE produces remarkably stable embeddings across all runs while maintaining clear class separation. The relative arrangement of classes remains consistent, with dissimilar digits such as 0 and 1 always well separated and structurally similar digits such as 4, 7, and 9 consistently forming neighborhoods. Digits with similar geometric features (2, 3, 5, 6, and 8) form coherent clusters in the central region. Importantly, ICE shows improved stability when moving from direct 2D embeddings to PCA-based projections of the $M$-dimensional representations. This behavior is explained by Lemma~\ref{lemma:centroid_dim}, which states that centroid configurations can be faithfully represented only when the latent dimension is at least $M-1$. In the direct 2D case ($p=2$), the model operates in a space-constrained configuration, leading to distortions in centroid relationships and reduced stability. In contrast, the $M$-dimensional embedding provides sufficient degrees of freedom to realize centroid geometry, and the PCA projection preserves this structure.

The observed increase in CNS and consistency across runs is therefore a direct consequence of the underlying geometry. This highlights the interpretability of ICE and its ability to preserve both local neighborhood structure and global class geometry in a robust and theoretically grounded manner.

\subsection{Visualization of FMNIST}
\label{supp:sec:supp_fmnist}

To verify that the qualitative behavior observed on MNIST extends to other domains, we present the same two visualization protocols for FMNIST and 10XPBMC.

\begin{figure}[!ht]
    \centering
    \includegraphics[width=17.5 cm, height=16.5 cm]{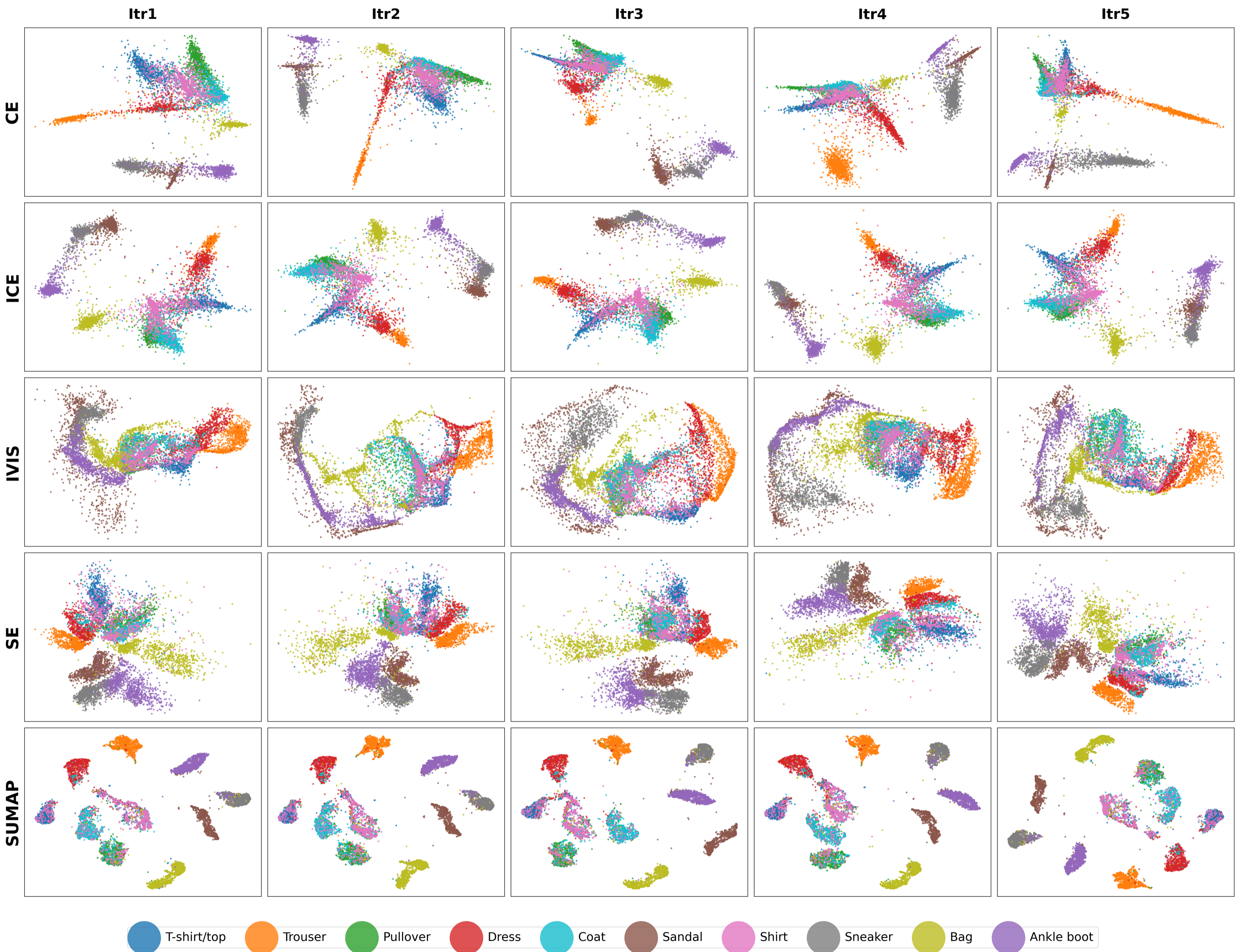}
    \caption{Two-dimensional scatter plots of FMNIST fashion item from test set across five independent runs for each method. }
    \label{supp:fig:embding_comparison_2D}
\end{figure}

We start with the direct two-dimensional layout of the fashion items of different models in Figure ~\ref{supp:fig:embding_comparison_2D}. For CE, the visualization of the \emph{trouser} class varies noticeably across runs. In most cases, the samples form a near-linear structure, with the exception of iteration 4 where the cluster becomes more elliptical. A similar behavior is observed for the \emph{dress} class. In addition, the relative positions among footwear classes are not consistent: while \emph{sandal} typically lies between \emph{ankle boot} and \emph{sneaker}, this relationship changes in iteration 3.

In contrast, the embeddings produced by ICE are highly consistent across runs. The three footwear classes (\emph{ankle boot}, \emph{sneaker}, and \emph{sandal}) are consistently grouped on one side of the embedding, while apparel classes (e.g., \emph{coat}, \emph{t-shirt}, \emph{shirt}) occupy the opposite side, with \emph{bag} positioned between these groups. Although CE occasionally exhibits a similar coarse structure, ICE maintains both consistent class shapes and stable relative positioning across runs.

For IVIS, the global structure varies substantially across runs, despite partial consistency in relative class positioning. In iteration 3, the embedding exhibits a circular structure, whereas in other runs it appears more diffuse. Significant overlap is observed among classes, particularly for \emph{bag}, which lies close to multiple categories despite having a distinct shape. Super-Encoder (SE) separates footwear from apparel to some extent, but the \emph{bag} class remains poorly separated. SUMAP tends to distribute classes more uniformly, without a clear separation between footwear and apparel as observed in ICE and CE. Moreover, class neighborhoods vary across runs; for example, \emph{trouser} and \emph{sneaker} exhibit changing nearest neighbors.

\begin{figure}[!ht]
    \centering
    \includegraphics[width=17.5 cm, height=16.5 cm]{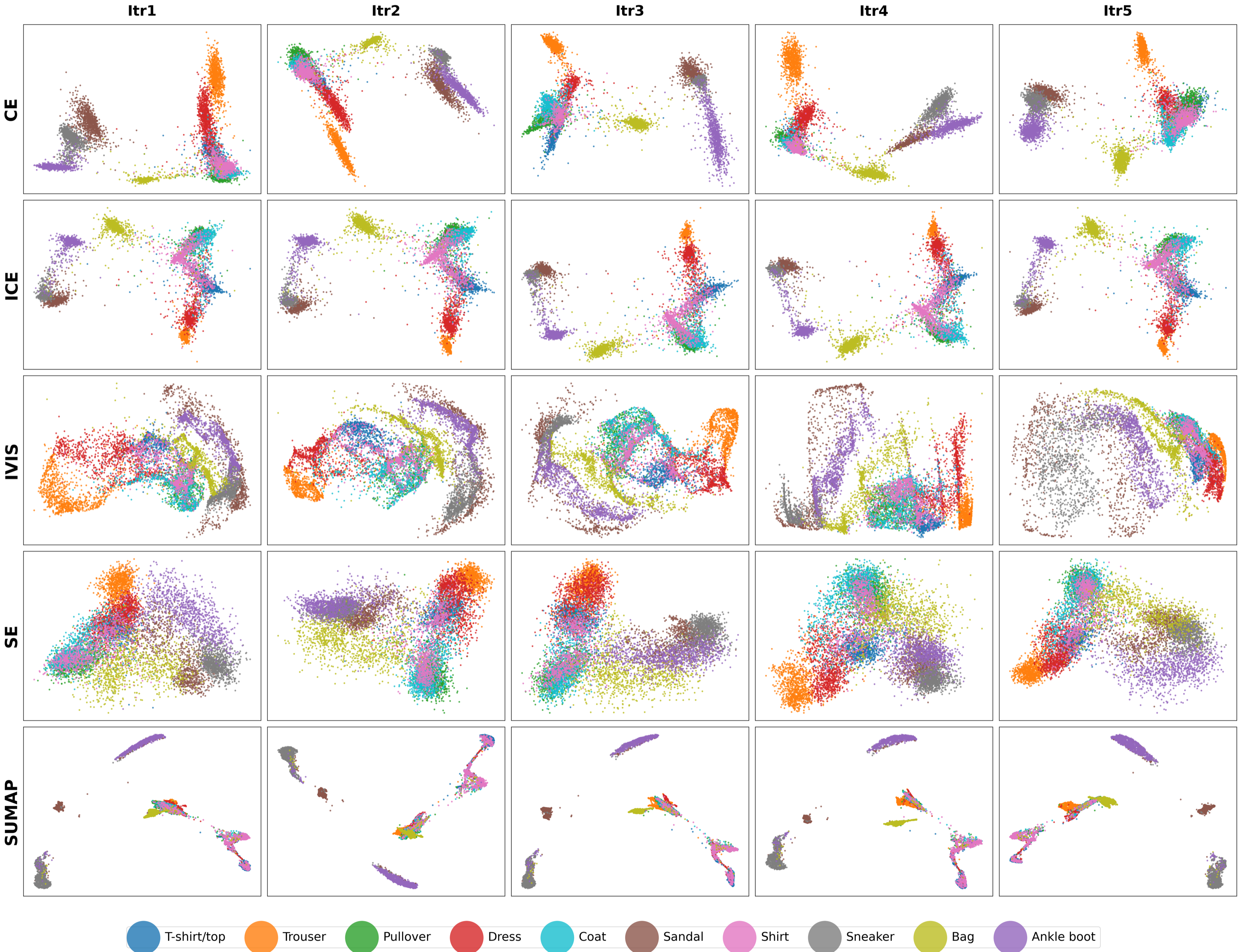}
    \caption{Two-dimensional scatter plots of FMNIST fashion item from test set across five independent runs for each method. The 2D embeddings are obtained by applying PCA to the corresponding $M$-dimensional representations. }
    \label{supp:fig:embding_comparison_M-D_2D}
\end{figure}

When examining the two-dimensional embeddings obtained via PCA from the common $M$-dimensional space in Figure ~\ref{supp:fig:embding_comparison_M-D_2D} , ICE better preserves both global and local structure across runs. Although certain classes (e.g., \emph{pullover}) are not consistently well separated, the overall arrangement remains stable. In contrast, SUMAP shows the largest variation in both global layout and local neighborhoods, followed by SE, IVIS, and CE. These observations suggest that these methods do not consistently preserve class geometry in higher-dimensional representations.

\subsection{Visualization of  10XPBMC Embeddings}
\label{supp:sec:supp_pbmc}

\begin{figure}[!ht]
    \centering
    \includegraphics[width=17.5 cm, height=16.5 cm]{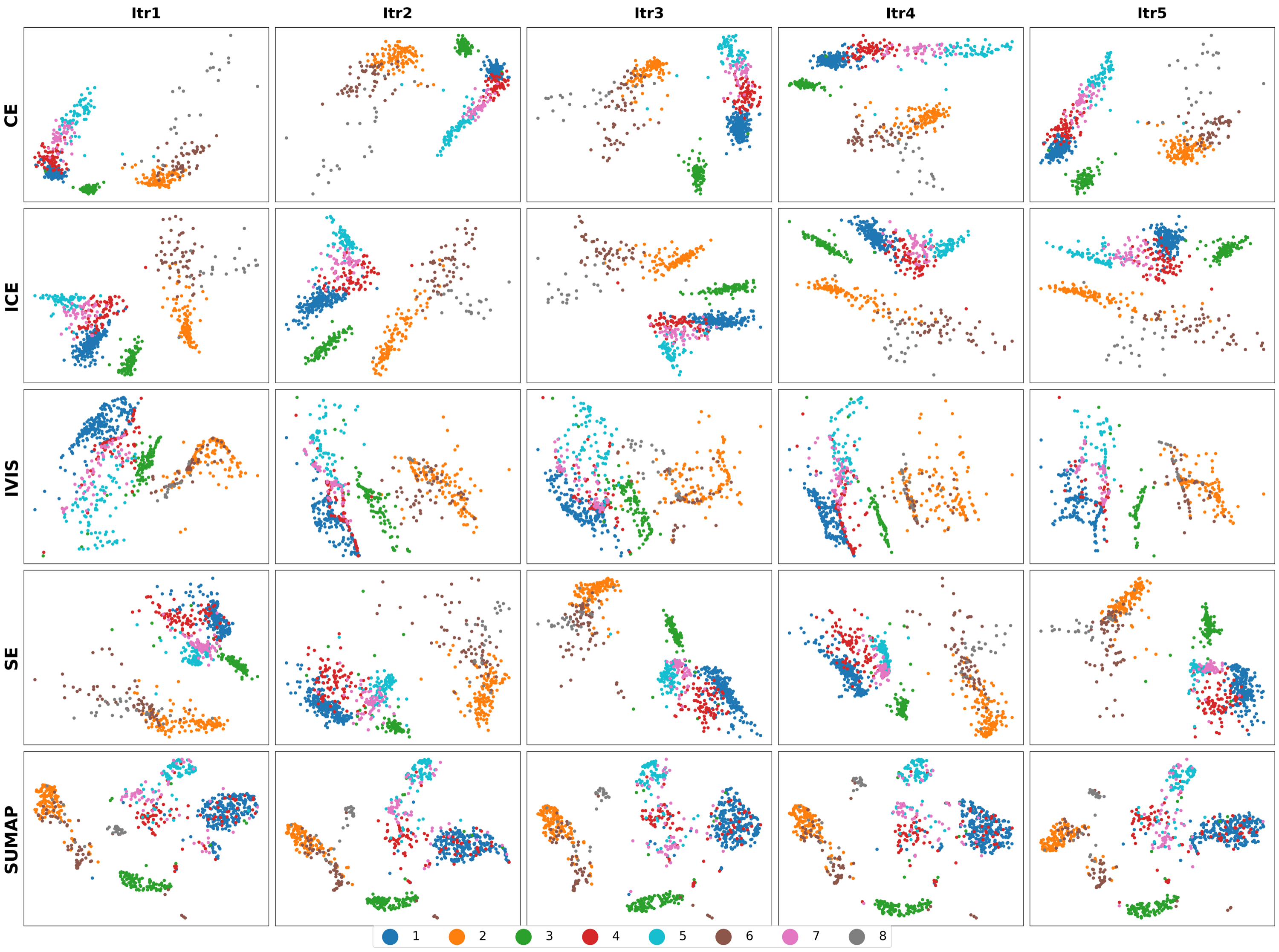}
    \caption{Two-dimensional scatter plots of 10XPBMC test samples across five independent runs for each method. }
    \label{supp:fig:10XPBMC_embding_comparison_2D}
\end{figure}

In Figures~\ref{supp:fig:10XPBMC_embding_comparison_2D} and~\ref{supp:fig:10XPBMC_embding_comparison_M-D_2D}, we present the visualization of the real-world biological dataset 10XPBMC. This dataset contains eight classes, with each sample residing in a high-dimensional space $\mathbb{R}^{16653}$, making it challenging to preserve class relationships in low-dimensional embeddings. We compare both direct two-dimensional embeddings (Figure~\ref{supp:fig:10XPBMC_embding_comparison_2D}) and projections obtained via PCA from the $M$-dimensional space (Figure~\ref{supp:fig:10XPBMC_embding_comparison_M-D_2D}). In most cases, seven out of the eight classes are distinguishable, with classes indexed 6 (brown) and 8 (gray) showing noticeable overlap.

\begin{figure}[!ht]
    \centering
    \includegraphics[width=17.5 cm, height=16.5 cm]{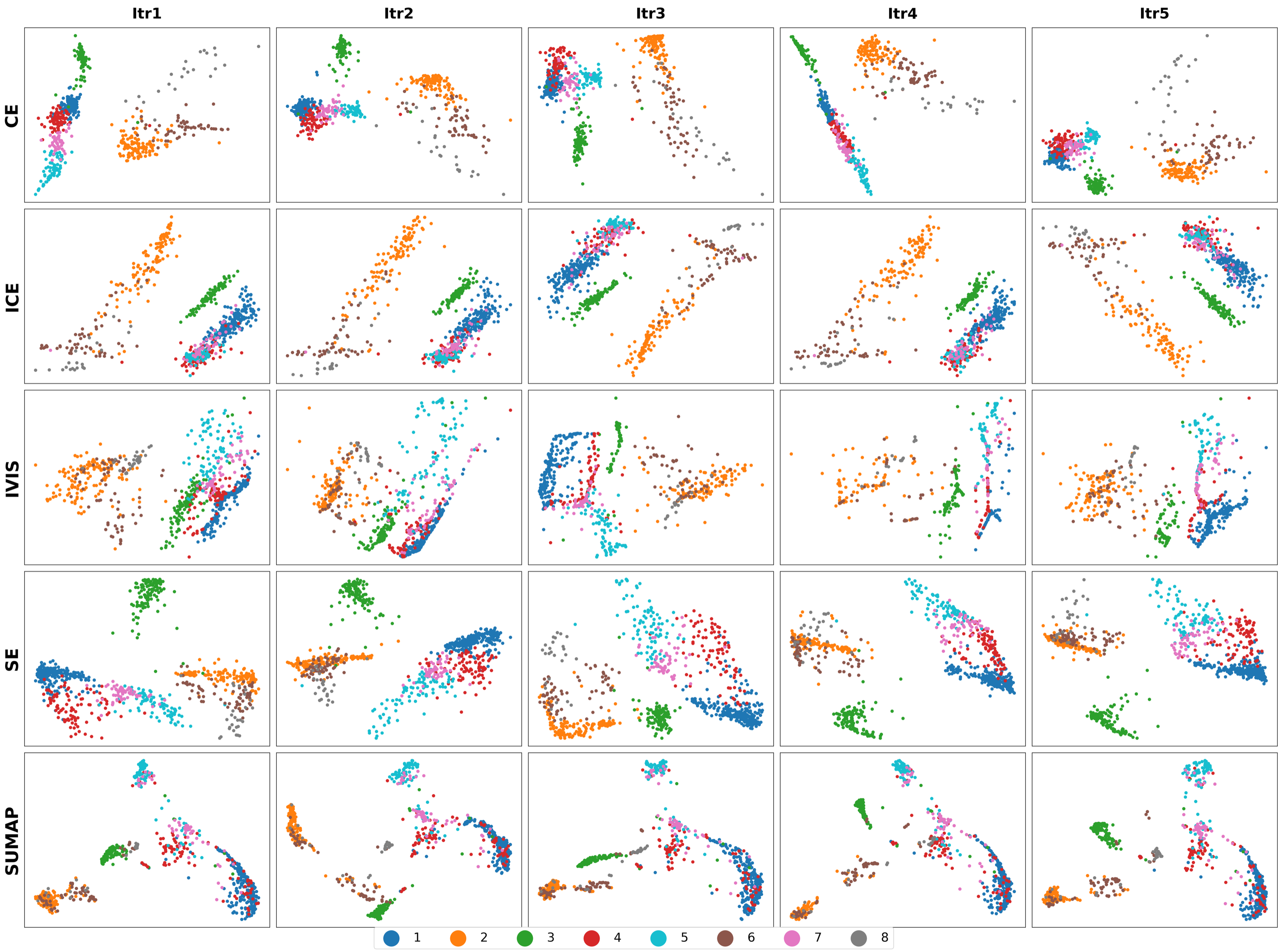}
    \caption{Two-dimensional scatter plots of 10XPBMC test  set across five independent runs for each method. The 2D embeddings are obtained by applying PCA to the corresponding $M$-dimensional representations. }
    \label{supp:fig:10XPBMC_embding_comparison_M-D_2D}
\end{figure}

SUMAP exhibits mixing between class 1 (blue) and class 7 (pink), while IVIS shows clear variability in its global structure across runs. Super-Encoder (SE) produces improved separation compared to its behavior on FMNIST, though some class overlap remains. In contrast, ICE maintains a consistent global arrangement across runs and shows strong agreement between direct 2D embeddings and PCA-based projections from the $M$-dimensional space. In the ICE embedding projected from the $M$-dimensional space, classes 4 (red), 7 (pink), and 5 (light blue) exhibit noticeable mixing, despite being well separated in the direct two-dimensional case. Despite this, the ICE embeddings remain stable and consistent across runs.

\subsection{Computational Efficiency}
\label{supp:sec:supp_efficiency}

The proposed ICE framework removes the decoder required by the original Centroid-Encoder (CE) and therefore employs a substantially simpler encoder-only architecture. To quantify the computational impact of this design, we compared the optimized implementation of ICE with CE on the MNIST dataset and on synthetic datasets of increasing size. The synthetic datasets consist of class-separated Gaussian blobs with 784 features and 10 classes, matching the dimensionality and number of classes of MNIST, with dataset sizes ranging from 5K to 300K samples. All experiments were performed under identical hardware and software environments.

Training efficiency was evaluated using the average training time per epoch, total training time, peak GPU memory allocation, and peak CPU resident memory (RSS). Peak GPU memory was measured using CUDA's built-in peak memory counters by resetting the statistics immediately before training and recording the maximum allocated memory after training. Peak CPU RSS was measured using a background process that sampled the resident memory of the training process every 50 ms throughout optimization.

\begin{table}[!ht]

\adjustbox{max width=\linewidth}{\begin{tabular}{lcccccc}
\hline
Dataset &
Model &
Parameters &
Time/Epoch (s) &
Total Training (s) &
Peak GPU (MB) &
Peak CPU RSS (MB) \\
\hline
MNIST (70K)
& ICE & 1.35M & 1.605 & 16.1 & 279 & 2372 \\
& CE  & 2.70M & 1.560 & 23.4 & 89 & 2668 \\
\hline
Synthetic (5K)
& ICE & 1.35M & 0.240 & 2.4 & 51 & 1976 \\
& CE  & 2.70M & 0.174 & 2.6 & 89 & 1691 \\
\hline
Synthetic (20K)
& ICE & 1.35M & 0.549 & 5.5 & 79 & 2066 \\
& CE  & 2.70M & 0.495 & 7.4 & 89 & 1914 \\
\hline
Synthetic (60K)
& ICE & 1.35M & 1.388 & 13.9 & 236 & 2310 \\
& CE  & 2.70M & 1.347 & 20.2 & 89 & 2517 \\
\hline
Synthetic (150K)
& ICE & 1.35M & 3.276 & 32.8 & 541 & 2853 \\
& CE  & 2.70M & 3.275 & 49.1 & 89 & 3862 \\
\hline
Synthetic (300K)
& ICE & 1.35M & 6.463 & 64.6 & 1037 & 3760 \\
& CE  & 2.70M & 6.499 & 97.5 & 89 & 6118 \\
\hline
\end{tabular}}
\caption{Computational comparison between the optimized implementation of ICE and Centroid-Encoder (CE). Peak GPU memory is measured using CUDA's peak allocation counters during training, while Peak CPU RSS denotes the maximum resident memory occupied by the training process.}
\label{supp:tab:computational_cost}
\end{table}

\noindent\textbf{Discussion:}
The optimized implementation of ICE contains approximately half the number of trainable parameters of CE (1.35M versus 2.70M) because it eliminates the decoder network and directly learns the latent representation. Consequently, ICE consistently reduces the overall training time across all datasets. For example, on the largest synthetic dataset (300K samples), the total training time decreases from 97.5\,s for CE to 64.6\,s for ICE, corresponding to approximately a 34\% reduction. Similar improvements are observed across the remaining benchmark datasets.

The CPU memory usage also favors ICE for larger datasets. Although CE requires slightly less CPU memory for the smallest synthetic datasets, its memory usage increases more rapidly as the dataset size grows. At 300K samples, the peak CPU RSS decreases from 6118\,MB for CE to 3760\,MB for ICE. This behavior is primarily due to CE constructing an additional full-size reconstruction target in CPU memory together with its encoder-decoder architecture, whereas ICE stores only the original training data and employs an encoder-only network.

The GPU memory measurements exhibit a different trend, with ICE requiring more GPU memory than CE on the larger datasets. This behavior is implementation-specific rather than algorithmic. Before mini-batch optimization begins, ICE computes the ambient-space class centroids required by the centroid-isometry objective by temporarily transferring the entire training set to the GPU. Consequently, the peak GPU memory increases with dataset size, even though the subsequent optimization is performed using mini-batches of size 512. In contrast, CE computes the corresponding ambient-space centroids on the CPU, and only mini-batches are transferred to the GPU during training, resulting in an almost constant GPU memory footprint across all dataset sizes. This one-time centroid computation is independent of the proposed optimization improvements and can be removed by computing the ambient-space centroids on the CPU, analogous to the current CE implementation.

Overall, these experiments demonstrate that the encoder-only formulation of ICE substantially reduces computational complexity while preserving the modeling capability of the original framework. The optimized implementation consistently provides faster training with lower CPU memory requirements on large datasets, whereas the observed increase in GPU memory is attributable to a one-time implementation choice for ambient-space centroid computation rather than to the ICE learning objective itself.

\section{Conclusion and Future Work}
\label{conc_future_work}
In this paper, we proposed the Isometric Centroid-Encoder (ICE), a supervised dimensionality reduction framework that combines centroid consistency with explicit preservation of pairwise class-centroid geometry. By introducing a centroid-isometry objective, ICE learns embeddings that preserve the global arrangement of classes while maintaining  local class localization.

A central contribution of this work is the theoretical and empirical characterization of class-centroid geometry in supervised embeddings. We showed that the intrinsic dimensionality required to represent the centroid configuration of $M$ classes is at most $M-1$, explaining why direct two-dimensional embeddings inevitably distort class geometry and exhibit greater run-to-run variability. Extensive analyses, including optimization dynamics, sensitivity studies, computational efficiency, and PCA-based preprocessing, further support the proposed formulation and provide practical guidelines for applying ICE.

To quantify embedding reproducibility, we introduced the Centroid Neighborhood Stability (CNS) metric, which measures the consistency of class-level neighborhood relationships across independent runs. Together with Kruskal Stress, Pearson Correlation, and Procrustes Disparity, CNS provides a complementary perspective on embedding quality by explicitly evaluating stability rather than geometric preservation alone. Our experiments demonstrate that these measures capture different aspects of embedding behavior and should be considered jointly when evaluating supervised dimensionality reduction methods.

The experimental results further show that, although two-dimensional embeddings remain valuable for visualization, faithfully preserving the complete class geometry generally requires a latent dimensionality close to $M-1$. When sufficient latent dimensionality is available, ICE consistently produces stable embeddings that accurately preserve the prescribed centroid geometry across both image and biological datasets.

Future work will investigate relaxed centroid-isometry objectives for extremely low-dimensional embeddings, uncertainty-aware formulations that explicitly model centroid variability, and extensions incorporating sparsity and domain-specific class-distance measures. More broadly, we believe that explicitly modeling geometry together with embedding stability provides a promising direction for developing more interpretable and reliable supervised dimensionality reduction methods.

\vskip 0.2in
\bibliography{mybibfile}

\begin{thebibliography}{39}
\providecommand{\natexlab}[1]{#1}
\providecommand{\url}[1]{\texttt{#1}}
\expandafter\ifx\csname urlstyle\endcsname\relax
  \providecommand{\doi}[1]{doi: #1}\else
  \providecommand{\doi}{doi: \begingroup \urlstyle{rm}\Url}\fi

\bibitem[Artemiou et~al.(2021)Artemiou, Dong, and Shin]{artemiou2021real}
Andreas Artemiou, Yuexiao Dong, and Seung~Jun Shin.
\newblock Real-time sufficient dimension reduction through principal least
  squares support vector machines.
\newblock \emph{Pattern Recognition}, 112:\penalty0 107768, 2021.

\bibitem[Bair et~al.(2006)Bair, Hastie, Paul, and Tibshirani]{bair2006}
Eric Bair, Trevor Hastie, Debashis Paul, and Robert Tibshirani.
\newblock Prediction by supervised principal components.
\newblock \emph{Journal of the American Statistical Association}, 101\penalty0
  (473):\penalty0 119--137, 2006.
\newblock \doi{10.1198/016214505000000628}.
\newblock URL \url{https://doi.org/10.1198/016214505000000628}.

\bibitem[Barshan et~al.(2011)Barshan, Ghodsi, Azimifar, and
  Zolghadri~Jahromi]{Barshan:2011:SPC:1950989.1951174}
Elnaz Barshan, Ali Ghodsi, Zohreh Azimifar, and Mansoor Zolghadri~Jahromi.
\newblock Supervised principal component analysis: Visualization,
  classification and regression on subspaces and submanifolds.
\newblock \emph{Pattern Recogn.}, 44\penalty0 (7):\penalty0 1357--1371, July
  2011.
\newblock ISSN 0031-3203.
\newblock \doi{10.1016/j.patcog.2010.12.015}.
\newblock URL \url{http://dx.doi.org/10.1016/j.patcog.2010.12.015}.

\bibitem[Benesty et~al.(2009)Benesty, Chen, Huang, and
  Cohen]{benesty2009pearson}
Jacob Benesty, Jingdong Chen, Yiteng Huang, and Israel Cohen.
\newblock Pearson correlation coefficient.
\newblock In \emph{Noise reduction in speech processing}, pages 1--4. Springer,
  2009.

\bibitem[B{\"o}hm et~al.(2022)B{\"o}hm, Berens, and Kobak]{bohm2022attraction}
Jan~Niklas B{\"o}hm, Philipp Berens, and Dmitry Kobak.
\newblock Attraction-repulsion spectrum in neighbor embeddings.
\newblock \emph{Journal of Machine Learning Research}, 23\penalty0
  (95):\penalty0 1--32, 2022.

\bibitem[Boutros et~al.(2022)Boutros, Damer, Kirchbuchner, and
  Kuijper]{boutros2022self}
Fadi Boutros, Naser Damer, Florian Kirchbuchner, and Arjan Kuijper.
\newblock Self-restrained triplet loss for accurate masked face recognition.
\newblock \emph{Pattern Recognition}, 124:\penalty0 108473, 2022.

\bibitem[Cai et~al.(2025)Cai, Yang, Zhu, and Nie]{cai2025multi}
Yingjie Cai, Hui Yang, Jianyong Zhu, and Feiping Nie.
\newblock Multi-subspace graph clustering joint dimensionality reduction and
  feature selection.
\newblock \emph{Pattern Recognition}, page 112557, 2025.

\bibitem[Chen et~al.(2025)Chen, Jiang, Jiao, Li, Xu, Wang, and
  Wang]{chen2025fourier}
Qiupu Chen, Hongkui Jiang, Lin Jiao, Zhou Li, Taosheng Xu, Xue Wang, and Rujing
  Wang.
\newblock Fourier-enhanced semi-supervised proxy learning for
  ultra-fine-grained novel class discovery.
\newblock \emph{Pattern Recognition}, page 112369, 2025.

\bibitem[Cheng et~al.(2012)Cheng, Cheng, and Guo]{cheng2012}
Jian Cheng, Can Cheng, and Yi-nan Guo.
\newblock Supervised isomap based on pairwise constraints.
\newblock In \emph{Proceedings of the 19th International Conference on Neural
  Information Processing - Volume Part I}, ICONIP'12, pages 447--454, Berlin,
  Heidelberg, 2012. Springer-Verlag.
\newblock ISBN 978-3-642-34474-9.
\newblock \doi{10.1007/978-3-642-34475-6_54}.
\newblock URL \url{http://dx.doi.org/10.1007/978-3-642-34475-6_54}.

\bibitem[Duan et~al.(2025)Duan, Jiao, Kang, Yang, and Zhou]{duan2025deep}
Chenguang Duan, Yuling Jiao, Lican Kang, Jerry~Zhijian Yang, and Fusheng Zhou.
\newblock Deep contrastive representation learning for supervised tasks.
\newblock \emph{Pattern Recognition}, 161:\penalty0 111309, 2025.

\bibitem[Fisher(1936)]{fisher36lda}
R.~A. Fisher.
\newblock The use of multiple measurements in taxonomic problems.
\newblock \emph{Annals of Eugenics}, 7\penalty0 (7):\penalty0 179--188, 1936.

\bibitem[Ghosh and Kirby(2022)]{ghosh2022supervised}
Tomojit Ghosh and Michael Kirby.
\newblock Supervised dimensionality reduction and visualization using
  centroid-encoder.
\newblock \emph{Journal of Machine Learning Research}, 23\penalty0
  (20):\penalty0 1--34, 2022.

\bibitem[Ghosh and Kirby(2024)]{ghosh2024linear}
Tomojit Ghosh and Michael Kirby.
\newblock Linear centroid encoder for supervised principal component analysis.
\newblock \emph{Pattern Recognition}, 155:\penalty0 110634, 2024.

\bibitem[Globerson and Roweis(2006)]{globerson2006metric}
Amir Globerson and Sam~T Roweis.
\newblock Metric learning by collapsing classes.
\newblock In \emph{Advances in neural information processing systems}, pages
  451--458, 2006.

\bibitem[Goldberger et~al.(2004)Goldberger, Roweis, Hinton, and
  Salakhutdinov]{Goldberger2004NCA}
Jacob Goldberger, Sam Roweis, Geoff Hinton, and Ruslan Salakhutdinov.
\newblock Neighbourhood components analysis.
\newblock In \emph{Proceedings of the 17th International Conference on Neural
  Information Processing Systems}, NIPS'04, pages 513--520, Cambridge, MA, USA,
  2004. MIT Press.
\newblock URL \url{http://dl.acm.org/citation.cfm?id=2976040.2976105}.

\bibitem[G{\"u}rsoy et~al.(2023)G{\"u}rsoy, Haddad, and
  Bothorel]{gursoy2023alignment}
Furkan G{\"u}rsoy, Mounir Haddad, and C{\'e}cile Bothorel.
\newblock Alignment and stability of embeddings: measurement and inference
  improvement.
\newblock \emph{Neurocomputing}, 553:\penalty0 126517, 2023.

\bibitem[He et~al.(2021)He, Wu, Zhao, Zhang, and Tan]{he2021representative}
Zaixing He, Mengtian Wu, Xinyue Zhao, Shuyou Zhang, and Jianrong Tan.
\newblock Representative null space lda for discriminative dimensionality
  reduction.
\newblock \emph{Pattern Recognition}, 111:\penalty0 107664, 2021.

\bibitem[Iwata et~al.(2026)Iwata, Nontani, and Mimura]{iwata2026making}
Kazunori Iwata, Yuki Nontani, and Kazushi Mimura.
\newblock Making clustering methods workable for shapes using the ordinary
  procrustes sum of squares.
\newblock \emph{Pattern Recognition}, 169:\penalty0 111878, 2026.

\bibitem[Kamal and Bae(2022)]{kamal2022super}
Imam~Mustafa Kamal and Hyerim Bae.
\newblock Super-encoder with cooperative autoencoder networks.
\newblock \emph{Pattern Recognition}, 126:\penalty0 108562, 2022.

\bibitem[Kamyab and Azimifar(2024)]{kamyab2024deep}
Shima Kamyab and Zohreh Azimifar.
\newblock Deep-mds framework for recovering the 3d shape of 2d landmarks from a
  single image.
\newblock \emph{Journal of Visual Communication and Image Representation},
  98:\penalty0 104032, 2024.

\bibitem[Kobak and Linderman(2021)]{kobak2021initialization}
Dmitry Kobak and George~C Linderman.
\newblock Initialization is critical for preserving global data structure in
  both t-sne and umap.
\newblock \emph{Nature biotechnology}, 39\penalty0 (2):\penalty0 156--157,
  2021.

\bibitem[Kruskal(1964)]{kruskal1964}
Joseph~B Kruskal.
\newblock Nonmetric multidimensional scaling: a numerical method.
\newblock \emph{Psychometrika}, 29\penalty0 (2):\penalty0 115--129, 1964.

\bibitem[Le et~al.(2018)Le, Patterson, and White]{le2018supervised}
Lei Le, Andrew Patterson, and Martha White.
\newblock Supervised autoencoders: Improving generalization performance with
  unsupervised regularizers.
\newblock \emph{Advances in neural information processing systems},
  31:\penalty0 107--117, 2018.

\bibitem[Makhzani et~al.(2015)Makhzani, Shlens, Jaitly, Goodfellow, and
  Frey]{makhzani2015}
Alireza Makhzani, Jonathon Shlens, Navdeep Jaitly, Ian Goodfellow, and Brendan
  Frey.
\newblock Adversarial autoencoders.
\newblock \emph{arXiv preprint arXiv:1511.05644}, 2015.

\bibitem[McInnes et~al.(2018)McInnes, Healy, and Melville]{mcinnes2018umap}
Leland McInnes, John Healy, and James Melville.
\newblock Umap: Uniform manifold approximation and projection for dimension
  reduction.
\newblock \emph{arXiv preprint arXiv:1802.03426}, 2018.

\bibitem[Meng et~al.(2023)Meng, van~den Elzen, Pezzotti, and
  Vilanova]{meng2023class}
Linhao Meng, Stef van~den Elzen, Nicola Pezzotti, and Anna Vilanova.
\newblock Class-constrained t-sne: combining data features and class
  probabilities.
\newblock \emph{IEEE Transactions on Visualization and Computer Graphics},
  30\penalty0 (1):\penalty0 164--174, 2023.

\bibitem[Min et~al.(2010)Min, Maaten, Yuan, Bonner, and Zhang]{min2010deep}
Martin~R Min, Laurens Maaten, Zineng Yuan, Anthony~J Bonner, and Zhaolei Zhang.
\newblock Deep supervised t-distributed embedding.
\newblock In \emph{Proceedings of the 27th International Conference on Machine
  Learning (ICML-10)}, pages 791--798, 2010.

\bibitem[Raducanu and Dornaika(2012)]{raducanu2012}
B.~Raducanu and F.~Dornaika.
\newblock A supervised non-linear dimensionality reduction approach for
  manifold learning.
\newblock \emph{Pattern Recogn.}, 45\penalty0 (6):\penalty0 2432--2444, June
  2012.
\newblock ISSN 0031-3203.
\newblock \doi{10.1016/j.patcog.2011.12.006}.
\newblock URL \url{http://dx.doi.org/10.1016/j.patcog.2011.12.006}.

\bibitem[Ritchie et~al.(2019)Ritchie, Scott, Balzano, Kessler, and
  Sripada]{ritchie2019supervised}
Alexander Ritchie, Clayton Scott, Laura Balzano, Daniel Kessler, and Chandra~S
  Sripada.
\newblock Supervised principal component analysis via manifold optimization.
\newblock In \emph{2019 IEEE Data Science Workshop (DSW)}, pages 6--10. IEEE,
  2019.

\bibitem[Salakhutdinov and Hinton(2007)]{salakhutdinov2007learning}
Ruslan Salakhutdinov and Geoff Hinton.
\newblock Learning a nonlinear embedding by preserving class neighbourhood
  structure.
\newblock In \emph{Artificial Intelligence and Statistics}, pages 412--419,
  2007.

\bibitem[Szubert et~al.(2019)Szubert, Cole, Monaco, and
  Drozdov]{szubert2019structure}
Benjamin Szubert, Jennifer~E Cole, Claudia Monaco, and Ignat Drozdov.
\newblock Structure-preserving visualisation of high dimensional single-cell
  datasets.
\newblock \emph{Scientific reports}, 9\penalty0 (1):\penalty0 8914, 2019.

\bibitem[Torgerson(1952)]{torgerson1952}
W.~S. Torgerson.
\newblock Multidimensional scaling: I. theory and method.
\newblock \emph{Psychometrika}, 17:\penalty0 401--419, 1952.

\bibitem[van~der Maaten(2009)]{DBLP:journals/jmlr/Maaten09}
Laurens van~der Maaten.
\newblock Learning a parametric embedding by preserving local structure.
\newblock In \emph{{AISTATS}}, volume~5 of \emph{{JMLR} Proceedings}, pages
  384--391. JMLR.org, 2009.

\bibitem[van~der Maaten and Hinton(2008)]{tSNE}
Laurens van~der Maaten and Geoffrey Hinton.
\newblock Visualizing data using {t-SNE}.
\newblock \emph{Journal of Machine Learning Research}, 9:\penalty0 2579--2605,
  2008.
\newblock URL \url{http://www.jmlr.org/papers/v9/vandermaaten08a.html}.

\bibitem[Wei et~al.(2022)Wei, Shen, Sun, Gao, and Ren]{wei2022neighborhood}
Dong Wei, Xiaobo Shen, Quansen Sun, Xizhan Gao, and Zhenwen Ren.
\newblock Neighborhood preserving embedding on grassmann manifold for image-set
  analysis.
\newblock \emph{Pattern Recognition}, 122:\penalty0 108335, 2022.

\bibitem[Witten and Tibshirani(2011)]{witten2011penalized}
Daniela~M Witten and Robert Tibshirani.
\newblock Penalized classification using fisher's linear discriminant.
\newblock \emph{Journal of the Royal Statistical Society: Series B (Statistical
  Methodology)}, 73\penalty0 (5):\penalty0 753--772, 2011.

\bibitem[Wong and Zhao(2012)]{wong2012supervised}
Wai~Keung Wong and HT~Zhao.
\newblock Supervised optimal locality preserving projection.
\newblock \emph{Pattern Recognition}, 45\penalty0 (1):\penalty0 186--197, 2012.

\bibitem[Zahir et~al.(2026)Zahir, Jbilou, and Ratnani]{zahir2026novel}
Alaeddine Zahir, Khalide Jbilou, and Ahmed Ratnani.
\newblock Novel dimensionality reduction approaches for higher-order
  multidimensional data.
\newblock \emph{Pattern Recognition}, 169:\penalty0 111938, 2026.

\bibitem[Zhou et~al.(2023)Zhou, Zhang, Zeng, Zhang, and Fang]{zhou2023latent}
Jianhang Zhou, Qi~Zhang, Shaoning Zeng, Bob Zhang, and Leyuan Fang.
\newblock Latent linear discriminant analysis for feature extraction via
  isometric structural learning.
\newblock \emph{Pattern Recognition}, page 110218, 2023.

\end{thebibliography}

\end{document}